\documentclass[twoside,11pt]{article}

\usepackage{jmlr2e}

\usepackage[margin=1in]{geometry}
\usepackage[utf8]{inputenc} 
\usepackage[T1]{fontenc}   
\usepackage{url}            
\usepackage{booktabs}       
\usepackage{amsfonts}      
\usepackage{nicefrac}      
\usepackage{microtype}      
\usepackage{graphicx}
\usepackage{subfig}
\usepackage{amssymb,amsbsy,amsmath,amsfonts,amssymb,amscd}
\usepackage{comment}
\usepackage{verbatim}
\usepackage[dvipsnames]{xcolor}
\usepackage{enumitem}
\usepackage{algorithm,algpseudocode}
\usepackage{hyperref}       
\usepackage{setspace}

\newtheorem{assum}{Assumption}[section]
\newtheorem{cor}{Corollary}[section]

\newcommand{\EE}{\mathbb{E}}

\newcommand{\ini}{\mathrm{ini}}

\newcommand{\dis}{\mathrm{dis}}
\newcommand{\rd}{\,\mathrm{d}}
\newcommand{\revise}[1]{\textcolor{black}{#1}}

\jmlrheading{1}{2000}{1-48}{4/00}{10/00}{meila00a}{Zhiyan Ding, Shi Chen, Qin Li, and Stephen J. Wright}

\ShortHeadings{Convergence of gradient descent for overparameterized multi-layer ResNets }{Ding, Chen, Li, and Wright}
\firstpageno{1}

\begin{document}

\title{Overparameterization of deep ResNet: zero loss and mean-field analysis}

\author{\name Zhiyan Ding \email zding49@math.wisc.edu\\
       \addr Mathematics Department\\
       University of Wisconsin-Madison\\
       Madison, WI 53706 USA.
       \AND
       \name Shi Chen \email schen636@wisc.edu\\
       \addr Mathematics Department\\
       University of Wisconsin-Madison\\
       Madison, WI 53706 USA.
       \AND
       \name Qin Li \email qinli@math.wisc.edu\\
       \addr Mathematics Department\\
       University of Wisconsin-Madison\\
       Madison, WI 53706 USA.
       \AND
       \name Stephen J. Wright \email swright@cs.wisc.edu\\
       \addr Department of Computer Sciences \\
       University of Wisconsin-Madison\\
       Madison, WI 53706 USA.
       }

\editor{}

\maketitle

\begin{abstract}
Finding parameters in a deep neural network (NN) that fit training data is a nonconvex optimization problem, but a basic first-order optimization method (gradient descent) finds a global optimizer with perfect fit (zero-loss) in many practical situations. We examine this phenomenon for the case of Residual Neural Networks (ResNet) with smooth activation functions in a limiting regime in which both the number of layers (depth) and the number of weights in each layer (width) go to infinity. First, we use a mean-field-limit argument to prove that the gradient descent for parameter training becomes a gradient flow for a probability distribution that is characterized by a partial differential equation (PDE) in the large-NN limit. Next, we show that \revise{under certain assumptions,} the solution to the PDE converges in the training time to a zero-loss solution. Together, these results \revise{suggest} that the training of the ResNet gives a near-zero loss if the ResNet is large enough. We give estimates of the depth and width needed to reduce the loss below a given threshold, with high probability.
\end{abstract}


\section{Introduction}
In the training of deep neural networks (NN), the loss function is nonconvex, but algorithms based on gradient descent (GD) appear to converge to a zero-loss solution that interpolates the training data. As the number of layers and the width of the NN increase to infinity, the training problem becomes vastly overparameterized, potentially yielding a large set of interpolating solutions. Still, it remains surprising that GD approaches find the global minimizer reliably.

\emph{What is the mechanism that allows gradient descent to find the minimizer?}

This article \revise{proposes a partial answer to} this question for residual neural networks (ResNet), relying on three main toolboxes: a continuous limit argument, a mean-field limit argument, and gradient flow analysis. These toolboxes are used to translate gradient descent for parameter training into a partial differential equation (PDE), where PDE analysis (specifically, steady-state equilibrium analysis) is employed to trace the convergence to the minimizer.

We denote by $L$ and $M$ the number of layers, and the number of weights per layer, respectively. We study highly overparameterized NN in which both $L$ and $M$ go to infinity. As $L\to\infty$, the NN approaches infinite depth, and ResNet translates into an ordinary differential equation (ODE) with a force term encoding the parameter configuration. We refer to this limit as the {\em continuous limit}. As $M\to\infty$, the NN approaches infinite width, and the ODE that characterizes ResNet further translates into an ordinary-integral equation (OIE), with the integrand encoding the configuration of neuron parameters. We refer to this limit as the {\em mean-field limit}. In the limiting regime, the training of the parameters becomes the training of the integrand, and gradient descent on the parameters translates into a gradient flow on the integrand. The gradient flow process turns out to be characterized by a PDE. We observe  that \revise{under certain assumptions,} the long ``time" limit of the PDE solution has zero cost, \revise{suggesting} that gradient descent applied to a (finite) parameter configuration yields almost-zero-cost solutions too. The smallness of this cost depends on the largeness of $L$ and $M$; we give a precise quantification.




We describe the problem setup in Section~\ref{sec:setup}, where we also formally derive the continuous and  mean-field limits. We unify notations in Section~\ref{sec:notations} and present our main results in Section~\ref{sec:conv}, including rigorous justifications for both limiting procedures, the result about convergence of the PDE solution to the global minimum, and a result about convergence to near-zero losses of gradient descent on large (but finite) networks.
As a byproduct of our analysis, well-posedness for both the ODE formulation and the PDE formulation is obtained in  Section~\ref{sec:well_posed}.

There is a vast of literature addressing the overparameterization of DNN, its associated nonconvex optimization problem, and the application of the first order (gradient descent) optimization methods. 
Three major perspectives have been taken. 
One approach is to analyze the landscape of the objective functions, identifying properties that a nonconvex optimization problem needs to satisfy for a first order method to converge to an optimizer. 
Different NN structures are then analyzed to fit these properties~\citep{pmlr-v70-jin17a,pmlr-v40-Ge15,NIPS2017_f79921bb,ge2018learning,Nguyen2018OptimizationLA,pmlr-v80-du18a,10.1109/TIT.2018.2854560,pmlr-v70-nguyen17a,NIPS2016_f2fc9902,yun2018global,safran2020effects,pmlr-v80-bartlett18a,bartlett2018representing}. The second approach largely depends on the insensibility of the Gram matrix as the number of weights goes to infinity~\citep{pmlr-v97-allen-zhu19a,pmlr-v97-du19c,zhang2019convergence,chatterji2021does,zou2018stochastic,du2018gradient,NEURIPS2019_6e2290db,NEURIPS2020_9afe487d}, as a way to analyze Neural Tangent Kernel~\citep{NEURIPS2018_5a4be1fa,NEURIPS2020_b7ae8fec}. The analysis crucially depends on the fact that GD confines the iteration in a small neighborhood if the initial conditions are random. However, as point out in~\citep{NEURIPS2019_58,Ba2020GeneralizationOT,Wei2019RegularizationMG,fang2019convex}, NTK, like many other kernel models~\cite{Arora2019OnEC}, is rather a limited description of NN. The success of NTK highly depends on the fact that a nonlinear NN can be approximated by a linear model of the infinite dimensional random features associated with the NTK. The third approach, which we adopt in this paper, involves the mean-field limit. 
The idea is to let the number of weights $M$ and/or the number of layers $L$ go to infinity, and translate the gradient descent for the parameters in a finite dimensional space to a gradient flow in a (infinite-dimensional) function space~\citep{araujo2019meanfield,fang2019convex,nguyen2019mean,pmlr-v97-allen-zhu19a,pmlr-v97-du19c,zhang2019convergence,chatterji2021does,NEURIPS2018_a1afc58c,MeiE7665,wojtowytsch2020convergence,pmlr-v119-lu20b,doi:10.1287/moor.2020.1118,doi:10.1137/18M1192184,e2020mathematical,nguyen2019mean,nguyen2021rigorous,cONG2021,pmlr-v119-lu20b,E_2020,jabir2021meanfield}.~\revise{In particular, in~\citep{E_2020}, the authors have formally derived the limiting gradient flow equation for deep and  overparameterized ResNet.}
Some PDE analysis techniques are borrowed to show the convergence of the limiting gradient flow equation. In~\citep{NEURIPS2018_a1afc58c,MeiE7665,wojtowytsch2020convergence}, the authors study the two-layer (one hidden layer) NN in the regime where $M$ is infinite. 
In these works, both the mean-field limit and the convergence of the gradient flow are proved, giving a complete answer about the convergence to the zero-loss when NN has two layers. 
The problem is significantly more complicated for NNs that have multiple layers, since the convexity of the objective function is lost. To the best of our knowledge, the current paper is the first that both rigorously proves the mean-field limit that translates the optimization problem into a gradient flow, and demonstrates why the gradient flow leads to a global minimizer.

\section{ResNet and gradient descent}\label{sec:setup}
General ResNet uses the following model:
\begin{equation}\label{eqn:disRes}
\begin{aligned}
z_{l+1}(x) &=z_{l}(x)+\frac{1}{ML}\sum^M_{m=1}f(z_{l}(x),\theta_{l,m})\,,\quad l=0,1,\dots,L-1\\
z_{0}(x) &=x\,,
\end{aligned}
\end{equation}
where $x\in\mathbb{R}^d$ is the input data. The configuration of the NN is encoded in ${\Theta}_{L,M}=\left\{\theta_{l,m}\right\}^{L-1,M}_{l=0,m=1}$, where $\theta_{l,m}\in\mathbb{R}^k$, and $f:\mathbb{R}^d\times\mathbb{R}^k\rightarrow\mathbb{R}^d$ is a given activation function. \citep[Section 3.3]{e2020mathematical} verifies that ``conventional'' ResNets are included in this framework. Define
\[
\begin{aligned}
z_{l+1}(x)& =z_{l}(x)+\frac{1}{ML}\sum^{M}_{m=1}U_{l,m}\sigma(W^\top_{l,m}z_{l}(x)+b_{l,m})\,,\quad l=0,1,\dots,L-1\\
z_{0}(x)& =x\,,
\end{aligned}
\]
where $W_{l,m},U_{l,m}\in\mathbb{R}^d$, $b_{l.m}\in\mathbb{R}$, and $\sigma$ is the ReLU activation function. 
In this example, we have $\theta_{l,m}=(W_{l,m},U_{l,m},b_{l,m})\in\mathbb{R}^k$, with $k=2d+1$.

Denote by $Z_{{\Theta}_{L,M}}(l;x)$ the solution to \eqref{eqn:disRes}. The goal of training ResNet is to seek parameters $\Theta_{L,M}$ such that the following cost is minimized:
\begin{equation}\label{eqn:costfunction}
E({\Theta}_{L,M})=\mathbb{E}_{x\sim\mu}\left[\frac{1}{2}\left(g(Z_{{\Theta}_{L,M}}(L;x))-y(x)\right)^2\right]\,,
\end{equation}
where $g(x):\mathbb{R}^d\rightarrow\mathbb{R}$ is a given measuring function, $y(x)\in\mathbb{R}$ is the label corresponding to $x$, and $\mu$ is the probability distribution from which $x$ is drawn.

Classical gradient descent updates the parameters according to the formula
\begin{equation*}
\Theta_{L,M}^{n+1}=\Theta_{L,M}^{n}-h\nabla_{{\Theta}_{L,M}} E(\Theta^n_{L,M})\,,
\end{equation*}
where $h$ is the step length. In the $h\to0$ limit, these parameters satisfy the following ODE~\citep[Def~2.2]{NEURIPS2018_a1afc58c}, with $s$ denoting the rescaled pseudo-time:
\begin{equation}\label{eqn:gradient_theta_LM}
    \frac{\rd{\Theta}_{L,M}(s)}{\rd s}=-ML\nabla_{{\Theta}_{L,M}} E({\Theta}_{L,M})\,,\quad\text{for}\; s\geq 0\,.
\end{equation}



\subsection{The continuous limit and mean-field limit}\label{sec:limit_formal}

We obtain the {\em continuous limit} of \eqref{eqn:disRes} by making the ResNet infinitely deep, that is,  $L\to\infty$. By reparameterizing the indices $l=[0,1, \dotsc,L-1]$ with the continuous variable $t\in[0,1]$, we view $z$ in~\eqref{eqn:disRes} as a function in $t$ that satisfies a coupled ODE, with $1/L$ being the stepsize in time. Accordingly, $\theta_{l,m}$ can be viewed as $\theta_m(t=l/L)$. Denoting $\Theta(t) = \{\theta_m(t)\}_{m=1}^M$, the continuous limit of \eqref{eqn:disRes} can be written as
\begin{equation}\label{eqn:contRes}
\frac{\rd z(t;x)}{\rd t}=\frac{1}{M}\sum^M_{m=1}f(z(t;x),\theta_m(t))\,,\quad t\in[0,1], \quad\text{with}\; z(0;x)=x\,.
\end{equation}
Following~\eqref{eqn:costfunction}, we define the cost functional $E$ for this formulation as
\begin{equation}\label{eqn:costfunctioncont}
E(\Theta)=\mathbb{E}_{x\sim\mu}\left[\frac{1}{2}\left(g(Z_{\Theta}(1;x))-y(x)\right)^2\right]\,,
\end{equation}
where $Z_{\Theta}(t;x)$ solves \eqref{eqn:contRes} for a given collection $\Theta(t)$ of the $M$ functions $\{\theta_m(t)$, $m=1,2,\dotsc,M\}$. Similar to~\eqref{eqn:gradient_theta_LM}, we can use GD to find the configuration of $\Theta(t)$ that minimizes~\eqref{eqn:costfunctioncont} by making $\Theta(t)$ flow in the descending direction of $E(\Theta)$, that is,
\begin{equation}\label{eqn:gradient_theta_M}
\frac{\partial \Theta(s;t)}{\partial s}=-M\left.\frac{\delta E}{\delta\Theta}\right|_{\Theta(s;\cdot)}\,,\quad s>0,\quad t\in[0,1]\,,
\end{equation}
where $\frac{\delta E}{\delta\Theta}$ is the functional derivative of $E$ with respect to $\Theta(t)$.

The {\em mean-field limit} is obtained by making the ResNet infinitely wide, that is, $M\to\infty$. 
Considering that the right hand side of~\eqref{eqn:contRes} has the form of an expectation, it naturally becomes an integral in the limit with respect to a certain probability density.
Denoting this density by $\rho(\theta,t)$, and assuming that the $\theta_m$ are drawn from it, the ODE \eqref{eqn:contRes} for $z$ translates to the following OIE:
\begin{equation}\label{eqn:meancontRes}
\frac{\rd z(t;x)}{\rd t}=\int_{\mathbb{R}^k}f(z(t;x),\theta)\rd\rho(\theta,t)\,,\quad t\in[0,1]\quad\text{with}\; z(0;x)=x\,.
\end{equation}
By mimicking~\eqref{eqn:costfunctioncont}, we define the cost function in the mean-field setting as
\begin{equation}\label{eqn:costfunctioncont2}
E(\rho)=\mathbb{E}_{x\sim\mu}\left[\frac{1}{2}\left(g(Z_{\rho}(1;x))-y(x)\right)^2\right]\,,
\end{equation}
where $Z_{\rho}(t;x)$ is the solution to \eqref{eqn:meancontRes} for a given $\rho$. Similar to the gradient flow for $\Theta_{L,M}$ and $\Theta(t)$, the probability distribution $\rho$ that encodes the configuration of $\theta$ in the mean-field setting, flows in the descending direction of $E(\rho)$ in pseudo-time $s$. 
Considering that $\rho(\theta,t,s)$ should be a probability density for all $s$ and $t$, this gradient flow is conducted in the Wasserstein metric. Following~\citep{NEURIPS2018_a1afc58c,pmlr-v119-lu20b}, we have
\begin{equation}\label{eqn:Wassgradientflow}
\frac{\partial \rho(\theta,t,s)}{\partial s}=\nabla_\theta\cdot\left(\rho(\theta,t,s)\nabla_\theta\left.\frac{\delta E(\rho)}{\delta \rho}\right|_{\rho(\cdot,\cdot,s)}\right)\,,\quad s>0,\; t\in[0,1]\,,
\end{equation}
where $\frac{\delta E}{\delta \rho}$ is the functional derivative with respect to $\rho$. With the classical variational method, one can explicitly compute this derivative as follows (see \citep[Section 3.1]{pmlr-v119-lu20b}):
\begin{equation}\label{eqn:Frechetderivative}
\left.\frac{\delta E}{\delta \rho}\right|_\rho(\theta,t)=\mathbb{E}_{x\sim\mu}\left(p^\top_\rho(t;x)f(Z_\rho(t;x),\theta)\right)\,,
\end{equation}
where $p_\rho(t;x)$, parameterized by $x$, maps $[0,1]\to\mathbb{R}^d$, and is a vector solution to the following ODE:
\begin{equation}\label{eqn:prho}
\left\{
\begin{aligned}
&\frac{\rd p^\top_\rho}{\rd t}=-p^\top_\rho\int_{\mathbb{R}^k}\partial_zf(Z_\rho,\theta)\rho(\theta,t)\rd\theta\,,\\
&p_\rho(t=1;x)=\left(g(Z_{\rho}(1;x))-y(x)\right)\nabla g(Z_{\rho}(1;x))\,.
\end{aligned}
\right.
\end{equation}
In the later sections, to emphasize the $s$ dependence, we use $\frac{\delta E(\Theta(s))}{\delta \Theta}$ and $\frac{\delta E(\rho(s))}{\delta \rho}$ to denote $\left.\frac{\delta E}{\delta \Theta}\right|_{\Theta(s;\cdot)}$ and $\left.\frac{\delta E}{\delta \rho}\right|_{\rho(\cdot,\cdot,s)}$ respectively. As a summary, to update $\rho(\theta,t,s)$ to $\rho(\theta,t,s+\delta s)$ with an infinitesimal $\delta s$, we solve~\eqref{eqn:meancontRes} for $Z_\rho(t;x)$ using the given $\rho(\theta,t,s)$ and compute $p_\rho$ using~\eqref{eqn:prho}. This allows us to compute $\frac{\delta E(\rho(s))}{\delta \rho}(\theta,t)$ which, in turn,  yields $\rho(\theta,t,s+\delta s)$ from~\eqref{eqn:Wassgradientflow}. In~\eqref{eqn:prho}, $\partial_zf$ is a $d\times d$ matrix that stands for the Jacobian of $f$ with respect to its $z$ argument.



\begin{remark} We stress that equation~\eqref{eqn:Wassgradientflow} is not a classical Wasserstein-2 gradient flow in the probability space. The classical gradient flow looks for one probability function of $\theta$ in the pseudo-time $s$~\citep{Gradientflow}. Instead, we now have $\rho(\theta,t)$, where $\rho$ is a probability density of $\theta$ for every $t$, and their flows are all coupled through the definition of $E$. Note that this feature emerges exactly because we have infinitely many layers $L\to\infty$, reformulated using a continuous variable $t$. When there are only 2 layers, the dependence on $t$ can be dropped, and the problem reduces to the classical gradient flow~\citep{NEURIPS2018_a1afc58c}. 
A core obstacle to be overcome in this paper is the need for  new technicalities to handle the non-standard gradient flow.
\end{remark}

\subsection{Contributions}\label{sec:contribution}
\revise{With the roadmap above, it is clear that we need to justify that the gradient flow in the discrete setting~\eqref{eqn:gradient_theta_LM} is tracked closely by the PDE, so that $E(\Theta_{L,M}(s))\approx E(\rho(\cdot,\cdot,s))$, and that the PDE~\eqref{eqn:Wassgradientflow} achieves the global minimum, for which $E(\rho(\cdot,\cdot,s=\infty))=0$.}

\revise{In this paper we fully address the first question, and we provide some conditions under which the global minimum can be achieved. More specifically, the two main contributions of the paper are as follows.}
\begin{itemize}
\item[--]{Contribution 1:} We give a rigorous proof of the continuous and mean-field limit: Theorem~\ref{thm:consistent}. This is to justify that when $M,L\to\infty$, $E(\Theta_{L,M}(s))\approx E(\Theta(s;\cdot))\approx E(\rho(\cdot,\cdot,s))$. The dependence of these approximations on $L$ and $M$ are made precise in the result.

\item[--]{Contribution 2:} \revise{We show that the global minimum can be achieved in the continuous setting under certain conditions: Theorem~\ref{thm:globalminimal}. That is,  the gradient flow equation~\eqref{eqn:Wassgradientflows}, a modified version of~\eqref{eqn:Wassgradientflow}, sends its distribution $\rho(\theta,t,s)$ to a global minimizer in the long-time regime (as $s\to\infty$). The conditions required may be stringent in real applications and we discuss their validity in certain scenarios in Appendix~\ref{sec:proofofdiscussion}.}
\end{itemize}
Theorems~\ref{thm:consistent} and  \ref{thm:globalminimal} are combined in Theorem~\ref{thm:maintheorem}, which shows that $E(\Theta_{L,M}(s))\to0$ for large $s$.

By comparison with existing mean-field limit results,  Theorem~\ref{thm:consistent} gives rigorous justifications for both $L\to\infty$ and $M\to\infty$. By contrast,~\citep{NEURIPS2018_a1afc58c,araujo2019meanfield,fang2019convex,nguyen2019mean} consider $M\to\infty$ with $L$ fixed. 
In most results for NTK~\citep{pmlr-v97-allen-zhu19a,pmlr-v97-du19c,zhang2019convergence,chatterji2021does}, the size of $M$ grows as $L$ increases, while our results have these two parameters independent of each other. Furthermore, we have no requirement on the number of training samples.

For the a priori estimates, we also show the well-posedness of all six equations (three for $z$ and three for $\theta$) and that these equations have unique solutions.
These results are mostly collected in Section~\ref{sec:well_posed}, as follows.
\begin{itemize}
\item[--]{Well-posedness of $Z$:} Theorem~\ref{thm:wmean-fieldlimit}. For a given $\rho$, equation~\eqref{eqn:meancontRes} has a unique solution $Z_\rho$, and the solution is stable with respect to $\rho$. With Remark~\ref{rmk:equivalence}, the uniqueness and stability result extend to treating~\eqref{eqn:contRes} for a given ${\Theta}(t)$.

\item[--]{Well-posedness of $\Theta_{L,M},\Theta(t),\rho$:} Theorems~\ref{thm:finiteLwell-posed1} and \ref{thm:Wassgradientflows}. These quantities satisfy the modified gradient flow equations~\eqref{eqn:classicalgradientflowfiniteL},~\eqref{eqn:Wassgradientflowsdis}, and~\eqref{eqn:Wassgradientflows} respectively, and these equations have unique solutions.
\end{itemize}

    


\section{Notations, assumptions, and definitions}\label{sec:notations}

Throughout the paper, we denote the collection of probability distribution that has a bounded second moment by $\mathcal{P}^2(\mathbb{R}^k)$, that is, $\mathcal{P}^2(\mathbb{R}^k)=\{\rho: \int_{\mathbb{R}^k}|\theta|^2\rd\rho(\theta)<\infty\}$. We assume certain regularity properties for the activation function $f$, the measuring function $g$, the data $y$, and the input measure $\mu$, as follows
%
\begin{assum}\label{assum:f}
Let $f:\mathbb{R}^{d}\times \mathbb{R}^k\rightarrow\mathbb{R}^d$ and $g,y:\mathbb{R}^d\rightarrow\mathbb{R}$ be $C^2$ functions. We assume that these functions and $\mu$ satisfy the following properties:
\begin{itemize}
\item[1.] For all $x\in\mathbb{R}^d,\theta\in\mathbb{R}^k$, we have
\begin{equation}\label{eqn:derivativebound}
\left|\partial^i_{x}\partial^j_{\theta}f\right|\leq C_1|\theta|^i\left(\left|x\right|+1\right)^j, \quad 0<i+j\leq 2,\ i\geq 0,\ j\geq 0\,,
\end{equation}
where $C_1$ is a constant, $\partial^i_{x}\partial^j_{\theta}f$ is $(i,j)$ derivative of $f$, and $|\cdot|$ is the Frobenius norm.
\item[2.] The function set $\left\{h\middle|h=\int_{\mathbb{R}^k}f(x,\theta)\rd\rho(\theta),\ \rho\in \mathcal{P}^2(\mathbb{R}^k)\right\}$ is dense in $\mathcal{C}\left(|x|<R;\mathbb{R}^d\right)$ for any $R>0$.

\item[3.] $g(x)$ and $\nabla g(x)$ are Lipschitz continuous. Moreover, $|\nabla g(x)|$ has a positive lower bound, that is, $\inf_{x\in\mathbb{R}^d}|\nabla g(x)|>0$.

\item[4.] The associated probability density function of $\mu(x)$ is in $L^\infty$, and is compactly supported:
\begin{equation}\label{mucompactsupport}
\mu(\left\{|x|\geq R\}\right)=0\,,
\end{equation}
for some fixed $R>0$.
\item[5.] $y(x)\in L^\infty_{\mathrm{loc}}(\mathbb{R}^d)$, that is,  $
\sup_{|x|\leq R}|y(x)|<\infty$.
\end{itemize}
\end{assum}

The assumption \eqref{eqn:derivativebound} is satisfied by most commonly used activation functions~\citep{e2020mathematical}. Consider, for example,
\[
f(x,\theta)=f(x,\theta_1,\theta_2,\theta_3,\theta_4)=\sigma(\theta_1x+\theta_2)-\sigma(\theta_3x+\theta_4)\,,
\]
where $\theta_1,\theta_3\in\mathbb{R}^{d\times d}$, and $\theta_2,\theta_4\in \mathbb{R}^{d}$, and $\sigma$ is a component-wise regularized ReLU activation function. The form of the assumption controls the growth of its derivatives, and is used widely~\citep{NEURIPS2018_a1afc58c,pmlr-v119-lu20b}.

For every fixed $s$, the solution to~\eqref{eqn:Wassgradientflow} is in the space $\mathcal{C}([0,1];\mathcal{P}^2)$. This is a non-standard probability space due to the introduction of $t$, so we need to build a new metric. 
\begin{definition}\label{def:path}\footnote{$\mathcal{C}([0,T];\mathcal{A})$ is a collection of functions $f(\theta,t)$ that are continuous in time $t\in[0,T]$, and for each fixed time $t_0$, $f(t_0,\theta)\in\mathcal{A}$. In this case, $T=1$ and $\mathcal{A}=\mathcal{P}^2$, equipped with $W_2$ metric.}
$\mathcal{C}([0,1];\mathcal{P}^2)$ is a collection of continuous paths of probability distribution $\rho(\theta,t)$ $(\theta\in\mathbb{R}^k,t\in[0,1])$ where
\begin{itemize}
\item[1.] For every fixed $t\in[0,1]$, $\rho(\cdot,t)\in \mathcal{P}^2(\mathbb{R}^k)$\,.

\item[2.] For any $t_0\in[0,1]$, $\lim_{t\rightarrow t_0}W_2\left(\rho(\cdot,t),\rho(\cdot,t_0)\right)=0$, where $W_2$ is the classical Wasserstein-2 distance.
\end{itemize}
The space $\mathcal{C}([0,1];\mathcal{P}^2)$ is equipped with the metric: $d_{1}\left(\rho_1,\rho_2\right)=\sup_tW_2(\rho_1(\cdot,t),\rho_2(\cdot,t))$.
\end{definition}

We further define the space involving the pseudo-time $s$. 
\begin{definition}\label{def:path2}\footnote{Similar to Definition \ref{def:path2}, $\mathcal{C}([0,\infty);\mathcal{C}([0,1];\mathcal{P}^2))$ is a collection of functions $f(\theta,t,s)$ that are continuous in $s$ and $t$, and for each fixed $(s_0,t_0)\in[0,\infty)\times[0,1]$, $f(\theta,t_0,s_0)\in\mathcal{P}^2$.}
$\mathcal{C}([0,\infty);\mathcal{C}([0,1];\mathcal{P}^2))$  is a collection of continuous paths of probability distribution $\rho(\theta,t,s)$ $(\theta\in\mathbb{R}^k,t\in[0,1],s\in[0,\infty))$ where
\begin{itemize}
\item[1.] For every fixed $s\in[0,\infty)$, $\rho(\cdot,\cdot,s)\in \mathcal{C}([0,1];\mathcal{P}^2)$.
\item[2.] For any $s_0\in[0,\infty)$,  $\lim_{s\rightarrow s_0}d_1\left(\rho(\cdot,\cdot,s),\rho(\cdot,\cdot,s_0)\right)=0$ where $d_1$ is given in Definition~\ref{def:path}.
\end{itemize}
The metric in $\mathcal{C}([0,\infty);\mathcal{C}([0,1];\mathcal{P}^2))$ is defined by
$
d_{2}\left(\rho_1,\rho_2\right)=\sup_{t,s}W_2(\rho_1(\cdot,t,s),\rho_2(\cdot,t,s))
$.
\end{definition}
Since $\mathcal{P}^2$ is complete in $W_2$ distance, $\mathcal{C}([0,1];\mathcal{P}^2)$ and $\mathcal{C}([0,\infty);\mathcal{C}([0,1];\mathcal{P}^2))$ are complete metric spaces under $d_1$ and $d_2$ respectively also.

For carrying out the mean-field limit, we use the classical approach and follow the particle presentation of $\rho$, this means we require the probability distribution $\rho$ to be admissible.
\begin{definition}\label{def:meanadmissible}
We call $\rho(\theta,t)$ admissible if it has a particle presentation, namely 
there exists a continuous stochastic process $\theta(t):[0,1]\rightarrow\mathbb{R}^k$ such that
\begin{equation}\label{eqn:existenceofpath}
\theta(t_0)\sim \rho(\theta,t_0)\,,\quad \lim_{t\rightarrow t_0}\mathbb{E}\left(|\theta(t)-\theta(t_0)|^2\right)=0\,.
\end{equation}
Furthermore, $\rho(\theta,t)$ is called limit-admissible if its $M$ averaged trajectory is bounded and Lipschitz with high probability. See the rigorous definition in Definition \ref{def:meanadmissible2}.
\end{definition}

\revise{\begin{remark} We have two comments on the admissible condition.
\begin{itemize}
\item Without the dependence on $t$, all probability distributions are admissible: One can define a particle by drawing it from this distribution. 
With the dependence on $t$, the admissible requirement is stricter: There needs to be a continuous stochastic process $\theta(t)$ whose underlying distribution is the given $\rho(\theta,t)$.
\item According to \cite[Theorem 8.2.1]{Gradientflow}, if $\rho(\cdot,t)$ is absolutely continuous in $t$ using the Wasserstein metric, it can be regarded as a probability measure over continuous function $C([0,1];\mathbb{R}^k)$, and can be represented by a continuous stochastic process $\theta(t)$.
\end{itemize}
\end{remark}
}

\section{Mean-field and convergence of the gradient flow}\label{sec:conv}
We present our main result in this section. To ensure the well-posedness of the gradient flow equation~\eqref{eqn:Wassgradientflow}, we need to add a regularizer to the cost function, as discussed in Section~\ref{sec:cost}. Sections~\ref{sec:meanfield} and \ref{sec:global} discuss validity of the continuous and mean-field limits and convergence to the global minimizer. 
The convergence to the global minimizer requires a fully-supported assumption; we discuss the validity of this assumption in Appendix~\ref{sec:proofofdiscussion}.

\subsection{Modified cost function}\label{sec:cost}
The analysis is conducted with a modified cost function to ensure the well-posedness of~\eqref{eqn:Wassgradientflow}, the gradient flow equation. To do so, we add a regularizer to the cost function~\eqref{eqn:costfunction}, as follows:
\begin{equation}\label{eqn:cost_LM_s}
E_s({\Theta}_{L,M})=\mathbb{E}_{x\sim\mu}\left[\frac{1}{2}\left(g(Z_{{\Theta}_{L,M}}(L;x))-y(x)\right)^2\right]+e^{-s}\frac{1}
{ML}\sum^{L-1}_{l=0}\sum^M_{m=1}|\theta_{l,m}|^2\,.
\end{equation}
Recalling our notation $\Theta(t)=\{\theta(t)\}_{m=1}^M$, the corresponding regularized cost function in the continuous limit ($L\rightarrow\infty$) is
\begin{equation}\label{eqn:cost_M_s}
E_s(\Theta(\cdot))=\mathbb{E}_{x\sim\mu}\left[\frac{1}{2}\left(g(Z_{\Theta}(1;x))-y(x)\right)^2\right]+\frac{e^{-s}}{M}\sum^M_{m=1}\int^1_0\left|\theta_{m}(t)\right|^2\rd t\,.
\end{equation}
Finally, in the mean-field limit ($M\rightarrow\infty$),  we assume $\theta_m$ is i.i.d. drawn from $\rho$, the corresponding regularized cost function becomes
\begin{equation}\label{eqn:cost_s}
E_s(\rho)=\mathbb{E}_{x\sim\mu}\left[\frac{1}{2}\left(g(Z_{\rho}(1;x))-y(x)\right)^2\right]+e^{-s}\int^1_0\int_{\mathbb{R}^k}|\theta|^2\rd\rho(\theta,t)\,.
\end{equation}
In these definitions, $Z_{\Theta_{L,M}}$, $Z_{\Theta}$, and $Z_{\rho}$ are solutions to the difference equation~\eqref{eqn:disRes}, ODE~\eqref{eqn:contRes} and OIE~\eqref{eqn:meancontRes}, respectively. It is straightforward to see that $E_s$ differs from the (non-regularized) cost function $E$ by only an exponentially small term, which disappears as $s \to \infty$. 
As a consequence, the global minimum of $E_s$ coincides with the global minimum of the true cost $E$ in the large-$s$ limit, so it makes sense to analyze gradient descent applied to $E_s$.

\subsection{Continuous and mean-field limit of the gradient flow}\label{sec:meanfield}
Mimicking the derivation in Section~\ref{sec:limit_formal}, the gradient flows for $\Theta_{L,M}$, $\Theta(t)$, and $\rho$ can be changed according to the regularized cost functions involving pseudo-time $s \in [0,\infty)$ defined above. 
$\Theta_{L,M}$ solves
\begin{equation}\label{eqn:classicalgradientflowfiniteL}
    \frac{\rd{\Theta}_{L,M}(s)}{\rd s}=-ML\nabla_{{\Theta}_{L,M}} E({\Theta}_{L,M}(s))-2e^{-s}{\Theta}_{L,M}(s)\,,\quad\text{for}\; s>0\,.
\end{equation}
For the continuous limit ($L\to\infty$), we have
\begin{equation}\label{eqn:Wassgradientflowsdis}
\frac{\rd\theta_m(s;t)}{\rd s}=-M\frac{\delta E(\Theta(s))}{\delta \theta_m}-2e^{-s}\theta_m(s;t)\,, 
\quad m=1,2,\dotsc,M.
\end{equation}
For the mean-field limit ($M\to\infty$), we have
\begin{equation}\label{eqn:Wassgradientflows}
\left\{
\begin{aligned}
&\frac{\partial \rho(\theta,t,s)}{\partial s}=\nabla_\theta\cdot\left(\rho(\theta,t,s)\nabla_\theta\frac{\delta E_s(\rho(s))}{\delta \rho}\right)\,,\quad\text{for}\; s>0,\quad t\in[0,1]\,,\\
&\rho(\theta,t,0)=\rho_{\ini}(\theta,t)\,.
\end{aligned}\right.
\end{equation}
Noting the definition of $E_s$ in~\eqref{eqn:cost_s}, we have the explicit expression for the Fr\'echet derivative:
\begin{equation}\label{eqn:derivaEs}
\frac{\delta E_s(\rho)}{\delta \rho}=\frac{\delta E(\rho)}{\delta \rho}+e^{-s}|\theta|^2\,,
\end{equation}
where $\frac{\delta E(\rho)}{\delta \rho}$ is defined in \eqref{eqn:Frechetderivative}. 



The connection between these three equations is made rigorous in the following result.




\begin{theorem}\label{thm:consistent}
Suppose that Assumption~\ref{assum:f} holds. Let $\rho_{\ini}(\theta,t)$ be limit-admissible, and $\{\theta_m(0;t)\}^M_{m=1}$ in \eqref{eqn:Wassgradientflowsdis} be $i.i.d.$ drawn from $\rho_{\ini}(\theta,t)$. Let
\begin{itemize}
\item $\Theta_{L,M}(s)=\{\theta_{l,m}(s)\}$ be the solution to~\eqref{eqn:classicalgradientflowfiniteL} with initial condition $\theta_{l,m}(s=0)=\theta_m\left(0;\frac{l}{L}\right)$\,, and 
\item $\rho(\theta,t,s)$ be the solution to \eqref{eqn:Wassgradientflows} with initial condition $\rho_{\ini}(\theta,t)$.
\end{itemize}
Then for any positive $\epsilon$, $\eta$, and $S$, there exists a constant $C>0$ that depends on $\rho_{\ini}(\theta,t)$ and $S$ such that when 
\[
M>\frac{C(\rho_{\ini}(\theta,t),S)}{\epsilon^2\eta}\,,\quad L>\frac{C(\rho_{\ini}(\theta,t),S)}{\epsilon}\,,
\]
we have
\[
\mathbb{P}\left(\left|E(\Theta_{L,M}(s))-E(\rho(\cdot,\cdot,s))\right|\leq \epsilon\right)\geq 1-\eta\,,\quad \forall s<S\,,
\]
where $E(\Theta_{L,M}(s))$ is defined in \eqref{eqn:costfunction} and $E(\rho(\cdot,\cdot,s))$ is defined in \eqref{eqn:costfunctioncont2}.
\end{theorem}

The proof of this result appears in Appendix \ref{sec:proofofthmconsistent}. 
This theorem suggests that for every fixed $S>0$, arbitrarily large, the cost generated by $\Theta_{L,M}$ that follows its gradient flow coded in GD is comparable (with high probability) to the cost generated by $\rho$, the continuous mean-field limit of $\Theta_{L,M}$, provided that $L$ and $M$ are sufficiently large. 
The size of the ResNet depends  negative-algebraically on $\epsilon$ (the desired accuracy) and $\eta$ (the confidence of success).
The essence of this theorem is that it matches the zero-loss property of the parameter configuration of a finite sized ResNet to its limiting PDE, whose analysis can be performed with some standard PDE tools.

The proof of Theorem~\ref{thm:consistent} divides naturally into two components: first, showing $E(\rho(\cdot,\cdot,s))\approx E(\Theta(s;\cdot))$ and second, showing $E(\Theta(s;\cdot))\approx E(\Theta_{L,M}(s))$ for all $s<S$. 
The former amounts to utilizing the mean-field limit theory and justifying that the particle trajectory $\theta_m(t,s)$ follows $\rho(\theta,t,s)$ for all $t$ in pseudo-time $s\in[0,S]$. 
The latter is to adopt the continuity in $t$ and trace the differences between $\theta_m(\frac{l}{L})$ and $\theta_{l,m}$. 
These two components of the proof are summarized in Theorems~\ref{thm:mean-field} and \ref{thm:contlimit}, respectively. 
According to the formula of the Fr\'echet derivatives~\eqref{eqn:Frechetderivative} and~\eqref{eqn:prho}, the estimates in these theorems naturally route through the boundedness of $p_{\rho}$, $p_\Theta$, $p_{\Theta_{L,M}}$, and similarly $Z_{\rho}$, $Z_{\Theta}$ and $Z_{\Theta_{L,M}}$. These bounds are collected in Theorem~\ref{thm:wmean-fieldlimit} and Lemmas~\ref{prop:wmean-fieldlimitfiniteL} and \ref{lem:boundofdiscrete}.

To intuitively see the equivalence between~\eqref{eqn:Wassgradientflowsdis} and~\eqref{eqn:Wassgradientflows}, we test them on the same smooth function $h(\theta)$. To test~\eqref{eqn:Wassgradientflows}, we multiply it on both sides and perform integration by parts to obtain:
\[
\frac{\rd}{\rd s}\int_{\mathbb{R}^k}h\rd\rho(\theta) = -\int_{\mathbb{R}^k}\nabla_\theta h\nabla_\theta\frac{\delta E_s(\rho(s))}{\delta\rho}\rd{\rho}\,.
\]
This is to say $\frac{\rd}{\rd s}\mathbb{E}(h)=-\EE\left(\nabla_\theta h\nabla_\theta\frac{\delta E_s(\rho(s))}{\delta\rho}\right)$. Testing $h$ on~\eqref{eqn:Wassgradientflowsdis} gives the same equivalence. Suppose $\rho = \frac{1}{M}\sum_{m=1}^M\delta_{\theta_m}$, then:
\[
\frac{\rd}{\rd s}\mathbb{E}(h) = \frac{1}{M}\sum_{m=1}^M\nabla_\theta h(\theta_m)\frac{\rd}{\rd s}\theta_m = -\sum_{m=1}^M\nabla_\theta h(\theta_m) \frac{\delta E_s(\Theta)}{\delta\theta_m}\,.
\]
The right hand side is also $-\EE\left(\nabla_\theta h\nabla_\theta\frac{\delta E_s(\rho(s))}{\delta\rho}\right)$ if and only if $M\frac{\delta E_s(\Theta(s))}{\delta \theta_m}=\nabla_\theta\frac{\delta E_s(\rho(s))}{\delta\rho}(\theta_m,t)$. This holds true, as presented in Lemma~\ref{lem:equivalence}. The proof in Theorem~\ref{thm:mean-field} gives the rigorous quantification in the convergence in $M$.

\subsection{Descending to the global minimum} \label{sec:global}
We investigate the convergence in $s$ of the $\rho$ system by studying the large-$s$ behavior of the limiting gradient flow equation~\eqref{eqn:Wassgradientflows}. Our main theorem is the following.
\begin{theorem}\label{thm:globalminimal}
Assume that Assumption~\ref{assum:f} holds and that $\rho_{\ini}(\theta,t)$ is admissible. 
Suppose that $\rho(\theta,t,s)\in \mathcal{C}([0,\infty);\mathcal{C}([0,1];\mathcal{P}^2))$ solves~\eqref{eqn:Wassgradientflows}. \revise{If the following two conditions hold:
\begin{itemize}
    \item  there exists a long-time limit $\rho_\infty(\theta,t)\in \mathcal{C}([0,1];\mathcal{P}^2)$ such that $\rho(\theta,t,s)$ converges to $\rho_\infty(\theta,t)$ in $\mathcal{C}([0,1];\mathcal{P}^2)$ as $s\rightarrow\infty$;
    \item the support of $\rho_{\infty}(\theta,t_0)$ is the full space $\mathbb{R}^k$ for some $t_0\in[0,1]$,
\end{itemize}
then $E(\rho_{\infty})=0$, so that $\rho_{\infty}$ is a global minimizer. }
More specifically, for any $\epsilon>0$, there is a constant $C_0>0$ depending only on $\rho_{\ini}(\theta,t)$ and $\epsilon$ such that
\[
E(\rho(s))\leq\epsilon,\quad \forall s>C_0(\rho_{\ini}(\theta,t),\epsilon).
\]
\end{theorem}

\revise{This theorem gives the conditions under which the long-time limit $\rho_\infty$, if achievable, becomes a global minimizer with zero value of the objective $E$.}
Its proof, which uses PDE analysis techniques, appears in Appendix~\ref{sec:proofofconvergencetoglobalminimal}.
The approach follows other similar results on the global minimum: For every  $\rho$ with $E(\rho)>0$, we can find a measure $\nu$ with $\int\rd\nu=0$ such that $\frac{\delta E(\rho)}{\delta\rho}$ in the direction of $\nu$ is negative, suggesting that a better probability measure than $\rho$ can be obtained by descending along the $\nu$ direction. 
This essentially equates a steady state with the global minimum; see Proposition~\ref{prop:localisglobal}. 
In addition, under the assumption that $\rho_\infty$ is supported in the entire space, we show that the global minimum is achievable with the flow suggested by the equation.

\begin{remark}[Assumptions in Theorem~\ref{thm:globalminimal}.]
\revise{There are two notable assumptions in Theorem 7: First, that $\rho(\theta,t,s)$ converges to $\rho_\infty(\theta,t)$ as $s\to\infty$, and second, that  $\rho_\infty$ is supported on the whole domain $\mathbb{R}^k$ for some time $t_0\in[0,1]$. These two assumptions are difficult to justify, but in Appendix~\ref{sec:proofofdiscussion} we describe some scenarios under which both conditions hold, though these scenarios are themselves stringent. We leave an investigation of weaker conditions under which the limit exists to future research.}

\revise{When the neural network has two layers, with $M\to\infty$, \citep{NEURIPS2018_a1afc58c} shows that the  limiting gradient flow equation can achieve the global minimum without the full-support condition. However, they require $f$ to be 2-homogeneous or partially 1-homogeneous (in $\theta$). They also assume the convergence of $\rho(\theta,s)$ to $\rho_\infty$ in the $s\to\infty$ limit. Without additional structure in the activation function, proving global convergence is a challenging task that we leave to future work.}
\end{remark}


\subsection{Main result: global convergence of finite-sized ResNet}\label{sec:main}
We are now ready to present the main result of the paper. Combining Theorems~\ref{thm:consistent} and \ref{thm:globalminimal} we naturally obtain the main theorem that justifies the overparameterizing deep ResNet architecture.
\begin{theorem}\label{thm:maintheorem}
Suppose that the conditions in Theorem~\ref{thm:consistent} and~\ref{thm:globalminimal} hold. Then for any positive $\epsilon$ and $\eta$, there exist positive constants $C_0$ depending on $\rho_{\ini}(\theta,t),\epsilon$ and $C$ depending on $\rho_{\ini}(\theta,t),s$ such that
 when 
\[
s>C_0(\rho_{\ini}(\theta,t),\epsilon)\,,\quad M>\frac{C(\rho_{\ini}(\theta,t),s)}{\epsilon^2\eta}\,,\quad L>\frac{C(\rho_{\ini}(\theta,t),s)}{\epsilon}\,,
\]
we have
\[
\mathbb{P}\left(\left|E(\Theta_{L,M}(s))\right|\leq \epsilon\right)\geq 1-\eta\,.
\]


\end{theorem}
This theorem states the final result of the paper: If ResNet is big enough, the gradient descent, after running long enough time, finds nearly-zero loss solution with high probability. The size of the ResNet algebraically depends on the accuracy and confidence level.

 



\section{Well-posedness and mean-field result}\label{sec:well_posed}
Some a priori estimates on the well-posedness of the dynamical system are necessary in the course of establishing the long-time convergence result. We list well-posedness results in this section. We show that the dynamical system for $z$ has a unique and stable solution, and the gradient flow of the parameter configuration is also well-posed. 
\subsection{Well-posedness of the OIE~\eqref{eqn:meancontRes}}\label{sec:wellposemean-field}
As $L$ and $M$ approach $\infty$, $z$ satisfies the ordinary-integral equation presented in~\eqref{eqn:meancontRes}. We justify that this differential equation is well-posed, in the sense that the solution is unique and stable.
\begin{theorem}\label{thm:wmean-fieldlimit}
Suppose that Assumption~\ref{assum:f} holds, that $x$ is in the support of $\mu$, and that $\rho_1,\rho_2\in \mathcal{C}([0,1];\mathcal{P}^2)$. Then \eqref{eqn:meancontRes} has a unique $\mathcal{C}^1$ solution. Further, this solution is stable, in the sense that 
\begin{subequations}
\begin{align} 
\label{boundofxsolution}
\left|Z_{\rho_1}(t;x)\right| & \leq C(\mathcal{L}_1) \\
\label{stabilityofxsolution}
\left|Z_{\rho_1}(t;x)-Z_{\rho_2}(t;x)\right|& \leq C(\mathcal{L}_1,\mathcal{L}_2)d_1(\rho_1,\rho_2)\,, \quad \mbox{for all $t \in [0,1]$.}
\end{align}
\end{subequations}
Here $C(\cdot)$ are constants depending only on the given arguments, and $\mathcal{L}_i$ are the second moments of $\rho_i$ ($i=1,2$), defined by  $\mathcal{L}_{i}=\int^1_0\int_{\mathbb{R}^k}|\theta|^2d\rho_i(\theta,t)\rd t$.
\end{theorem}
The proof of this result appears in Appendix~\ref{sec:proofofwell-posed}. The theorem suggests that a small perturbation to $\rho$ is linearly reflected in $Z_\rho$, the solution to~\eqref{eqn:meancontRes}. Thus, a small perturbation in the parameterization of the ResNet leads only to a small perturbation to the ResNet output. 

\subsection{Well-posedness of the adjusted gradient flow}\label{sec:wellposegradientflow}
The gradient flow of the parametrization is also well-posed, in both the discrete setting and the continuous mean-field limit.

In the discrete setting, we have the following.
\begin{theorem}\label{thm:finiteLwell-posed1} 
\revise{Suppose that Assumption~\ref{assum:f} holds, then \eqref{eqn:classicalgradientflowfiniteL} has a unique solution.}
\end{theorem}

Next, we show that equation~\eqref{eqn:Wassgradientflows}, which characterizes the dynamics of the continuous mean-field limit of the parameter configuration, is also well-posed.
\begin{theorem}\label{thm:Wassgradientflows}
Suppose that Assumption~\ref{assum:f} holds and that $\rho_{\ini}(\theta,t)$ is admissible. Then \eqref{eqn:Wassgradientflows} has a unique solution $\rho(\theta,t,s)$ in $\mathcal{C}([0,\infty);\mathcal{C}([0,1];\mathcal{P}^2))$ with initial condition $\rho_{\ini}(\theta,t)$.
Furthermore, \revise{$\rho(\theta,t,s)$ is admissible for any $s$ and}
\begin{equation}\label{eqn:decayEsmean-field}
\frac{\rd E_s(\rho(\cdot,\cdot,s))}{\rd s}\leq 0\,.
\end{equation}
\end{theorem}

The proofs of both theorems can be found in Appendix~\ref{sec:proofofthmWassgradientflows}.

\begin{remark}\label{rmk:equivalence}
Although we do not directly show the well-posedness of \eqref{eqn:contRes} and \eqref{eqn:Wassgradientflowsdis}, it is effectively a corollary of Theorem \ref{thm:wmean-fieldlimit} and \ref{thm:Wassgradientflows} above. 
One way to connect them is to reformulate the discrete probability distribution as
\begin{equation}\label{eqn:dis_rho}
\rho^{\dis}_{\Theta}(\theta,t,s)=\frac{1}{M}\sum^M_{m=1}\delta_{\theta_m(s;t)}(\theta)\,,
\end{equation}
where $\Theta(s;t)=\{\theta_{m}(s;t)\}_{m=1}^M$ is the list of trajectories. Since $\theta_m(0;t)$ is continuous in $t$, we have $\rho^{\dis}_{\Theta}(\theta,0;t)\in \mathcal{C}([0,1];\mathcal{P}^2)$. Because
\[
\frac{1}{M}\sum^M_{m=1}f(z(t;x),\theta_m(0;t))=\int_{\mathbb{R}^k}f(z(t;x),\theta)\rd \rho^{\dis}_{\Theta}(\theta,t,0)\,,
\]
using Theorem \ref{thm:wmean-fieldlimit}, \eqref{eqn:contRes} has a unique $\mathcal{C}^1$ solution $Z_\Theta(t;x)$ when $\Theta(t)$ is continuous. Furthermore, according to the definition~\eqref{eqn:cost_M_s} and~\eqref{eqn:cost_s}, we have
\[
E_s\left(\rho^{\dis}_{\Theta}(\cdot,\cdot,s)\right)=E_s\left(\Theta(s;\cdot)\right)\,.
\]
As a consequence, if $\Theta(s;t)$ satisfies~\eqref{eqn:Wassgradientflowsdis}, then $\rho^{\dis}_{\Theta}$ satisfies~\eqref{eqn:Wassgradientflows}, and vice versa. The well-posedness result in Theorem~\ref{thm:Wassgradientflows} for~\eqref{eqn:Wassgradientflows} then can be extended to justify the well-posedness of \eqref{eqn:Wassgradientflowsdis}. The rigorous proof for this comment can be found in Remark~\ref{re:G.1}.
\end{remark}

\section{Conclusion}
First-order methods such as gradient descent can find a global optimizer that provides zero-loss for fitting deep neural network. We explain this mechanism for ResNet in the limiting regime when both the number of layers and the number of weights per layer approach infinity. We show that GD of parameter configuration of ResNet can be translated to a gradient flow of the limiting probability distribution. \revise{Furthermore, we give an estimate of the size of such ResNets. Moreover, we show that under some conditions, the limiting gradient flow captures the zero-loss as its global minimum.}

In future work, we would like to relax the assumptions for the activation function and the fully-supported condition required in Theorem~\ref{thm:globalminimal}. Furthermore, the regularized version $E_s$ is used as a replacement of $E$ for technical reasons, but we believe such strategy can be relaxed as well.
\acks{Q.L. acknowledges support from Vilas Early Career award. The research of Z.D., S.C., and Q.L. is supported in part by NSF-CAREER-1750488 and Office of the Vice Chancellor for Research and Graduate Education at the University of Wisconsin Madison with funding from the Wisconsin Alumni Research Foundation. The work of Z.D., S.C., Q.L., and S.W. is supported in part by NSF via grant DMS-2023239. S.W. also acknowledges support from NSF Awards 1628384, 1740707, 1839338, and 1934612, and Subcontract 8F-30039 from Argonne National Laboratory. The authors also acknowledge two anonymous reviewers for insightful suggestions.}
\newpage

\appendix
\section*{Appendix}
This appendix contains proofs and supporting analysis for the  theorems in the main text. 
We start by proving well-posedness of the ResNet ODE and the gradient flow. This is followed by the global convergence of the gradient-flow PDE. Finally, we prove the validity of continuous and mean-field limit.

The specific organization of the appendix is as follows.
\begin{enumerate}[wide,labelindent=0pt]
    \item[Appendix \ref{sec:proofofwell-posed}:] Proof of Theorem~\ref{thm:wmean-fieldlimit}: Well-posedness of the ResNet ODE in its continuous limit~\eqref{eqn:contRes}, and the mean-field limit~\eqref{eqn:meancontRes}.
    \item[Appendix \ref{sec:prioriestimation}:] Preparation of a-priori estimates, subsequently used to show  well-posedness of the gradient flow equation.
    \item[Appendix \ref{sec:proofofthmWassgradientflows}:] Proof of Theorem~\ref{thm:finiteLwell-posed1} and \ref{thm:Wassgradientflows}: Well-posedness of gradient flows~\eqref{eqn:classicalgradientflowfiniteL} and~\eqref{eqn:Wassgradientflows}. 
    \item[Appendix \ref{sec:proofofconvergencetoglobalminimal}, \ref{sec:proofofdiscussion}:] Proof of Theorem~\ref{thm:globalminimal}: Global convergence of the gradient flow. (Section~\ref{sec:proofofdiscussion} demonstrates validity of the assumption we made in the statement of Theorem~\ref{thm:globalminimal}.)
    \item[Appendix \ref{sec:proofofthmconsistent}-\ref{sec:proofofcontlimit}:] Proof of Theorem~\ref{thm:consistent}. Section~\ref{sec:proofofthmconsistent}  lays out the structure of the proof, Section~\ref{sec:proofofthmmean-field} shows the continuous limit, and Section~\ref{sec:proofofcontlimit} shows the mean-field limit.
\end{enumerate}

The analytical core of the paper lies in Appendices~\ref{sec:proofofconvergencetoglobalminimal} and \ref{sec:proofofthmconsistent}, which describe properties of the gradient flow PDEs and explain why the gradient descent method for ResNet can be explained by these gradient flow equations. The technical results of Appendix~\ref{sec:proofofwell-posed}-\ref{sec:proofofthmWassgradientflows} can be skipped by readers who are interested to proofs of the main results.


Throughout, we denote by $C(\cdot)$ a generic constant that depends on its arguments $(\cdot)$. 
The precise value of this  constant may change at each instance. 

\section{Proof of Theorem~\ref{thm:wmean-fieldlimit}}\label{sec:proofofwell-posed}
First, we study well-posedness of the dynamical system by proving Theorem~\ref{thm:wmean-fieldlimit}. Recall the equation:
\begin{equation}\label{eqn:meancontRes_2}
\frac{\rd Z_\rho(t;x)}{\rd t}=F(Z_\rho,t)\,,\quad \forall t\in[0,1]\;\text{with}\; z(0;x)=x\,,
\end{equation}
where for a given $\rho\in \mathcal{C}([0,1];\mathcal{P}^2)$ we use the notation
\begin{equation} \label{eqn:defFf}
    F(z,t)=\int_{\mathbb{R}^k}f(z,\theta)\rd\rho(\theta,t)\,.
\end{equation}

The proof of Theorem~\ref{thm:wmean-fieldlimit} relies on the classical Lipschitz condition for the well-posedness of an ODE.
\begin{proof}[Proof of Theorem~\ref{thm:wmean-fieldlimit}]
Since $\rho\in \mathcal{C}([0,1];\mathcal{P}^2)$, we obtain that
\[
\sup_{0\leq t\leq 1}\, \int_{\mathbb{R}^k}|\theta|^2\rd\rho(\theta,t)<C<\infty\,.
\]
This implies that $\int_{\mathbb{R}^k}|\theta|\rd\rho(\theta,t)<\sqrt{C}$ for all $t$. 

For any $t\in[0,1]$, using \eqref{eqn:derivativebound} from Assumption~\ref{assum:f}, we have
\begin{equation}\label{eqn:boundoff}
|f(z,\theta)|\leq C(|\theta|+1)(|z|+1)\,,\quad |f(z_1,\theta)-f(z_2,\theta)|\leq C_1|\theta||z_1-z_2|\,.
\end{equation}
Then, with the boundedness of the first moment of $\theta$, we have 
\begin{equation}\label{eqn:FLipschitz}
\begin{aligned}
\left|F(z_1,t)-F(z_2,t)\right| & \leq \left|\int_{\mathbb{R}^k} \left(f(z_1,\theta)-f(z_2,\theta)\right)\rd\rho(\theta,t)\right|\\
& \leq C_1\left|\int_{\mathbb{R}^k}|\theta|\rd\rho(\theta,t)\right||z_1-z_2|<C|z_1-z_2|\,,
\end{aligned}
\end{equation}
meaning that $F(z,t)$ is uniformly Lipschitz in $z$ for all $t\in[0,1]$. Therefore~\eqref{eqn:meancontRes} has a unique $\mathcal{C}^1$ solution.

To show the boundedness in~\eqref{boundofxsolution}, we have from~\eqref{eqn:boundoff} and \eqref{eqn:defFf} that
\begin{equation}\label{eqn:boundofF}
|F(z,t)|\leq C(|z|+1)\int_{\mathbb{R}^k}(|\theta|+1)\rd\rho(\theta,t)\,.
\end{equation}
Multiplying~\eqref{eqn:meancontRes_2} and using~\eqref{eqn:boundofF}, we obtain
\[
\begin{aligned}
&\frac{\rd |Z_{\rho_1}(t;x)|^2}{\rd t}\\
& \leq 2C\left(|Z_{\rho_1}|^2+|Z_{\rho_1}|\right)\int_{\mathbb{R}^k}(|\theta|+1)\rd\rho_1(\theta,t)\\
& \leq 4C\int_{\mathbb{R}^k}(|\theta|+1)\rd\rho_1(\theta,t)\left(|Z_{\rho_1}(t;x)|^2+1\right)\,,
\end{aligned}
\]
Using Gr\"onwall's inequality, we have 
\begin{align*}
|Z_{\rho_1}(t;x)|& \leq\exp\left(2C\left(\int^1_0\int_{\mathbb{R}^k}|\theta|\rd\rho_1\rd t+1\right)\right)(|x|+1)\\
& \leq \exp\left(2C\left(\left(\int^1_0\int_{\mathbb{R}^k}|\theta|^2d\rho_1\rd t\right)^{1/2}+1\right)\right)(|x|+1) \\
& \leq \exp\left(2C\left(\sqrt{\mathcal{L}_1}+1\right)\right)(|x|+1)\,,
\end{align*}
where $\mathcal{L}_1=\int_0^1\int_{\mathbb{R}^k}|\theta|^2\rd\rho_1\rd t$. Finally, to prove the stability result~\eqref{stabilityofxsolution}, we define for  $\rho_1,\rho_2\in \mathcal{C}([0,1];\mathcal{P}^2)$ the notation
\[
\Delta(t;x)=Z_{\rho_1}(t;x)-Z_{\rho_2}(t;x)\,.
\]
Then we have (using \eqref{eqn:FLipschitz}) that
\begin{equation}\label{boundofDeltatZrho}
\begin{aligned}
& \frac{\rd |\Delta(t;x)|^2}{\rd t} \\
& =2\left\langle \Delta(t;x), \int_{\mathbb{R}^k}f(Z_{\rho_1}(t;x),\theta)\rd\rho_1(\theta,t)-\int_{\mathbb{R}^k}f(Z_{\rho_2}(t;x),\theta)\rd\rho_2(\theta,t)\right\rangle\\
& =2\left\langle \Delta(t;x), \int_{\mathbb{R}^k}f(Z_{\rho_1}(t;x),\theta)\rd\rho_1(\theta,t)-\int_{\mathbb{R}^k}f(Z_{\rho_2}(t;x),\theta)\rd\rho_1(\theta,t)\right\rangle\\
& \quad +2\left\langle \Delta(t;x), \int_{\mathbb{R}^k}f(Z_{\rho_2}(t;x),\theta)\rd\rho_1(\theta,t)-\int_{\mathbb{R}^k}f(Z_{\rho_2}(t;x),\theta)\rd\rho_2(\theta,t)\right\rangle\\
& \leq 2C_1|\Delta(t;x)|^2\int_{\mathbb{R}^k}|\theta|\rd\rho_1(\theta,t)+2C_1|\Delta(t;x)|(|Z_{\rho_2}(t;x)|+1)d_1(\rho_1,\rho_2)\\
& \leq 2C_1\left(\int_{\mathbb{R}^k}|\theta|\rd\rho_1(\theta,t)+1\right)|\Delta(t;x)|^2+2C_1(|Z_{\rho_2}(t;x)|+1)^2d^2_1(\rho_1,\rho_2)\,,
\end{aligned}
\end{equation}
where in the first inequality we first use the mean-value theorem and Assumption \ref{assum:f} \eqref{eqn:derivativebound} to obtain
\[
|f(Z_{\rho_1}(t;x),\theta)-f(Z_{\rho_2}(t;x),\theta)|\leq C_1|\theta||Z_{\rho_1}(t;x)-Z_{\rho_2}(t;x)|=C_1|\theta||\Delta(t;x)|\,.
\]
Next, let $\theta_1\sim \rho_1(\theta,t)$ and $\theta_2\sim \rho_2(\theta,t)$ such that $\left(\EE\left(|\theta_1-\theta_2|^2\right)\right)^{1/2}=W_2(\rho_1(\cdot,t),\rho_2(\cdot,t))$. Similarly, we also have
\[
\begin{aligned}
&\left|\int_{\mathbb{R}^k}f(Z_{\rho_2}(t;x),\theta)\rd\rho_1(\theta,t)-\int_{\mathbb{R}^k}f(Z_{\rho_2}(t;x),\theta)\rd\rho_2(\theta,t)\right|\\
\leq &\EE\left(\left|f(Z_{\rho_2}(t;x),\theta_1)-f(Z_{\rho_2}(t;x),\theta_2)\right|\right)\\
\leq &C_1\left(|Z_{\rho_2}(t;x)|+1\right)\EE\left(|\theta_1-\theta_2|\right)\\
\leq &C_1\left(|Z_{\rho_2}(t;x)|+1\right)\left(\EE\left(|\theta_1-\theta_2|^2\right)\right)^{1/2}\\
\leq &C_1\left(|Z_{\rho_2}(t;x)|+1\right)d_1(\rho_1,\rho_2)\\
\end{aligned}\,,
\]

Finally, since $|\Delta(0;x)|=0$, we have
\[
|\Delta(t;x)|\leq 2C_1d^2_1(\rho_1,\rho_2)\exp\left(2C_1\left(\sqrt{\mathcal{L}_1}+t\right)\right)\int^t_0(|Z_{\rho_2}(t;x)|+1)^2
\rd s\,
\]
where we use Gr\"onwall's inequality and $\int^1_0\int_{\mathbb{R}^k}|\theta|\rd\rho_1(\theta,t)dt\leq \sqrt{\mathcal{L}_1}$ by H\"older's inequality.
By substituting $|Z_{\rho_2}(t;x)|$ from~\eqref{boundofxsolution}, we complete the proof of~\eqref{stabilityofxsolution}.
\end{proof}

\section{A priori estimates of the cost function}\label{sec:prioriestimation}

Some a priori estimates are necessary in the proof for the main theorems. We prove two lemmas here.

For $E_s$, standard techniques (see~\citep{pmlr-v119-lu20b}) show that
\begin{equation}\label{eqn:frechet_s}
\frac{\delta E_s(\rho)}{\delta \rho}(\theta,t)=\mathbb{E}_{x\sim\mu}\left(p^\top_\rho(t;x)f(Z_\rho(t;x),\theta)\right)+e^{-s}|\theta|^2\,,
\end{equation}
where $p_\rho(t;x)$ is the solution to the following ODE:
\begin{equation}\label{eqn:prhoappendix}
\left\{
\begin{aligned}
&\frac{\partial p^\top_\rho}{\partial t}=-p^\top_\rho\int_{\mathbb{R}^k}\partial_zf(Z_\rho(t;x),\theta)\rd\rho(\theta,t)\,,\\
&p_\rho(1;x)=\left(g(Z_{\rho}(1;x))-y(x)\right)\nabla g(Z_{\rho}(1;x))\,.
\end{aligned}
\right.
\end{equation}
According to assumption~\eqref{mucompactsupport}, $\mu(x)$ has a compact support with $|x|<R$. 

Our first lemma shows that $p_\rho$ is Lipschitz continuous with respect to $\rho$.
\begin{lemma}\label{lem:prhoprop1}
Suppose that Assumption \ref{assum:f} holds and  that $x$ in the support of $\mu$. Suppose that $\rho_1,\rho_2\in \mathcal{C}([0,1];\mathcal{P}^2)$ with corresponding solutions $p_{\rho_1}$, $p_{\rho_2}$ of \eqref{eqn:prhoappendix}, and denote
\[
\mathcal{L}_i=\int^1_0\int_{\mathbb{R}^k}|\theta|^2\rd\rho_i(\theta,t)\rd t\,.
\]
Then the following two bounds are satisfied:
\begin{subequations}
\begin{align}
\label{boundofprho}
\left|p_{\rho_1}(t;x)\right|& \leq C(\mathcal{L}_1)\,, \\
\label{stabilityofprho}
\left|p_{\rho_1}(t;x)-p_{\rho_2}(t;x)\right|& \leq C(\mathcal{L}_1,\mathcal{L}_2)d_1(\rho_1,\rho_2)\,.
\end{align}
\end{subequations}
\end{lemma}
\begin{proof}
From \eqref{eqn:derivativebound} in Assumption \ref{assum:f}, with $i=1$ and $j=0$, we have
\begin{equation}\label{boundofrateprho}
\left|\int_{\mathbb{R}^k}\partial_zf(Z_{\rho_1}(t;x),\theta)\rd\rho_1(\theta,t)\right|\leq C_1\int_{\mathbb{R}^k}|\theta|\rd\rho_1(\theta,t)\,.
\end{equation}
it follows from the initial conditions of~\eqref{eqn:prhoappendix} that
\begin{equation} \label{eq:xj1}
\left| p_{\rho_1}(1;x) \right|\leq C(|Z_{\rho_1}(1,x)|+1)\leq C(\mathcal{L}_1)\,.
\end{equation}
where we use Assumption \ref{assum:f} in the first inequality and \eqref{boundofxsolution} in the second inequality. Noting that~\eqref{eqn:prhoappendix} is a linear equation,~\eqref{boundofprho} follows naturally when we combine  \eqref{boundofrateprho} with \eqref{eq:xj1} and use $\int^1_0\int_{\mathbb{R}^k}|\theta|\rd\rho_1(\theta,t)\rd t\leq \mathcal{L}^{1/2}_1$. 


To prove \eqref{stabilityofprho}, we define
\[
\Delta(t;x)=p_{\rho_1}(t;x)-p_{\rho_2}(t;x)\,.
\]
First, when $t=1$ and $|x|<R$, we have from the initial conditions that
\begin{equation}\label{eqn:ODEforDEltaprhoinitial}
\begin{aligned}
|\Delta(1;x)|=&|p_{\rho_1}(1;x)-p_{\rho_2}(1;x)|\\
=&\left|\left(g(Z_{\rho_1}(1;x))-y(x)\right)\nabla g(Z_{\rho_1}(1;x))-\left(g(Z_{\rho_2}(1;x))-y(x)\right)\nabla g(Z_{\rho_2}(1;x))\right|\\
\leq &C(\mathcal{L}_1,\mathcal{L}_2)|Z_{\rho_1}(1;x)-Z_{\rho_2}(1;x)|\leq C(\mathcal{L}_1,\mathcal{L}_2)d_1(\rho_1,\rho_2)\,,
\end{aligned}
\end{equation}
where we use Assumption~\ref{assum:f}, \eqref{boundofxsolution}, and $|x|<R$ in the first inequality and \eqref{stabilityofxsolution} in the second inequality.
The following ODE is satisfied by $\Delta$:
\begin{equation}\label{eqn:ODEforDEltaprho}
\frac{\partial \Delta^\top(t;x)}{\partial t}=-\Delta^\top(t;x)\int_{\mathbb{R}^k}\partial_zf(Z_{\rho_1}(t;x),\theta)\rd\rho_1(\theta,t)+p^\top_{\rho_2}(t;x)D_{\rho_1,\rho_2}(t;x)\,,
\end{equation}
where 
\[
D_{\rho_1,\rho_2}(t;x)=\int_{\mathbb{R}^k}\partial_zf(Z_{\rho_2}(t;x),\theta)\rd\rho_2(\theta,t)-\int_{\mathbb{R}^k}\partial_zf(Z_{\rho_1}(t;x),\theta)\rd\rho_1(\theta,t)\,.
\]
The boundedness of $D_{\rho_1,\rho_2}(t;x)$ can be shown similarly to~\eqref{boundofDeltatZrho}, by first splitting  into two terms:
\begin{equation}\label{eqn:boundofDrho}
\begin{aligned}
|D_{\rho_1,\rho_2}(t;x)|\leq&
\underbrace{\left|\int_{\mathbb{R}^k}\partial_zf(Z_{\rho_2}(t;x),\theta)\rd\rho_2(\theta,t)-\int_{\mathbb{R}^k}\partial_zf(Z_{\rho_2}(t;x),\theta)\rd\rho_1(\theta,t)\right|}_{\mathrm{(I)}}\\
&+\underbrace{\left|\int_{\mathbb{R}^k} \left[ \partial_zf(Z_{\rho_2}(t;x),\theta) -
\partial_zf(Z_{\rho_1}(t;x),\theta) \right] \rd\rho_1(\theta,t)\right|}_{\mathrm{(II)}}\,.
\end{aligned}
\end{equation}
The bound of $\mathrm{(II)}$ relies on Assumption \ref{assum:f} and \eqref{stabilityofxsolution}:
\begin{equation}\label{eqn:boundofDrho2}
\begin{aligned}
\mathrm{(II)} &\leq C_1\left|Z_{\rho_1}(t;x)-Z_{\rho_2}(t;x)\right|\int_{\mathbb{R}^k}|\theta|^2d\rho_1(\theta,t) \\
& \leq C(\mathcal{L}_1,\mathcal{L}_2)\left(\int_{\mathbb{R}^k}|\theta|^2d\rho_1(\theta,t)\right)d_1(\rho_1,\rho_2).
\end{aligned}
\end{equation}

To bound $\textrm{(I)}$, noticing that $\rho_i$ are admissible, we let $\theta_1\sim \rho_1(\theta,t)$ and $\theta_2\sim \rho_2(\theta,t)$ such that $\left(\EE\left(|\theta_1-\theta_2|^2\right)\right)^{1/2}=W_2(\rho_1(\cdot,t),\rho_2(\cdot,t))$. We then have
\begin{equation}\label{eqn:boundofDrho1}
\begin{aligned}
\textrm{(I)}&\leq \EE\left(\left|\partial_zf(Z_{\rho_2}(t;x),\theta_1)-\partial_zf(Z_{\rho_2}(t;x),\theta_2)\right|\right)\\
&\leq C_1\left(|Z_{\rho_2}(t;x)|+1\right)\EE\left((|\theta_1|+|\theta_2|)|\theta_1-\theta_2|\right)\\
&\leq C(\mathcal{L}_2)\EE\left((|\theta_1|+|\theta_2|)|\theta_1-\theta_2|\right)\\
&\leq C(\mathcal{L}_2)\left(\EE\left(|\theta_1|^2+|\theta_2|^2\right)\right)^{1/2}\left(\EE\left(|\theta_1-\theta_2|^2\right)\right)^{1/2}\\
&\leq C(\mathcal{L}_2)\left(\EE\left(|\theta_1|^2+|\theta_2|^2\right)\right)^{1/2}d_1(\rho_1,\rho_2)\\
&= C(\mathcal{L}_2)\left(\int_{\mathbb{R}^k}|\theta|^2d\rho_1(\theta,t)+\int_{\mathbb{R}^k}|\theta|^2d\rho_2(\theta,t)\right)^{1/2}d_1(\rho_1,\rho_2)\,,
\end{aligned}\,
\end{equation}
where we use the mean-value theorem and Assumption \ref{assum:f} \eqref{eqn:derivativebound} in the second inequality: For some $\lambda \in (0,1)$, we have
\begin{equation}\label{eqn:boundpartialzfexample}
\begin{aligned}
\left|\partial_zf(Z_{\rho_2}(t;x),\theta_1)-\partial_zf(Z_{\rho_2}(t;x),\theta_2)\right|
& \leq |\partial_z\partial_\theta f(Z_{\rho_2}(t;x),\theta_1+\lambda\theta_2)||\theta_1-\theta_2|\\
& \leq C_1\left(|Z_{\rho_2}(t;x)|+1\right)(|\theta_1|+|\theta_2|)|\theta_1-\theta_2|.
\end{aligned}
\end{equation}
We use \eqref{boundofxsolution} in the third inequality. 

By substituting  \eqref{eqn:boundofDrho1} and \eqref{eqn:boundofDrho2} into \eqref{eqn:boundofDrho} and using H\"older's inequality, we obtain 
\[
|D_{\rho_1,\rho_2}(t;x)|\leq C(\mathcal{L}_1,\mathcal{L}_2)\left(\int_{\mathbb{R}^k}|\theta|^2d\rho_1(\theta,t)+\int_{\mathbb{R}^k}|\theta|^2d\rho_2(\theta,t)+1\right) d_1(\rho_1,\rho_2)\,.
\]
By using this bound in \eqref{eqn:ODEforDEltaprho}, using \eqref{boundofrateprho} and H\"older's inequality, we have
\begin{equation}\label{LemmaBexample}
\begin{aligned}
\frac{\rd |\Delta(t;x)|^2}{\rd t}\leq& C(\mathcal{L}_1,\mathcal{L}_2)\left(\int_{\mathbb{R}^k}|\theta|^2d\rho_1(\theta,t)+\int_{\mathbb{R}^k}|\theta|^2d\rho_2(\theta,t)+1\right)|\Delta(t;x)|^2\\
&+C(\mathcal{L}_1,\mathcal{L}_2)\left(\int_{\mathbb{R}^k}|\theta|^2d\rho_1(\theta,t)+\int_{\mathbb{R}^k}|\theta|^2d\rho_2(\theta,t)+1\right)d^2_1(\rho_1,\rho_2)\,.
\end{aligned}
\end{equation}
Finally, we use Gr\"onwall's inequality with the initial condition~\eqref{eqn:ODEforDEltaprhoinitial} to obtain \eqref{stabilityofprho}.
\end{proof}

Our second lemma concerns the continuity of~$\nabla_\theta \frac{\delta E_s(\rho)}{\delta \rho}$.
\begin{lemma}\label{lemmadeltaEs}  Assume $\rho,\rho_1,\rho_2\in \mathcal{C}([0,1];\mathcal{P}^2)$. Defining
\[
\mathcal{L}=\int^1_0\int_{\mathbb{R}^k}|\theta|^2d\rho(\theta,t)\rd t,\quad  \widetilde{\mathcal{L}}=\sup\left\{\int^1_0\int_{\mathbb{R}^k}|\theta|^2d\rho_1(\theta,t)\rd t,\int^1_0\int_{\mathbb{R}^k}|\theta|^2d\rho_2(\theta,t)\rd t\right\}\,,
\]
then for any $(\theta_1,t_1),(\theta_2,t_2)\in \mathbb{R}^k\times[0,1]$ and $s>0$, we have
\begin{itemize}
\item Boundedness:
\begin{equation}\label{bound}
\left|\nabla_\theta\frac{\delta E_s(\rho)}{\delta \rho}(\theta_1,t_1)\right|\leq 2|\theta_1|+C(\mathcal{L})\,,
\end{equation}
\item Lipschitz continuity in $\theta$ and $t$:
\begin{equation}
\label{Lipshictz}
\begin{aligned}
& \left|\nabla_\theta\frac{\delta E_s(\rho)}{\delta \rho}(\theta_1,t_1)-\nabla_\theta\frac{\delta E_s(\rho)}{\delta \rho}(\theta_2,t_2)\right|\\
& \quad \leq C(\mathcal{L})\left(|\theta_1-\theta_2|+(|\theta_2|+1)|t_1-t_2|^{1/2}\right)\,,
\end{aligned}
\end{equation}
\item Lipschitz continuity in $\rho$:
\begin{equation}\label{eqn:Deltabound2}
\left|\nabla_\theta\frac{\delta E_s(\rho_1)}{\delta \rho}(\theta,t)-\nabla_\theta\frac{\delta E_s(\rho_2)}{\delta \rho}(\theta,t)\right|\leq C(\widetilde{\mathcal{L}})d_1(\rho_1,\rho_2)\left(|\theta|+1\right)\,,
\end{equation}
where $d_1$ is defined in \eqref{def:path}.
\end{itemize}
\end{lemma}
\begin{proof} 
From \eqref{eqn:frechet_s}, we have
\begin{equation} \label{eq:fs2}
\nabla_\theta\frac{\delta E_s(\rho)}{\delta \rho}(\theta,t)=\mathbb{E}_{x\sim\mu}\left(p_\rho^\top(t;x)\partial_\theta f(Z_\rho(t;x),\theta)\right)+2e^{-s}\theta\,,
\end{equation}
which gives
\[
\left|\nabla_\theta\frac{\delta E_s(\rho)}{\delta \rho}(\theta,t)\right|\leq \mathbb{E}_{x\sim\mu}\left(\left\|\partial_\theta f(Z_\rho(t;x),\theta)\right\|\left|p_\rho(t;x)\right|\right)+2\left|\theta\right|\leq 2|\theta|+C(\mathcal{L})\,,
\]
where we use \eqref{mucompactsupport}, \eqref{boundofxsolution}, and \eqref{boundofprho} in the second inequality. Thus, \eqref{bound} is proved.

To prove \eqref{Lipshictz},  we use assume $t_1>t_2$  without loss of generality, use \eqref{eq:fs2}, and reformulate bound the left-hand side of \eqref{Lipshictz} as follows:
\begin{equation}\label{eqn:boundofdeltaE}
\begin{aligned}
&\left|\nabla_\theta\frac{\delta E_s(\rho)}{\delta \rho}(\theta_1,t_1)-\nabla_\theta\frac{\delta E_s(\rho)}{\delta \rho}(\theta_2,t_2)\right|\\
& \leq \underbrace{\left|\mathbb{E}_{x\sim\mu}\left(p^\top_\rho(t_1;x)\partial_\theta f(Z_\rho(t_1;x),\theta_1)-p^\top_\rho(t_1;x)\partial_\theta f(Z_\rho(t_1;x),\theta_2)\right)\right|}_{\mathrm{(I)}}\\
& \quad +\underbrace{\left|\mathbb{E}_{x\sim\mu}\left(p^\top_\rho(t_1;x)\partial_\theta f(Z_\rho(t_1;x),\theta_2)-p^\top_\rho(t_1;x)\partial_\theta f(Z_\rho(t_2;x),\theta_2)\right)\right|}_{\mathrm{(II)}}\\
& \quad +\underbrace{\left|\mathbb{E}_{x\sim\mu}\left(p^\top_\rho(t_1;x)\partial_\theta f(Z_\rho(t_2;x),\theta_2)-p^\top(t_2;x)\partial_\theta f(Z_\rho(t_2;x),\theta_2)\right)\right|}_{\mathrm{(III)}}+2e^{-s}|\theta_1-\theta_2|\,.
\end{aligned}
\end{equation}
The bounds on all three terms in this expression rely on \eqref{eqn:derivativebound} with the mean-value theorem in Assumption \ref{assum:f}. For $\textrm{(I)}$, we have
\begin{align*}
|\partial_\theta f(Z_\rho(t_1;x),\theta_1)-\partial_\theta f(Z_\rho(t_1;x),\theta_2)|
& \leq  |\partial^2_\theta f(Z_\rho(t_1;x),(1-\lambda)\theta_1+\lambda\theta_2)||\theta_1-\theta_2|\\
& \leq C_1(|Z_\rho(t_1;x)|+1)^2|\theta_1-\theta_2|\\
& \leq C(\mathcal{L})|\theta_1-\theta_2|\,,
\end{align*}
where we use \eqref{eqn:derivativebound} with mean-value theorem $(\lambda\in[0,1])$ in the first two inequalities, and \eqref{boundofxsolution} in the last inequality.

To bound term $\textrm{(II)}$, we have from \eqref{eqn:meancontRes_2} and \eqref{eqn:boundoff} that 
\begin{equation*}
\begin{aligned}
\left|Z_\rho(t_1;x)-Z_\rho(t_2;x)\right|
& \leq \left|\int^{t_1}_{t_2}\int_{\mathbb{R}^k}f(Z_\rho(t;x),\theta)\rd\rho(\theta,t)\rd t\right|\\
& \leq \int^{t_1}_{t_2}\int_{\mathbb{R}^k}\left|f(Z_\rho(t;x),\theta)\right|\rd\rho(\theta,t)\rd t \\
& \leq C\int^{t_1}_{t_2}\int_{\mathbb{R}^k}(\left|Z_\rho(t;x)\right|+1)(\left|\theta\right|+1)\rd\rho(\theta,t)\rd t \\
&  \stackrel{\mathrm{(a)}}{\leq} C(\mathcal{L})\left(\int^{t_1}_{t_2}\int_{\mathbb{R}^k}(\left|\theta\right|+1)\rd\rho(\theta,t)\rd t\right) \\
&\stackrel{\mathrm{(b)}}{\leq} C(\mathcal{L})\left(\int^{t_1}_{t_2}\int_{\mathbb{R}^k}(\left|\theta\right|+1)^2\rd\rho(\theta,t)\rd t\right)^{1/2}\left(\int^{t_1}_{t_2}\int_{\mathbb{R}^k}1\rd\rho(\theta,t)\rd t\right)^{1/2}\\
& \stackrel{\mathrm{(c)}}{\leq} C(\mathcal{L})\sqrt{t_1-t_2}\,,
\end{aligned}
\end{equation*}
where we use \eqref{boundofxsolution} in (a), H\"older's inequality in (b), and $\int_{\mathbb{R}^k}d\rho(\theta,t)=1$ for all $t\in[0,1]$ in (c). This implies
\[
\begin{aligned}
\mathrm{(II)}&\leq \mathbb{E}_{x\sim\mu}\left(|p^\top_\rho(t_1;x)|\left|\partial_\theta f(Z_\rho(t_1;x),\theta_2)-\partial_\theta f(Z_\rho(t_2;x),\theta_2)\right|\right)\\
&\leq  C(\mathcal{L})\mathbb{E}_{x\sim\mu}\left(\left|\partial_\theta f(Z_\rho(t_1;x),\theta_2)-\partial_\theta f(Z_\rho(t_2;x),\theta_2)\right|\right)\\
&\leq  C(\mathcal{L})\mathbb{E}_{x\sim\mu}\left(\left|\theta_2\right|\left|Z_\rho(t_1;x)-Z_\rho(t_2;x)\right|\right)\leq C(\mathcal{L})|\theta_2|\sqrt{t_1-t_2}\,,
\end{aligned}
\]
where we use \eqref{boundofprho} in the first inequality, $|\partial_x\partial_\theta f(x,\theta)|\leq C_1|\theta|(|x|+1)$ with mean-value theorem and \eqref{boundofxsolution} in the second inequality.

To bound term $\textrm{(III)}$, we also use \eqref{eqn:prhoappendix} to obtain
\begin{equation*}
\begin{aligned}
|p_\rho(t_1;x)-p_\rho(t_2;x)|
& \leq \left|\int^{t_1}_{t_2}\int_{\mathbb{R}^k}p^\top_\rho(t;x)\partial_zf(Z_\rho(t;x),\theta)\rd\rho(\theta,t)\rd t\right| \\
&\leq C_1|p^\top_\rho(t;x)|\int^{t_1}_{t_2}\int_{\mathbb{R}^k}\left|\theta\right|\rd\rho(\theta,t)\rd t\\
& \leq C(\mathcal{L})\left|\int^{t_1}_{t_2}\int_{\mathbb{R}^k}\left|\theta\right|\rd\rho(\theta,t)\rd t\right| \\
& \leq C(\mathcal{L})\sqrt{t_1-t_2}\,,
\end{aligned}
\end{equation*}
where we use $|\partial_zf(Z_\rho(t;x),\theta)|<C_1|\theta|$ in the second inequality, \eqref{boundofprho} in the third inequality, and H\"older's inequality with $\int_{\mathbb{R}^k}d\rho(\theta,t)=1$ in the last inequality. 
This implies
\[
\mathrm{(III)}\leq \mathbb{E}_{x\sim\mu}\left(|\partial_\theta f(Z_\rho(t_2;x),\theta_2)||p_\rho(t_1;x)-p_\rho(t_2;x)|\right)\leq C(\mathcal{L})\sqrt{t_1-t_2}\,,
\]
where we also use $|\partial_\theta f(Z_\rho(t_2;x),\theta_2)|\leq C_1(|Z_\rho(t_2;x)|+1)\leq C(\mathcal{L})$ by \eqref{boundofxsolution}. These estimates together with~\eqref{eqn:boundofdeltaE} lead to~\eqref{Lipshictz}. 

Finally, to prove \eqref{eqn:Deltabound2}, we have from \eqref{eqn:frechet_s} that
\[
\begin{aligned}
&\left|\nabla_\theta\frac{\delta E_s(\rho_1)}{\delta \rho}(\theta,t)-\nabla_\theta\frac{\delta E_s(\rho_2)}{\delta \rho}(\theta,t)\right|\\
& =\mathbb{E}_{x\sim \mu}\left(\left|\partial_\theta f(Z_{\rho_1}(t;x),\theta)p_{\rho_1}(t;x)-\partial_\theta f(Z_{\rho_2}(t;x),\theta)p_{\rho_2}(t;x)\right|\right)\\
& \leq \mathbb{E}_{x\sim \mu}\left(\left|\partial_\theta f(Z_{\rho_1}(t;x),\theta)-\partial_\theta f(Z_{\rho_2}(t;x),\theta)\right||p_{\rho_1}(t;x)|\right)\\
&  \quad +\mathbb{E}_{x\sim \mu}\left(\left|\partial_\theta f(Z_{\rho_2}(t;x),\theta)\right|\left|p_{\rho_1}(t;x)-p_{\rho_2}(t;x)\right|\right)\\
& \leq C\left(\widetilde{\mathcal{L}}\right)\left(|\theta|\mathbb{E}_{x\sim \mu}\left(\left|Z_{\rho_1}(t;x)-Z_{\rho_2}(t;x)\right|\right)+\mathbb{E}_{x\sim \mu}\left(\left|p_{\rho_1}(t;x)-p_{\rho_2}(t;x)\right|\right)\right)\\
& \leq C\left(\widetilde{\mathcal{L}}\right)d_1(\rho_1,\rho_2)\left(|\theta|+1\right)\,.
\end{aligned}
\]
In the second inequality we use $\left|\partial_z\partial_\theta f(z,\theta)\right|\leq C_1|\theta|(|z|+1)$ with mean-value theorem and \eqref{boundofxsolution} to obtain that
\[
\begin{aligned}
&\left|\partial_\theta f(Z_{\rho_1}(t;x),\theta)-\partial_\theta f(Z_{\rho_2}(t;x),\theta)\right|\\
\leq &\left|\partial_z\partial_\theta f((1-\lambda)Z_{\rho_1}(t;x)+\lambda Z_{\rho_2}(t;x),\theta)\right|\left|Z_{\rho_1}(t;x)-Z_{\rho_2}(t;x)\right|\\
\leq &C\left(\widetilde{\mathcal{L}}\right)|\theta|\left|Z_{\rho_1}(t;x)-Z_{\rho_2}(t;x)\right|\,,
\end{aligned}
\]
where $\lambda\in[0,1]$.
\end{proof}


\section{Proof of Theorem~\ref{thm:Wassgradientflows} and Theorem~\ref{thm:finiteLwell-posed1}}\label{sec:proofofthmWassgradientflows}
Theorem~\ref{thm:Wassgradientflows} states the well-posedness of equation~\eqref{eqn:Wassgradientflows}, the continuous mean-field limit of the gradient flow. 
We verify that the solution is unique and the energy is nonincreasing.

The proof uses a fixed-point argument: We build a map and show that this map allows only a single fixed point. 
Fixing $S>0$, for any $\phi(\theta,t,s)\in \mathcal{C}\left([0,S];\mathcal{C}([0,1];\mathcal{P}^2)\right)$ and $\phi(\theta,t,0)=\rho_{\ini}(\theta,t)$, we define a map
\begin{equation}\label{eqn:map}
\varphi=\mathcal{T}_S(\phi):\mathcal{C}\left([0,S];\mathcal{C}([0,1];\mathcal{P}^2) \right)\rightarrow \mathcal{C}\left([0,S];\mathcal{C}([0,1];\mathcal{P}^2)\right)
\end{equation}
where $\varphi$ solves:
\begin{equation}\label{eqn:varphi}
\left\{
\begin{aligned}
\frac{\partial \varphi(\theta,t,s)}{\partial s} & =\nabla_\theta\cdot\left(\varphi(\theta,t,s)\nabla_\theta\frac{\delta E_s(\phi(s))}{\delta \rho}(\theta,t)\right)\,,\\
\varphi(\theta,t,0) &=\rho_{\ini}(\theta,t)\,.
\end{aligned}\right.
\end{equation}
The proof is divided into three steps:
\begin{itemize}
\item[Step 1:]
We show this map is well-defined and maps a $\mathcal{C}\left([0,S];\mathcal{C}([0,1];\mathcal{P}^2)\right)$ measure to a $\mathcal{C}\left([0,S];\mathcal{C}([0,1];\mathcal{P}^2)\right)$ measure. Furthermore, if the second moment of $\phi$ is bounded, then $\mathcal{T}(\phi)$ also has a bounded second moment, and we can specify such boundedness.
\item[Step 2:] We give a bound of $d_2(\mathcal{T}_S(\phi_1),\mathcal{T}_S(\phi_2))$ using $d_2(\phi_1\,,\phi_2)$. Furthermore, the associated energy is non-increasing.\\
By combining steps 1 and 2, we can find a small enough $S$ so that $\mathcal{T}_S$ is a contraction map in a complete subset of $\mathcal{C}\left([0,S];\mathcal{C}([0,1];\mathcal{P}^2)\right)$. 
Then, according to the contraction map theorem, there is a unique distribution in the function space so that $\phi^\ast=\mathcal{T}_S(\phi^\ast)$, meaning that $\phi^\ast$ solves~\eqref{eqn:varphi} itself, and is thus the unique solution to~\eqref{eqn:Wassgradientflows} for $s<S$.
\item[Step 3:] We extend the local solution to a global solution.
\end{itemize}

\subsection{Step 1}\label{sec:Step1}
For fixed $\phi(\theta,t,s)$, we define the gradient flow:
\begin{equation}\label{eqn:ODEdeltaEs}
\left\{
\begin{aligned}
\frac{\rd \theta_\phi(s;t)}{\rd s} & = -\nabla_\theta\frac{\delta E_s(\phi(s))}{\delta \rho}(\theta,t)\left(\theta_\phi(s;t),t\right)\,,\\
 \theta_\phi(0;t) & \sim \rho_{\ini}(\theta,t)\,.
\end{aligned}
\right.
\end{equation}
Then $\varphi=\mathcal{T}_S(\phi)$ is the probability measure of $\theta_\phi$ for any $s\in[0,S]$. Well-posedness of $\varphi$ translates to the well-posedness of $\theta_\phi$.

According to \eqref{Lipshictz} in Lemma~\ref{lemmadeltaEs}, the force $\nabla_\theta\frac{\delta E_s(\phi(s))}{\delta \rho}(\theta,t)(\cdot,t)$ is a Lipschitz function with any $t\in[0,1]$. The classical ODE theory then suggests there is a unique solution for $s \in [0,S]$, which depends continuously on the initial $\theta(0;t)$.

Denoting
\begin{equation}\label{eqn:L}
\mathcal{L}_{S,\phi}=\sup_{0\leq s\leq S}\int^1_0\int_{\mathbb{R}^k}|\theta|^2\rd\phi(\theta,t,s) \rd t,\quad \mathcal{L}^{\sup}_{\ini}=\sup_{0\leq t\leq 1}\int_{\mathbb{R}^k}|\theta|^2\rd\rho_{\ini}(\theta,t)\,,
\end{equation}
we have the following proposition.
\begin{proposition}\label{proposition:deltaEODE}
Suppose that $\theta_\phi(0;t)$ is a random measurable path drawn from $\rho_{\ini}(\theta,t)$ and that $\theta_\phi(s;t)$ is the solution to \eqref{eqn:ODEdeltaEs}.
Then for any $(t_1,s_1),(t_2,s_2)\in[0,1]\times[0,S]$, we have
\begin{equation}\label{boundofODEsolutiontheta}
\mathbb{E}\left(\left|\theta_\phi(s_1;t_1)\right|^2\right)\leq \exp(5S)\left(\mathcal{L}^{\sup}_{\ini}+SC(\mathcal{L}_{S,\phi})\right)\,,
\end{equation}
and
\begin{equation}\label{stabilityofODEsolutiontheta}
\begin{aligned}
&\mathbb{E}\left(\left|\theta_\phi(s_1;t_1)-\theta_\phi(s_2;t_2)\right|^2\right)\\
& \leq C\left(\mathcal{L}_{S,\phi},\mathcal{L}^{\sup}_{\ini},S\right)\left(\mathbb{E}\left(\left|\theta_\phi(0;t_1)-\theta_\phi(0;t_2)\right|^2\right)+|t_1-t_2|+|s_1-s_2|^2\right)\,,
\end{aligned}
\end{equation}
\end{proposition}
\begin{proof}
To prove \eqref{boundofODEsolutiontheta}, we use \eqref{bound} and \eqref{eqn:ODEdeltaEs} to obtain that
\[
\frac{\rd \left|\theta_\phi(s;t_1)\right|^2}{\rd s}\leq 4\left|\theta_\phi(s;t_1)\right|^2+2\left|\theta_\phi(s;t_1)\right|C(\mathcal{L}_{S,\phi})\leq 5\left|\theta_\phi(s;t_1)\right|^2+C(\mathcal{L}_{S,\phi})\,.
\]
Using Gr\"onwall's inquality, we obtain, for all $s_1<S$
\[
\left|\theta_\phi(s_1;t_1)\right|^2\leq \exp(5S)\left|\theta_\phi(0;t_1)\right|^2+\exp(5S)SC(\mathcal{L}_{S,\phi})\,.
\] 
By taking the expectation with respect to $\rho_{\ini}(\theta,t)$, we have
\[
\mathbb{E}\left(\left|\theta_\phi(s_1;t_1)\right|^2\right)\leq \exp(5S)\mathbb{E}\left(\left|\theta_\phi(0;t_1)\right|^2\right)+\exp(5S)SC(\mathcal{L}_{S,\phi})\leq \exp(5S)(\mathcal{L}^{\sup}_{\ini}+SC(\mathcal{L}_{S,\phi}))\,,
\]
completing the proof of~\eqref{boundofODEsolutiontheta}.

For the left-hand side of \eqref{stabilityofODEsolutiontheta}, we have
\begin{equation}\label{error}
\begin{aligned}
&\mathbb{E}\left(\left|\theta_\phi(s_1;t_1)-\theta_\phi(s_2;t_2)\right|^2\right)\\
& \leq 2\underbrace{\mathbb{E}\left(\left|\theta_\phi(s_1;t_1)-\theta_\phi(s_1;t_2)\right|^2\right)}_{\textrm{(I)}}+2\underbrace{\mathbb{E}\left(\left|\theta_\phi(s_1;t_2)-\theta_\phi(s_2;t_2)\right|^2\right)}_{\textrm{(II)}}\,,
\end{aligned}
\end{equation}
and bound the two terms $\textrm{(I)}$ and $\textrm{(II)}$ separately.
\begin{enumerate}[wide,   labelindent=0pt]
\item[(I):] From \eqref{eqn:ODEdeltaEs} and \eqref{Lipshictz} in Lemma~\ref{lemmadeltaEs}, we have
\[
\frac{\rd\left|\theta_\phi(s;t_1)-\theta_\phi(s;t_2)\right|^2}{\rd s}\leq C(\mathcal{L}_{S,\phi})\left|\theta_\phi(s;t_1)-\theta_\phi(s;t_2)\right|^2+\left|\theta_\phi(s;t_1)\right|^2|t_1-t_2|\,,
\]
which implies
\begin{equation}\label{errorI}
\begin{aligned}
&\mathbb{E}\left(\left|\theta_\phi(s_1;t_1)-\theta_\phi(s_1;t_2)\right|^2\right)\\
& \leq C\left(\mathcal{L}_{S,\phi},\mathcal{L}^{\sup}_{\ini},S\right)\left(\mathbb{E}\left(\left|\theta_\phi(0;t_1)-\theta_\phi(0;t_2)\right|^2\right)+\left(\int^{s_1}_0\mathbb{E}\left|\theta_\phi(s;t_1)\right|^2ds\right)|t_1-t_2|\right)\\
& \leq C\left(\mathcal{L}_{S,\phi},\mathcal{L}^{\sup}_{\ini},S\right)\left(\mathbb{E}\left(\left|\theta_\phi(0;t_1)-\theta_\phi(0;t_2)\right|^2\right)+|t_1-t_2|\right)\,,
\end{aligned}
\end{equation}
where we use the bound~\eqref{boundofODEsolutiontheta} in the second inequality.

\item[(II):] This is estimated by integrating~\eqref{eqn:ODEdeltaEs} from $s_1$ to $s_2$ and using the bound of $\nabla_\theta\frac{\delta E_s(\rho)}{\delta \rho}$ from~\eqref{bound}. From the Gr\"onwall inequality and \eqref{boundofODEsolutiontheta}, we have
\[
\left|\theta_\phi(s_1;t_2)-\theta_\phi(s_2;t_2)\right|\leq C\left(\mathcal{L}_{S,\phi},\mathcal{L}^{\sup}_{\ini},S\right)(|s_1-s_2|+|e^{s_1}-e^{-s_2}|)\leq C\left(\mathcal{L}_{S,\phi},\mathcal{L}^{\sup}_{\ini},S\right)|s_1-s_2|\,.
\]
The second inequality comes from the fact that $s_i<S$. This leads to
\begin{equation}\label{errorII}
\mathbb{E}\left|\theta_\phi(s_1;t_2)-\theta_\phi(s_2;t_2)\right|^2\leq C\left(\mathcal{L}_{S,\phi},\mathcal{L}^{\sup}_{\ini},S\right)|s_1-s_2|^2\,.
\end{equation}
\end{enumerate}
By substituting \eqref{errorI} and~\eqref{errorII} into \eqref{error}, we complete the proof of~\eqref{stabilityofODEsolutiontheta}.
\end{proof}

An immediate corollary of Proposition~\ref{proposition:deltaEODE} is that the map $\mathcal{T}_S(\cdot)$ is well-defined.


\begin{cor}\label{cor:Tswell-defined}
For every $S>0$, the map $\mathcal{T}_S$ is well-defined. That is, for any $\phi\in \mathcal{C}\left([0,S];\mathcal{C}([0,1];\mathcal{P}^2)\right)$, one can find $\varphi=\mathcal{T}_S(\phi)\in \mathcal{C}([0,S];\mathcal{C}([0,1];\mathcal{P}^2))$ as the unique solution of~\eqref{eqn:varphi}. In particular, for any $(t,s)\in[0,1]\times[0,S]$, we have
\begin{equation}\label{eqn:secondmomentbound}
\int_{\mathbb{R}^k}|\theta|^2d\varphi(\theta,t,s)\leq \exp(5S)(\mathcal{L}^{\sup}_{\ini}+SC(\mathcal{L}_{S,\phi}))\,,
\end{equation}
where $C(\mathcal{L}_{S,\phi})$ is a quantity depending only on  $\mathcal{L}_{S,\phi}$.
\end{cor}

\begin{proof} For fixed $(t,s)\in[0,1]\times[0,S]$, define $\varphi(\theta,t,s)$ as the distribution of $\theta_\phi(s;t)$. Using the classical stochastic theory~\citep[Prop 8.1.8]{Gradientflow}, $\varphi(\theta,t,s)$ is a solution to \eqref{eqn:varphi}. For fixed $(t,s)\in[0,1]\times[0,S]$, we prove \eqref{eqn:secondmomentbound} due to~\eqref{boundofODEsolutiontheta}. Finally, using \eqref{eqn:existenceofpath} and \eqref{stabilityofODEsolutiontheta}, we obtain that 
\[
\lim_{(t,s)\rightarrow (t_0,s_0)}W_2(\varphi(\cdot,t,s),\varphi(\cdot,t_0,s_0))\leq \lim_{(t,s)}\left(\mathbb{E}\left(\left|\theta_\phi(s;t)-\theta_\phi(s_0;t_0)\right|^2\right)\right)^{1/2}=0\,,
\]
which proves that $\varphi\in \mathcal{C}\left([0,S];\mathcal{C}([0,1];\mathcal{P}^2)\right)$.
\end{proof}
\subsection{Step 2.} We show first the contraction property.
\begin{proposition}\label{prop:contraction}
For any $\phi_1,\phi_2\in \mathcal{C}([0,S];\mathcal{C}([0,1];\mathcal{P}^2))$, we have
\begin{equation}\label{eqn:d2bound}
d_2(\mathcal{T}_S(\phi_1),\mathcal{T}_S(\phi_2))\leq Q(\mathcal{L}_{S})(S\exp(5S)(\mathcal{L}^{\sup}_{\ini}+SQ(\mathcal{L}_S))+S)^{1/2}d_2(\phi_1,\phi_2),
\end{equation}
where $Q:\mathbb{R}_+\rightarrow\mathbb{R}_+$ is an increasing function and $\mathcal{L}_S=\max\{\mathcal{L}_{S,\phi_1}\,,\mathcal{L}_{S,\phi_2}\}$ with $\mathcal{L}_{S,\phi}$ defined in~\eqref{eqn:L}.
\end{proposition}
\begin{proof}
Denote by $\theta_{\phi_i}(s;t)$ the solutions to~\eqref{eqn:ODEdeltaEs} with  $\phi=\phi_i$, and suppose that the initial conditions match, that is,
\[
\theta_{\phi_1}(0;t)=\theta_{\phi_2}(0;t)\,.
\]
Similar to the previous subsection, we translate the study of $\varphi_i$ to the study of $\theta_{\phi_i}$. Define 
\[
\Delta_t(s)=|\theta_{\phi_1}(s;t)-\theta_{\phi_2}(s;t)|\,,
\]
we have 
\[
d_2(\mathcal{T}_S(\phi_1),\mathcal{T}_S(\phi_2))\leq\sup_{(t,s)\in[0,1]\times[0,S]} \mathbb{E}\left(\Delta^2_t(s)\right)\,.
\]
Using \eqref{eqn:ODEdeltaEs}, we have
\begin{equation}\label{eqn:Deltaiteration}
\begin{aligned}
\frac{\rd(\Delta_t(s))^2}{\rd s} & \leq 4(\Delta_t(s))^2+4\left|\nabla_\theta\frac{\delta E_s(\phi_1(s))}{\delta \rho}(\theta_{\phi_1},t)-\nabla_\theta\frac{\delta E_s(\phi_2(s))}{\delta \rho}(\theta_{\phi_2},t)\right|^2\\
& \leq 4(\Delta_t(s))^2+8\left|\nabla_\theta\frac{\delta E_s(\phi_1(s))}{\delta \rho}(\theta_{\phi_1},t)-\nabla_\theta\frac{\delta E_s(\phi_1(s))}{\delta \rho}(\theta_{\phi_2},t)\right|^2\\
& \quad +8\left|\nabla_\theta\frac{\delta E_s(\phi_1(s))}{\delta \rho}(\theta_{\phi_2},t)-\nabla_\theta\frac{\delta E_s(\phi_2(s))}{\delta \rho}(\theta_{\phi_2},t)\right|^2\,.
\end{aligned}
\end{equation}
The second term on the right-hand side concerns the continuity addressed in Lemma~\ref{lemmadeltaEs} \eqref{Lipshictz}. To bound the last term, we use \eqref{eqn:Deltabound2} to obtain
\[
\begin{aligned}
&\left|\nabla_\theta\frac{\delta E_s(\phi_1(s))}{\delta \rho}(\theta_{\phi_2},t)-\nabla_\theta\frac{\delta E_s(\phi_2(s))}{\delta \rho}(\theta_{\phi_2},t)\right|\\
& \leq C(\mathcal{L}_S)d_1(\phi_1(s),\phi_2(s))\left(|\theta_{\phi_2}|+1\right)\leq C(\mathcal{L}_S)d_2(\phi_1,\phi_2)\left(|\theta_{\phi_2}|+1\right)\,,
\end{aligned}
\]
where we use \eqref{eqn:derivativebound}, \eqref{boundofxsolution}, \eqref{boundofprho} in the first inequality and \eqref{stabilityofxsolution}, \eqref{stabilityofprho} in the second inequality, with $d_2$ defined in Definition~\ref{def:path2}.
By substituting into~\eqref{eqn:Deltaiteration}, we obtain
\[
\frac{\rd(\Delta_t(s))^2}{\rd s}\leq C(\mathcal{L}_S)\left[(\Delta_t(s))^2+d^2_2(\phi_1,\phi_2)\left(|\theta_{\phi_2}|^2+1\right)\right]\,,
\]
Using the Gr\"onwall inequality, we then have
\begin{align*}
\mathbb{E}\left((\Delta_t(s))^2\right)& \leq C(\mathcal{L}_S)\left(\int^s_0\mathbb{E}|\theta_{\phi_2}(u;t)|^2du+s\right)d^2_2(\phi_1,\phi_2)\\
& \leq C(\mathcal{L}_S)d^2_2(\phi_1,\phi_2)(S\exp(5S)(\mathcal{L}^{\sup}_{\ini}+SC(\mathcal{L}_S))+s)\,,
\end{align*}
completing the proof.
\end{proof}

We are now ready to run the contraction-map argument that justifies the existence and uniqueness of the local solution. According to the contraction mapping theorem, we need to verify two conditions in order to show that there is a $\phi^\ast=\mathcal{T}_S(\phi^\ast)$:
\begin{itemize}
    \item There is a closed subset in $C([0,S];C([0,1],W_2))$ so that $\mathcal{T}_S$ maps it to itself; and
    \item $\mathcal{T}_S$ is a contraction map in this subset.
\end{itemize}

We define the closed subset next.
\begin{definition}\label{def:srho0}
\[
\mathcal{B}_{\rho_0}=\left\{\phi\in \mathcal{C}([0,S];\mathcal{C}([0,1];\mathcal{P}^2))\middle| \sup_{t\in[0,1],0\leq s\leq S}\int_{\mathbb{R}^k}|\theta|^2\rd\phi(\theta,t,s)\leq 4\mathcal{L}^{\sup}_{\ini}\right\}\,.
\]
\end{definition}

We now claim that for small enough $S$, $\mathcal{T}_S$ is a contraction map in $\mathcal{B}_{\rho_0}$.
\begin{proposition}\label{cor:c2} Suppose that $S$ is small enough  that 
\begin{equation}\label{eqn:conditionofS}
\begin{aligned}
\exp(5S)(\mathcal{L}^{\sup}_{\ini}+SQ(4\mathcal{L}^{\sup}_{\ini}))& <4\mathcal{L}^{\sup}_{\ini}\,, \\ Q(4\mathcal{L}^{\sup}_{\ini})(S\exp(5S)(\mathcal{L}^{\sup}_{\ini}+SQ(4\mathcal{L}^{\sup}_{\ini}))+S)^{1/2}& <\tfrac{1}{2}\,,
\end{aligned}
\end{equation}
where $Q$ comes from Proposition~\ref{prop:contraction}. Then the following are true.
\begin{itemize}
\item If $\phi\in\mathcal{B}_{\rho_0}$, then $\mathcal{T}_S(\phi)\in\mathcal{B}_{\rho_0}$, that is,
\begin{equation}\label{eqn:Deltabound3}
\sup_{t\in[0,1],0\leq s\leq S}\int_{\mathbb{R}^k}|\theta|^2\rd\mathcal{T}(\phi)(\theta,t,s)\leq 4\mathcal{L}^{\sup}_{\ini}
\end{equation}
\item $\mathcal{T}_S$ is a contraction map in this subset, meaning that for any $\phi_1,\phi_2\in \mathcal{B}_{\rho_0}$, we have
\begin{equation}\label{eqn:d2contraction}
d_2(\mathcal{T}_S(\phi_1),\mathcal{T}_S(\phi_2))\leq \tfrac{1}{2}d_2(\phi_1,\phi_2)\,.
\end{equation}
\end{itemize}
\end{proposition}
\begin{proof} First, using Corollary~\ref{cor:Tswell-defined} \eqref{eqn:secondmomentbound}, we have
\[
\int^1_0\int_{\mathbb{R}^k}|\theta|^2\rd\mathcal{T}(\phi_1)(\theta,t,s)dt\leq \exp(5S)(\mathcal{L}^{\sup}_{\ini}+SQ(4\mathcal{L}^{\sup}_{\ini}))\leq 4\mathcal{L}^{\sup}_{\ini}\,,
\]
which proves \eqref{eqn:Deltabound3}. Then, using \eqref{eqn:d2bound}, we have
\[
d_2(T_S(\phi_1),T_S(\phi_2))\leq Q(4\mathcal{L}^{\sup}_{\ini})(S\exp(5S)(\mathcal{L}^{\sup}_{\ini}+SQ(4\mathcal{L}^{\sup}_{\ini}))+S)^{1/2}d_2(\phi_1,\phi_2)<\tfrac{1}{2}d_2(\phi_1,\phi_2),
\]
which proves \eqref{eqn:d2contraction}.
\end{proof}

Using the contraction mapping theorem, we can obtain directly that $\mathcal{T}_S(\phi)$ has a fixed point in $\mathcal{B}_{\rho_0}$ when $S$ is small enough.
\begin{cor}\label{cor:localsolution} If $S$ satisfies \eqref{eqn:conditionofS}, then there exists $\phi^\ast(\theta,t,s)\in \mathcal{B}_{\rho_0}\subset \mathcal{C}([0,S];\mathcal{C}([0,1];\mathcal{P}^2))$ such that $\phi^\ast(\theta,t,s)$ is a solution to \eqref{eqn:Wassgradientflows} with initial condition $\rho_{\ini}(\theta,t)$ and~\revise{admissible for each s}.
\end{cor}
\begin{proof}~\revise{According to Proposition~\ref{cor:c2}, $\mathcal{T}_S$ is a contraction map in  $\mathcal{B}_{\rho_0}$. The result then follows immediately from the contraction mapping theorem and the admissibility of $\mathcal{T}_S(\phi)$ by its definition.}
\end{proof}

Finally, we prove that the cost function decays along the flow.
\begin{lemma}\label{lem:decayofcost}
Fix $S>0$. If $\phi^\ast(\theta,t,s)\in \mathcal{C}([0,S];\mathcal{C}\left([0,1];\mathcal{P}^2)\right)$ such that $\phi^\ast(\theta,t,s)$ is a solution to \eqref{eqn:Wassgradientflows} with initial condition $\rho_{\ini}(\theta,t)$, then for $0<s<S$, we have
\[
\frac{\rd E_{s}(\phi^\ast(\cdot,\cdot,s))}{\rd s}\leq 0\,.
\]
\end{lemma}
\begin{proof}
Denote by $\theta^\ast(s;t)$ the solution to \eqref{eqn:ODEdeltaEs} with $\phi=\phi^\ast$. Then since $\phi^\ast$ is the unique solution up to $s\leq S$, the probability distribution of $\theta^\ast$ is $\phi^\ast$ itself.

According to \eqref{eqn:Wassgradientflows}, we obtain formally that
\begin{equation}\label{eqn:derivative2}
\begin{aligned}
& \frac{\rd E_{s}(\phi^\ast(\cdot,\cdot,s))}{\rd s} \\
&=-\int^1_0\int_{\mathbb{R}^k} \left|\nabla_\theta\frac{\delta E_s(\phi^\ast(s))}{\delta \rho}(\theta,t)\right|^2\rd\phi^\ast(\theta,t,s)\rd t-\exp^{-s}\int^1_0\int_{\mathbb{R}^k}|\theta|^2\rd\phi^\ast(\theta,t,s)\rd t\\
& =-\int^1_0\mathbb{E}\left(\left|\nabla_\theta\frac{\delta E_s(\phi^\ast(s))}{\delta \rho}\left(\theta^\ast(s;t),t\right)\right|^2\right)\rd t-\exp^{-s}\int^1_0\mathbb{E}(|\theta^\ast(s;t)|^2)\rd t\\
& \leq 0\,,
\end{aligned}
\end{equation}
which proves the result. We note that the rigorous proof of \eqref{eqn:derivative2} is in Appendix \ref{proofoflemdecayofcost}.
\end{proof}

\subsection{Step 3:}\label{sec:Step3} 

In this final step of the proof, we extend the local solution of Corollary~\ref{cor:localsolution} to a global solution. Since Lemma~\ref{lem:decayofcost} shows that $E_s$ is decaying with $s$, we can then improve the bound for the second moment of the solution \eqref{eqn:secondmomentbound}. We have the following corollary.
\begin{cor}\label{cor:C.3}
For fixed $S>0$, denote by $\phi^\ast(\theta,t,s)\in \mathcal{C}\left([0,S];\mathcal{C}([0,1];\mathcal{P}^2)\right)$ the solution to \eqref{eqn:Wassgradientflows} with initial condition $\rho_{\ini}(\theta,t)$.
Then for any $(t,s)\in[0,1]\times[0,S]$, we have
\begin{equation}\label{eqn:secondmomentbound2}
\int_{\mathbb{R}^k}|\theta|^2\rd\phi^\ast(\theta,t,s)\leq \exp(5s)(\mathcal{L}^{\sup}_{\ini}+C(s))\,,
\end{equation}
where the quantity $C$  depends only on s.
\end{cor}
\begin{proof}  
Using $E_s(\phi^\ast(\theta,t,s))<E_0(\phi^\ast(\theta,t,0))$, we first have
\[
\sup_{0\leq u\leq s}\int^1_0\int_{\mathbb{R}^k}|\theta|^2\rd\phi^\ast(\theta,t,s) \rd t\leq \exp(s)E_0(\rho_{\ini}(\theta,t))\,.
\]
Then, using \eqref{eqn:secondmomentbound}, we obtain that
\[
\int_{\mathbb{R}^k}|\theta|^2\rd\mathcal{T}(\phi^\ast)(\theta,t,s)\leq \exp(5s)(\mathcal{L}^{\sup}_{\ini}+sC(\exp(s)E_0(\rho_{\ini}(\theta,t))))\,,
\]
proving \eqref{eqn:secondmomentbound2}.
\end{proof}

This corollary takes an important step. By contrast with~\eqref{eqn:secondmomentbound}, it removes the dependence of the bound on $\mathcal{L}_{S,\phi}$, so the fixed-point argument need no longer depend on the initial guess $\phi$.

We are now are ready to prove the Theorem~\ref{thm:Wassgradientflows}.

\begin{proof}[Proof of Theorem~\ref{thm:Wassgradientflows}] 

From Corollary~\ref{cor:localsolution}, if $S_1$ satisfies \eqref{eqn:conditionofS}, we obtain a local solution $\phi^\ast\in C([0,S_1];\mathcal{C}([0,1];\mathcal{P}^2))$ to \eqref{eqn:Wassgradientflows},~\revise{which is admissible for any $s$.}

Denote by $S^\ast$ the supremum of  such constants for which the solution exists. We hope to show that $S^\ast=\infty$. If not, then denote the solution by $\phi^\ast\in C\left([0,S^\ast);\mathcal{C}([0,1];\mathcal{P}^2)\right)$. 
According to~\eqref{eqn:secondmomentbound2}, for any $s<S^\ast$ and $t\in[0,1]$, we have
\[
\int_{\mathbb{R}^k}|\theta|^2\rd\phi^\ast(\theta,t,s)\leq \exp(5S^\ast)(\mathcal{L}^{\sup}_{\ini}+Q_2(S^\ast))=\mathcal{L}^{\ast,\sup}\,.
\]
Let us choose $\Delta_{S^\ast}$ small enough to satisfy
\begin{align*}
\exp(5\Delta_{S^\ast})(\mathcal{L}^{\ast,\sup}+\Delta_{S^\ast}Q(4\mathcal{L}^{\ast,\sup}))& <4\mathcal{L}^{\ast,\sup}, \\
Q(4\mathcal{L}^{\ast,\sup})(\Delta_{S^\ast}\exp(5\Delta_{S^\ast})(\mathcal{L}^{\ast,\sup}+\Delta_{S^\ast}Q(4\mathcal{L}^{\ast,\sup}))+\Delta_{S^\ast})^{1/2} & <\tfrac{1}{2}\,.
\end{align*}
Then, using Proposition~\ref{cor:c2} and Corollary~\ref{cor:localsolution}, we can further extend $\phi^\ast$ to be supported on $C\left([0,S^\ast+\Delta_{S^\ast});\mathcal{C}([0,1];\mathcal{P}^2)\right)$. This contradicts the assumption that we have to stop at  a finite value of $S^\ast$. Thus $S^*=\infty$.

Finally, \eqref{eqn:decayEsmean-field} is a direct result of Lemma~\ref{lem:decayofcost}.
\end{proof}

\subsection{Proof of Theorem~\ref{thm:finiteLwell-posed1}}
In this section, we consider the finite layer case \eqref{eqn:disRes} and prove Theorem~\ref{thm:finiteLwell-posed1}. 
In \eqref{eqn:costfunction}, $\nabla_{{\Theta}_{L,M}} E({\Theta}_{L,M})$ can also be written as
\begin{equation}\label{eqn:partialEfiniteL}
\frac{\partial E({\Theta}_{L,M})}{\partial \theta_{l,m}}=\frac{1}{ML}\mathbb{E}_{x\sim\mu}\left(\partial_\theta f(Z_{\Theta_{L,M}}(l;x),\theta_{l,m})p_{\Theta_{L,M}}(l;x)\right)\,,
\end{equation}
where $p_{\Theta_{L,M}}(l;x)$ can be solved by the following iteration formula:
\begin{equation}\label{eqn:prhofiniteL}
\left\{
\begin{aligned}
p^\top_{\Theta_{L,M}}(l;x)&=p^\top_{\Theta_{L,M}}(l+1;x)\left(I+\frac{1}{ML}\sum^M_{m=1}\partial_zf\left(Z_{\Theta_{L,M}}(l+1;x),\theta_{l+1,i}\right)\right)\\
p_{\Theta_{L,M}}(L-1;x)&=\left(g(Z_{\Theta_{L,M}}(L;x))-y(x)\right)\nabla g(Z_{\Theta_{L,M}}(L;x))
\end{aligned}
\right.\,,
\end{equation}
for $0\leq l\leq L-2$. Similar to Definitions~\ref{def:path} and \ref{def:path2}, we make the following definition.
\begin{definition}\label{def:pathfiniteL}
${\Theta}_{L,M}=\left\{\theta_{l,m}\right\}^{L-1,M}_{l=0,m=1}\in L^\infty_{L,M}$ if and only if 
\[
\sup_{l,m}|\theta_{l,m}|<\infty\,.
\]
The metric in $L^\infty_{L,M}$ is defined as
\[
d_{1,L,M}\left({\Theta}_{L,M},\widetilde{\Theta}_{L,M}\right)=\max_{l}\left(\frac{1}{M}\sum^M_{m=1}|\theta_{l,m}-\widetilde{\theta}_{l,m}|^2\right)^{1/2}\,.
\]
\end{definition}
\begin{definition}
For $s \ge 0$, we have  $\Theta_{L,M}(s)=\left\{\theta_{l,m}(s)\right\}^{L-1,M}_{l=0,m=1}\in \mathcal{C}([0,\infty);L^\infty_{L,M})$ if and only if
\begin{itemize}
\item[1.] For fixed $s\in[0,\infty)$, $\Theta_{L,M}(s)\in L^\infty_{L,M}$, and
\item[2.] For any $s_0\in[0,\infty)$, we have
\[
\lim_{s\rightarrow s_0}d_{1,L,M}\left(\Theta_{L,M}(s),\Theta_{L,M}(s_0)\right)=0\,,
\]
where $d_{1,L,M}$ is defined in Definition~\ref{def:pathfiniteL}.
\end{itemize}
The metric in $\mathcal{C}([0,\infty);L^\infty_{L,M})$ is defined by
\[
d_{2,L,M}\left({\Theta}_{L,M},\widetilde{\Theta}_{L,M}\right)=\sup_{s}d_{1,L,M}({\Theta}_{L,M}(s),\widetilde{\Theta}_{L,M}(s))\,.
\]
\end{definition}

Theorem~\ref{thm:finiteLwell-posed1} is a corollary of the following result.
\begin{theorem}\label{thm:finiteLwell-posed} Suppose that Assumption \ref{assum:f} holds. If $\Theta_{L,M}(0)\in L^\infty_{L,M}$ then \eqref{eqn:classicalgradientflowfiniteL} has a unique solution in $\mathcal{C}([0,\infty);L^\infty_{L,M})$.
\end{theorem}

Before proving the Theorem~\ref{thm:finiteLwell-posed}, we first prove boundedness and stability of $Z_{\Theta_{L,M}}$ and $p_{\Theta_{L,M}}$, similarly to Theorem~\ref{thm:wmean-fieldlimit} and Lemma~\ref{lem:prhoprop1}.
\begin{lemma}\label{prop:wmean-fieldlimitfiniteL}
Suppose that Assumption~\ref{assum:f} holds  and  that $x$ in the support of $\mu$. Let 
\[
\widetilde{\Theta}_{L,M}=\left\{\widetilde{\theta}_{l,m}\right\}^{L-1,M}_{l=0,m=1}
\] 
be a set of parameters, in addition to $\Theta_{L,M}$, and denote
\[
\mathcal{L}_{\Theta_{L,M}}=\frac{1}{LM}\sum^{L-1}_{l=0}\sum^M_{m=1}|\theta_{l,m}|^2,\quad \mathcal{L}_{\widetilde{\Theta}_{L,M}}=\frac{1}{LM}\sum^{L-1}_{l=0}\sum^M_{m=1}|\widetilde{\theta}
_{l,m}|^2\,.
\]
Then there exist constants $C(\mathcal{L}_{\Theta_{L,M}})$ depending on $\mathcal{L}_{\Theta_{L,M}}$ and $C(\mathcal{L}_{\Theta_{L,M}},\mathcal{L}_{\widetilde{\Theta}_{L,M}})$ depending on $\mathcal{L}_{\Theta_{L,M}},\mathcal{L}_{\widetilde{\Theta}_{L,M}}$ such that for all $0\leq l\leq L-1$, we have the following bounds:
\begin{subequations}
\begin{align} \label{boundofxsolutionfiniteL}
\left|Z_{\Theta_{L,M}}(l+1;x)\right|&\leq C(\mathcal{L}_{\Theta_{L,M}})\,, \\
\label{stabilityofxsolutionfiniteL}
\left|Z_{\Theta_{L,M}}(l+1;x)-Z_{\widetilde{\Theta}_{L,M}}(l+1;x)\right|& \leq C(\mathcal{L}_{\Theta_{L,M}},\mathcal{L}_{\widetilde{\Theta}_{L,M}})d_{1,L,M}\left({\Theta}_{L,M},\widetilde{\Theta}_{L,M}\right)\,,
\end{align}
\end{subequations}
\begin{subequations}
\begin{align}
\label{boundofprhofiniteL}
\left|p_{\Theta_{L,M}}(l;x)\right|& \leq C(\mathcal{L}_{\Theta_{L,M}})\,, \\
\label{stabilityofpsolutionfiniteL}
\left|p_{\Theta_{L,M}}(l;x)-p_{\widetilde{\Theta}_{L,M}}(l;x)\right| & \leq C(\mathcal{L}_{\Theta_{L,M}},\mathcal{L}_{\widetilde{\Theta}_{L,M}})d_{1,L,M}\left({\Theta}_{L,M},\widetilde{\Theta}_{L,M}\right)\,.
\end{align}
\end{subequations}
\end{lemma}
\begin{proof} Using \eqref{eqn:disRes} and \eqref{eqn:boundoff}, we obtain that
\[
\begin{aligned}
\left(\left|Z_{\Theta_{L,M}}(l+1;x)\right|+1\right)&\leq C_1\left(1+\frac{1}{LM}\sum^M_{m=1}|\theta_{l,m}|\right)(\left|Z_{\Theta_{L,M}}(l;x)\right|+1)\\
&\leq C_1\exp\left(\frac{1}{LM}\sum^M_{m=1}|\theta_{l,m}|\right)(\left|Z_{\Theta_{L,M}}(l;x)\right|+1)\,,
\end{aligned}
\]
which proves \eqref{boundofxsolutionfiniteL} by iteration.

Similarly, using \eqref{eqn:derivativebound}, we obtain that
\[
\frac{1}{ML}\sum^M_{m=1}|\partial_zf\left(Z_{\Theta_{L,M}}(l+1;x),\theta_{l,m}\right)|\leq \frac{C_1}{ML}\sum^M_{m=1}|\theta_{l,m}|\,,
\]
which by \eqref{eqn:prhofiniteL} implies that
\[
|p_{\Theta_{L,M}}(l;x)|\leq \left(1+\frac{C_1}{LM}\sum^M_{m=1}|\theta_{l,m}|\right)|p_{\Theta_{L,M}}(l+1;x)|\,.
\]
From this bound, together with $|p_{\Theta_{L,M}}(x,L-1)|\leq C|Z_{\Theta_{L,M}}(L;x)|\leq C(\mathcal{L}_{\Theta_{L,M}})$, we prove \eqref{boundofprhofiniteL} by iteration.

To prove \eqref{stabilityofxsolutionfiniteL}, similarly to \eqref{boundofDeltatZrho}, we use \eqref{eqn:disRes}, \eqref{eqn:derivativebound}, and \eqref{boundofxsolutionfiniteL} to obtain 
\[
\begin{aligned}
&\left|Z_{\Theta_{L,M}}(l+1;x)-Z_{\widetilde{\Theta}_{L,M}}(l+1;x)\right|\\
\leq &\left(1+\frac{C_1}{ML}\sum^M_{m=1}|\theta_{l,m}|\right)\left|Z_{\Theta_{L,M}}(l;x)-Z_{\widetilde{\Theta}_{L,M}}(l;x)\right|+\frac{C(\mathcal{L}_{\Theta_{L,M}},\mathcal{L}_{\widetilde{\Theta}_{L,M}})}{ML}\sum^M_{m=1}|\theta_{l,m}-\widetilde{\theta}_{l,m}|\,.
\end{aligned}
\]
By using this bound iteratively in conjunction with the boundary condition $|Z_{\Theta_{L,M}}(0;x)-Z_{\widetilde{\Theta}_{L,M}}(0;x)|=0$, we prove \eqref{stabilityofxsolutionfiniteL}.

To prove \eqref{stabilityofpsolutionfiniteL}, similarly to the proof of Lemma~\ref{lem:prhoprop1} \eqref{eqn:ODEforDEltaprho}-\eqref{LemmaBexample}, we use \eqref{eqn:derivativebound} and  \eqref{eqn:prhofiniteL}-\eqref{boundofprhofiniteL} to obtain
\begin{equation}\label{eqn:iterationappendix1}
\begin{aligned}
&\left|p_{\Theta_{L,M}}(l;x)-p_{\widetilde{\Theta}_{L,M}}(l;x)\right|\\
& \leq \left|(p_{\Theta_{L,M}}(l+1;x)-p_{\widetilde{\Theta}_{L,M}}(l+1;x))^\top\left(I+\frac{1}{ML}\sum^M_{m=1}\partial_zf\left(Z_{\Theta_{L,M}}(l+1;x),\theta_{l+1,m}\right)\right)\right|\\
& \quad +\left|p^\top_{\widetilde{\Theta}_{L,M}}(l+1;x)\left(\frac{1}{ML}\sum^M_{m=1}\left(\partial_zf\left(Z_{\Theta_{L,M}}(l+1;x),\theta_{l+1,i}\right)-\partial_zf\left(Z_{\widetilde{\Theta}_{L,M}}(l+1;x),\widetilde{\theta}_{l+1,m}\right)\right)\right)\right|\\
& \leq C(\mathcal{L}_{\Theta_{L,M}},\mathcal{L}_{\widetilde{\Theta}_{L,M}})\left(\frac{1}{ML}\sum^M_{m=1}|\theta_{l,m}|+1\right)\left|p_{\Theta_{L,M}}(l+1;x)-p_{\widetilde{\Theta}_{L,M}}(l+1;x)\right|\\
& \quad +C (\mathcal{L}_{\Theta_{L,M}},\mathcal{L}_{\widetilde{\Theta}_{L,M}} )\left(\frac{1}{ML}\sum^M_{m=1}|\theta_{l,m}|^2+|\widetilde{\theta}_{l,m}|^2+\frac{1}{L}\right)d_{1,L,M}\left({\Theta}_{L,M},\widetilde{\Theta}_{L,M}\right)\,.
\end{aligned}
\end{equation}
Similarly to \eqref{eqn:ODEforDEltaprhoinitial}, we have 
\begin{align*}
    |p_{\Theta_{L,M}}(L-1;x)-p_{\widetilde{\Theta}_{L,M}}(L-1;x)|&\leq C|Z_{\Theta_{L,M}}(L;x)-Z_{\widetilde{\Theta}_{L,M}}(L;x)|\\&\leq C\left(\mathcal{L}_{\Theta_{L,M}},\mathcal{L}_{\widetilde{\Theta}_{L,M}}\right)d_{1,L,M}\left({\Theta}_{L,M},\widetilde{\Theta}_{L,M}\right)\,.
\end{align*} 
By combining with \eqref{eqn:iterationappendix1}, we prove \eqref{stabilityofpsolutionfiniteL} by iteration.
\end{proof}

We are now ready to prove Theorem~\ref{thm:finiteLwell-posed}, and thus also its corollary, Theorem~\ref{thm:finiteLwell-posed1}.

\begin{proof}[Proof of Theorem~\ref{thm:finiteLwell-posed}] 
The proof is quite similar to that of Theorem~\ref{thm:Wassgradientflows}, so we omit some details. 

Recall \eqref{eqn:classicalgradientflowfiniteL}:
\[
\frac{\rd{\Theta}_{L,M}(s)}{\rd s}=-ML\nabla_{{\Theta}_{L,M}} E({\Theta}_{L,M}(s))-2e^{-s}{\Theta}_{L,M}(s)=-ML\nabla_{{\Theta}_{L,M}} E_s({\Theta}_{L,M}(s)).
\]
Similarly to Lemma~\ref{lemmadeltaEs}, using \eqref{eqn:partialEfiniteL} and Lemma~\ref{prop:wmean-fieldlimitfiniteL}, we have the following bound and stability inequalities for $ML\nabla_{{\Theta}_{L,M}} E_s({\Theta}_{L,M})$:
\begin{subequations}
\begin{align}\label{boundfiniteL}
\left|ML\nabla_{{\Theta}_{L,M}} E_s({\Theta}_{L,M})\right| & \leq 2|\Theta_{L,M}|+C\left(\mathcal{L}\right) \\
\label{LipschitzfiniteL}
\left|ML\nabla_{{\Theta}_{L,M}} E_s({\Theta}_{L,M})-ML\nabla_{\widetilde{\Theta}_{L,M}} E_s(\widetilde{\Theta}_{L,M})\right|& \leq C\left(\mathcal{L}\right)|{\Theta}_{L,M}-\widetilde{\Theta}_{L,M}|\,,
\end{align}
\end{subequations}
where
\[
\mathcal{L}=\max\left\{\frac{1}{LM}\sum^{L-1}_{l=0}\sum^M_{m=1}|\theta_{l,m}|^2,\frac{1}{LM}\sum^{L-1}_{l=0}\sum^M_{m=1}|\widetilde{\theta}_{l,m}|^2\right\}\,.
\]
Similarly to Step 1 of the proof in Section~\ref{sec:proofofthmWassgradientflows}, we can first construct a map induced by \eqref{eqn:classicalgradientflowfiniteL}. Using \eqref{boundfiniteL} and \eqref{LipschitzfiniteL}, similar to Step 1 in the earlier proof, we can show that this map is well-defined. 
Further, similar to Step 2 in Section~\ref{sec:proofofthmWassgradientflows}, we can show that this map is locally contracted in $\mathcal{C}([0,\infty);L^\infty_{L,M})$. 
Finally, for a local solution to \eqref{eqn:classicalgradientflowfiniteL}, $\frac{\rd E_s}{\rd s}\leq 0$ is clearly true, which implies that the proof in Step 3 of Section~\ref{sec:proofofthmWassgradientflows} can also be applied here. Thus, we obtain a global solution.
\end{proof}
\section{Long-time equilibrium and the convergence to the global minimizer}\label{sec:proofofconvergencetoglobalminimal}

We prove Theorem~\ref{thm:globalminimal} in this section. The strategy is as follows. 
First, we show that for any $\rho$ for which $E(\rho)>0$, we can find a direction $\nu$ so that the Fr\'echet derivative of $E(\cdot)$ along this direction is negative, meaning that an infinitesimal change to $\rho$ in direction $\nu$ will  decrease $E$. 
This fact suggests that any stable point $\rho$, which has $\frac{\delta E(\rho)}{\delta\rho}=0$, must be a global minimizer  of $E$, with $E(\rho)=0$. 
This part of the discussion presented in Section~\ref{sec:stable_point}, Proposition~\ref{prop:localisglobal}. Second, in Section~\ref{sec:proofofconvergencetoglobalminimal2}, we show that such a minimizer of $E$ can be obtained as the long-time equilibrium of~\eqref{eqn:Wassgradientflows}.

Note that we assume throughout this section that Assumption~\ref{assum:f} holds (it being one of the assumptions in Theorem~\ref{thm:globalminimal}).

\subsection{$\frac{\delta E(\rho)}{\delta\rho}=0$ vs. $E(\rho)=0$}\label{sec:stable_point}
We prove here  that if $\rho$ is a stable point for $E$, then it is actually a global minimizer. We start by proving a lower bound for $p_\rho$.
\begin{lemma}\label{lem:lowerboundofprho} 
Suppose that $\rho\in \mathcal{C}([0,1];\mathcal{P}^2)$ and that $p_\rho$ is a solution to \eqref{eqn:prho}. 
Denoting
\[
\mathcal{L}_\rho=\int^1_0\int_{\mathbb{R}^k}|\theta|^2\rd\rho(\theta,t)\,,
\]
then for any $t\in[0,1]$ we have
\begin{equation}\label{eqn:lowerboundofprho}
\EE_{x\sim\mu}\left(|p_\rho(t;x)|^2\right)\geq Q_2(\mathcal{L}_{\rho})E(\rho)\,,
\end{equation}
where $Q_2:\mathbb{R}_+\rightarrow\mathbb{R}_+$ is a decreasing function. 
\end{lemma} 
\begin{proof}
Recall the initial condition for $p_\rho$ in~\eqref{eqn:prho}:
\[
p_\rho(1;x)=\left(g(Z_{\rho}(1;x))-y(x)\right)\nabla g(Z_{\rho}(1;x))\,.
\]
According to Assumption~\ref{assum:f} (part 3) and the definition \eqref{eqn:costfunctioncont2} we therefore have
\[
\EE_{x\sim\mu}\left(|p_{\rho}(1;x)|^2\right)\geq \delta_1^2E(\rho)\,.
\]
Further, since the equation \eqref{eqn:prho} is linear, that is, 
\begin{equation}\label{eqn:prhoappendix2}
\frac{\partial p^\top_\rho}{\partial t}=-p^\top_\rho\int_{\mathbb{R}^k}\partial_zf(Z_\rho,\theta)\rd\rho(\theta,t)\,,
\end{equation}
we have
\[
\frac{\rd|p_\rho(t;x)|^2}{\rd t}\leq \left(2C_1\int_{\mathbb{R}^k}|\theta|\rd\rho(\theta,t)\right)|p_\rho(t;x)|^2\,,
\]
where we used~\eqref{boundofrateprho} to bound $\left|\int_{\mathbb{R}^k}\partial_zf(Z_\rho(t;x),\theta)\rd\rho(\theta,t)\right|$. 
By solving this equation, we obtain
\[
|p_\rho(t;x)|^2\geq |p_{\rho}(1;x)|^2\exp\left(-2C_1\int^1_t\int_{\mathbb{R}^k}|\theta|\rd\rho(\theta,t)\right)\geq C(\mathcal{L}_{\rho})|p_{\rho}(1;x)|^2\,.
\]
We finalize the proof by taking expectation of both sides. 
(Monotonicity is a consequence of the format of the exponential term.)
\end{proof}

The following proposition shows existence of a descent direction from any $\rho$ for which $E(\rho)>0$.
\begin{proposition}\label{prop:localisglobal}
Suppose that $\rho\in \mathcal{C}([0,1];\mathcal{P}^2)$. 
If $E(\rho)>0$, then for any $t_0\in[0,1]$, there exists a measure $\nu(\theta)$ of $\mathbb{R}^k$ such that $\int_{\mathbb{R}^k}\rd\nu(\theta)=0$
and
\begin{equation}\label{eqn:negative}
\int_{\mathbb{R}^k}\frac{\delta E(\rho)}{\delta \rho}(\theta,t_0)\rd\nu(\theta)<0\,.
\end{equation}
\end{proposition}
\begin{proof} 
Denote
\[
\mathcal{L}_{\rho}=\int^1_0\int_{\mathbb{R}^k}|\theta|^2\rd\rho(\theta,t)\rd t\,.
\]
According to the existence and uniqueness of solution to \eqref{eqn:meancontRes}, for any $t\in[0,1]$, we can construct a map $\mathcal{Z}_t$ such that
\[
\mathcal{Z}_t(x)=Z_{\rho}\left(t;x\right)\,.
\]
Since the trajectory can be computed backwards in time, $\mathcal{Z}^{-1}_t$ is well-defined. Furthermore, we denote $\mu^\ast_t=(\mathcal{Z}_t)_{\sharp}\mu$ the pushforward of $\mu$ under map $\mathcal{Z}_t$ and let:
\[
p^\ast(t;x)=p_\rho\left(t;\mathcal{Z}^{-1}_t(x)\right)\,.
\]
By Assumption \ref{assum:f} and classical ODE theory, $\mathcal{Z}_t$ and $\mathcal{Z}^{-1}_t$ are both continuous maps in $x$, and so are $p_\rho(t;x)$ and $p^*(t;x)$. With the change of variables, for all $t\in[0,1]$
\begin{equation}\label{changeofvariable}
\frac{\delta E(\rho)}{\delta \rho}(\theta,t)=\int_{\mathbb{R}^d}p^\top_\rho(t;x)f(\mathcal{Z}_t(x),\theta)\rd\mu=\int_{\mathbb{R}^d}(p^\ast(t;x))^\top f(x,\theta)\rd\mu^\ast_t\,.
\end{equation}
For a fixed $t_0\in[0,1]$, calling Lemma \ref{lem:prhoprop1}, we have the boundedness of the Jacobian, meaning $\sup_{x\in\textrm{supp}(\mu)}\left\|\frac{\rd\mu^\ast_{t_0}(\mathcal{Z}^{-1}_t(x))}{\rd\mu(x)}\right\|_2\leq C(\mathcal{L}_{\rho})$. As a consequence, $\mu^\ast_{t_0}(x)$ has a compact support since $\mu(x)$ has one. We denote the size of the support $R^\ast$, meaning $\text{supp}\left(\mu^\ast_{t_0}(x)\right)\subset \left\{x:|x|<R^\ast\right\}$.

We now derive a general formula for $\int\frac{\delta E(\rho)}{\delta \rho}(\theta,t)\rd\nu$. Recall~\eqref{changeofvariable}, we have:
\begin{equation}\label{eqn:tryprovenegative2}
\begin{aligned}
&\int_{\mathbb{R}^k}\frac{\delta E(\rho)}{\delta \rho}(\theta,t_0)\rd\nu(\theta)\\
=&\int_{\mathbb{R}^d}(p^\ast(t_0;x))^\top\left(\int_{\mathbb{R}^k}f(x,\theta)\rd\nu(\theta)\right)\rd\mu^\ast_{t_0}(x)\\
=&\int_{\mathbb{R}^d}(p^\ast(t_0;x))^\top\left(\int_{\mathbb{R}^k}f(x,\theta)\rd{\nu}(\theta)+p^\ast(t_0;x)\right)\rd\mu_{t_0}(x)\\
&-\int_{\mathbb{R}^d}(p^\ast(t_0;x))^\top p^\ast(t_0;x) \rd\mu^\ast_{t_0}(x)\\
=&\int_{\mathbb{R}^d}(p^\ast(t_0;x))^\top\left(\int_{\mathbb{R}^k}f(x,\theta)\rd\rho(\theta)+p^\ast(t_0;x)-\int_{\mathbb{R}^k}f(x,\theta)\rd(\rho-\nu)(\theta)\right)\rd\mu_{t_0}(x)\\
&-\int_{\mathbb{R}^d}(p^\ast(t_0;x))^\top p^\ast(t_0;x) \rd\mu^\ast_{t_0}(x)\,.
\end{aligned}
\end{equation}

Noticing that according to Lemma~\ref{lem:lowerboundofprho}, if $E(\rho)\neq 0$, the second term above is strictly negative ($<-Q_2(\mathcal{L}_\rho)E(\rho)$), the goal then is to find $\nu$ so that $\int\rd\nu = 0$ and that the first term is rather small. This makes the full term $\int_{\mathbb{R}^k}\frac{\delta E(\rho)}{\delta \rho}(\theta,t_0)\rd\nu(\theta)$ negative. 

Denote a continuous function
\[
h(x)=\int_{\mathbb{R}^k}f(x,\theta)\rd\rho(\theta,t_0)+p^\ast\left(t_0;x\right)\,,
\]
then according to Assumption~\ref{assum:f}, for arbitrarily small $\epsilon$, there is a $\hat{\nu}$ so that $\int\hat{\nu}=0$ and
\[
\left\|h(x)-\int_{\mathbb{R}^k}f(x,\theta)\rd(\rho-\hat{\nu})(\theta)\right\|_{L^\infty_{|x|<R^\ast}}\leq \epsilon\,.
\]
This implies
\begin{equation}\label{prhomustarsmall}
\begin{aligned}
&\int_{\mathbb{R}^d}(p^\ast(t_0;x))^\top\left(\int_{\mathbb{R}^k}f(x,\theta)\rd\rho(\theta)+p^\ast(t_0;x)-\int_{\mathbb{R}^k}f(x,\theta)\rd(\rho-\hat{\nu})(\theta)\right)\rd\mu_{t_0}(x)\\
=&\int_{\mathbb{R}^d}(p^\ast(t_0;x))^\top\left(h(x)-\int_{\mathbb{R}^k}f(x,\theta)\rd(\rho-\hat{\nu})(\theta)\right)\rd\mu_{t_0}(x)\\
\leq&\int_{\mathbb{R}^d}\left|p^\ast(t_0;x)\right|\left|h(x)-\int_{\mathbb{R}^k}f(x,\theta)\rd(\rho-\hat{\nu})(\theta)\right|\rd\mu^\ast_{t_0}(x)\\
\leq& \left\|p^\ast(t_0;x)\right\|_{L^\infty_{|x|<R^*}}\left\|h(x)-\int_{\mathbb{R}^k}f(x,\theta)\rd(\rho-\hat{\nu})(\theta)\right\|_{L^\infty_{|x|<R^\ast}}\leq \epsilon \left\|p^\ast(t_0;x)\right\|_{L^\infty_{|x|<R^*}}
\end{aligned}\,.
\end{equation}
Choose $\epsilon$ small enough so that 
\begin{equation}\label{eqn:upperboundQE}
\begin{aligned}
&\int_{\mathbb{R}^d}(p^\ast(t_0;x))^\top\left(\int_{\mathbb{R}^k}f(x,\theta)\rd\rho(\theta)+p^\ast(t_0;x)-\int_{\mathbb{R}^k}f(x,\theta)\rd(\rho-\hat{\nu})(\theta)\right)\rd\mu_{t_0}(x)\\
\leq &\epsilon \left\|p^\ast(t_0;x)\right\|_{L^\infty_{|x|<R^*}}
\leq\frac{1}{2}Q(\mathcal{L}_\rho)E(\rho)\,.
\end{aligned}
\end{equation} 
Let $\nu = \hat{\nu}$. Plugging \eqref{eqn:upperboundQE} into \eqref{eqn:tryprovenegative2} and using Lemma~\ref{lem:lowerboundofprho} \eqref{eqn:lowerboundofprho}, we have
\[\int\frac{\delta E(\rho)}{\delta\rho}(\theta,t_0)\rd\nu(\theta)<0\,,\]
which finishes the proof.

\end{proof}

\subsection{Proof of Theorem~\ref{thm:globalminimal}}\label{sec:proofofconvergencetoglobalminimal2}

%
\begin{proof}[Proof of Theorem~\ref{thm:globalminimal}]
If $\rho(\theta,t,s)$ converges to $\rho_\infty(\theta,t)$ in $\mathcal{C}([0,1];\mathcal{P}^2)$, then we have 
\[
\sup_{s\geq 0,t\in[0,1]} \, \int_{\mathbb{R}^k}|\theta|^2\rd\rho(\theta,t,s)<\infty\,.
\]
Since $\rho(\theta,t,s)$ is a weak solution to \eqref{eqn:Wassgradientflows} and it converges to $\rho_\infty(\theta,t)$, we obtain that $\partial_s\rho|_{\rho_\infty}=0$, so that\footnote{
This statement can be made rigorous. Since $\lim_{s\rightarrow\infty}d_1(\rho(\theta,t,s),\rho_\infty(\theta,t))=0$, we have from \eqref{eqn:Deltabound2} that
$
\lim_{s\rightarrow\infty}\nabla \frac{\delta E_s(\rho)}{\delta \rho(s)}(\theta,t)=\nabla \frac{\delta E(\rho)}{\delta \rho_\infty}(\theta,t)
$. Due to the bound \eqref{eqn:Deltabound2}, the convergence is uniform in $(\theta,t)\in \mathcal{B}_{r}\times [0,1]$ for any $r>0$. Following the protocol for weak solution, we multiply smooth compact support test functions on both sides of \eqref{eqn:Wassgradientflows} and integrate. 
As $s\rightarrow\infty$, the integrand converges uniformly, leading to~\eqref{eqn:stationary}.} 

\begin{equation}\label{eqn:stationary}
\nabla\cdot\left(\rho_\infty\nabla \frac{\delta E(\rho_\infty)}{\delta \rho}(\theta,t)\right)=0\,,\quad a.e.
\end{equation}
From \eqref{eqn:derivativebound}, \eqref{eqn:boundoff}, and \eqref{boundofprho},
\[
\left|\frac{\delta E(\rho_\infty)}{\delta \rho}(\theta,t)\right|\leq C(|\theta|+1),\quad \left|\nabla \frac{\delta E(\rho_\infty)}{\delta \rho}(\theta,t)\right|\leq C\,.
\]
Multiplying $\frac{\delta E(\rho_\infty)}{\delta \rho}(\theta,t)$ on the left hand-side of \eqref{eqn:stationary} and integrating by parts, since $\rho_\infty(\theta,t)$ in $\mathcal{C}([0,1];\mathcal{P}^2)$, the boundary term vanishes. And we obtain
\[
\int^1_0\int_{\mathbb{R}^k}\left|\nabla \frac{\delta E(\rho_\infty)}{\delta \rho}(\theta,t)\right|^2\rd\rho_\infty\rd t=0\,,
\]
From continuity of $\nabla \frac{\delta E(\rho_\infty)}{\delta \rho}$ and $\rho_\infty(\theta,t)$ in $\mathcal{C}([0,1];\mathcal{P}^2)$, 
$\int_{\mathbb{R}^k}\left|\nabla \frac{\delta E(\rho_\infty)}{\delta \rho}(\theta,t)\right|^2\rd\rho_\infty$
is continuous in $t$. Thus, we finally have
\[
\rho_\infty(\theta,t)\nabla \frac{\delta E(\rho_\infty)}{\delta \rho}(\theta,t)=0\,,\quad \forall t\in[0,1]
\]

When the support of $\rho_{\infty}(\theta,t_0)$ is $\mathbb{R}^k$, we have from continuity of $\nabla \frac{\delta E(\rho_\infty)}{\delta \rho}$ in~\eqref{Lipshictz} that
\[
\nabla \frac{\delta E(\rho_\infty)}{\delta \rho}(\theta,t_0)=0\,,\quad \forall \theta\in \mathbb{R}^k\,,
\]
implying that there is a constant $c$ such that
\[
\frac{\delta E(\rho_\infty)}{\delta \rho}(\theta,t_0)=c\,,\quad \forall \theta\in \mathbb{R}^k\,.
\]

When $E(\rho)>0$, then according to Proposition~\ref{prop:localisglobal}, we can find a measure $\nu(\theta)$ so that
\[
\int_{\mathbb{R}^k}\rd\nu(\theta)=0\,\quad\text{and}\quad\int_{\mathbb{R}^k}\frac{\delta E(\rho_\infty)}{\delta \rho}(\theta,t_0)\rd\nu(\theta)<0\,.
\]
We thus have a contradiction, since 
\[
0>\int_{\mathbb{R}^k}\frac{\delta E(\rho_\infty)}{\delta \rho}(\theta,t_0)\rd\nu(\theta)=c\int_{\mathbb{R}^k}\rd\nu(\theta)=0\,.
\]
We conclude that $E(\rho_{\infty})=0$, proving the theorem.
\end{proof}

\section{Discussion of the fully-supported condition in Theorem~\ref{thm:globalminimal}}
\label{sec:proofofdiscussion}
In this section we give two examples where the fully-supported condition holds true.
\begin{proposition}\label{lemm:wholesupport}
Suppose that Assumption~\ref{assum:f} holds and that $\rho_{\ini}(\theta,t)$ is admissible. 
Let $\rho(\theta,t,s)\in \mathcal{C}([0,\infty);\mathcal{C}([0,1];\mathcal{P}^2))$ be the solution to \eqref{eqn:Wassgradientflows}. 
If the support of $\rho_{\ini}(\theta,t_0)$ is $\mathbb{R}^k$ for some $t_0\in[0,1]$, then for any finite $s$, the support of $\rho(\theta,t_0,s)$ is $\mathbb{R}^k$.
\end{proposition}
This proposition suggests that if the initial probability distribution $\rho_\ini$ is supported on the whole domain, this property is preserved for all finite $s$. A second example is as follows.
\begin{proposition}\label{prop:wholesupport2}
Suppose that Assumption~\ref{assum:f} holds, that $\rho_{\ini}(\theta,t)$ is admissible, and that the support of $\rho_{\ini}(\theta,t_0)$ is $\mathbb{R}^k$ for some $t_0\in[0,1]$. Let $\rho(\theta,t,s)\in \mathcal{C}([0,\infty);\mathcal{C}([0,1];\mathcal{P}^2))$ be the solution to \eqref{eqn:Wassgradientflows}, and $\theta_{\rho}(s;t,\theta_0)$ be the solution to corresponding particle presentation:
\begin{equation}\label{eqn:Wassgradientflowsdis1}
\left\{
\begin{aligned}
&\frac{\partial \theta_{\rho}(s;t,\theta_0)}{\partial s}=-\nabla_{\theta}\frac{\delta E_s(\rho(s))}{\delta \rho}(\theta_{\rho}(s;t,\theta_0),t)\,,\quad \forall s>0,\ t\in[0,1]\\
&\theta_{\rho}(0;t,\theta_0)=\theta_0\,.
\end{aligned}
\right.
\end{equation}
If $\lim_{s\rightarrow\infty}\theta_\rho(s;t,\theta_0)=\theta_\rho(\infty;t,\theta_0)$ uniformly in $\theta_0,t$, then there exists $\rho_\infty(\theta,t)$ so that $\lim_{s\to\infty}\rho(\theta,t,s)=\rho_\infty(\theta,t)$ in $\mathcal{C}([0,1];\mathcal{P}^2)$ and the support of $\rho_\infty(\theta,t_0)$ is the full space $\mathbb{R}^k$.
\end{proposition}
This proposition validates the condition required in Theorem \ref{thm:consistent}, thus allowing $E(\rho_\infty)=0$ in the prescribed situation.

We prove the two propositions here.

\begin{proof}[Proof of Proposition~\ref{lemm:wholesupport}]
According to Section \ref{sec:Step1}, $\rho(\theta,t,s)$ has a particle representation from the solution of the ODE \eqref{eqn:ODEdeltaEs}. 
Let $S>0$ be fixed. 
According to Lemma~\ref{lemmadeltaEs}, we have for any $s \in [0,S]$ that $\nabla_{\theta}\frac{\delta E_s(\rho(s))}{\delta \rho}$ is Lipschitz in $\theta$. Therefore, the following ODE also has a unique solution for $s\in[0,S]$ for any $t\in[0,1]$ and $\theta_S\in\mathbb{R}^k$:
\[
\left\{
\begin{aligned}
\frac{\partial \theta(s;t)}{\partial s} &=-\nabla_{\theta}\frac{\delta E_s(\rho(s))}{\delta \rho}\left(\theta(s;t),t\right)\\
\theta(S;t) &=\theta_S
\end{aligned}
\right.\,.
\]
This fact suggests that \eqref{eqn:ODEdeltaEs} produces a bijection from $\theta(0;t)\in\mathbb{R}^k$ to $\theta(S;t)\in\mathbb{R}^k$. Because the solution to \eqref{eqn:ODEdeltaEs} also depends continuously on initial data and the support of $\rho_{\ini}(\theta,t_0)$ is $\mathbb{R}^k$, we obtain that the support of $\rho(\theta,t_0,s)$ is also $\mathbb{R}^k$ for any $s<\infty$, proving the result.
\end{proof}

\begin{proof}[Proof of Proposition~\ref{prop:wholesupport2}]
According to Section \ref{sec:Step1}, $\rho(\theta,t,s)$ has a particle representation from the solution of the ODE \eqref{eqn:Wassgradientflowsdis1}, where $\theta_0$ is replaced by $\theta_0(t)$ and $\theta_0(t)\sim\rho_{\ini}(\theta,t)$. 
We define $\rho_\infty(\theta,t)$ as the distribution of $\theta_\rho(\infty;t,\theta_0(t))$. 
Since 
\[
\lim_{s\rightarrow\infty}\theta_\rho(s;t,\theta)=\theta_\rho(\infty;t,\theta) 
\]
uniformly in $\theta,t$, we obtain that
\[
\lim_{s\rightarrow\infty}\sup_{t\in[0,1]}W_2\left(\rho_\infty(\theta,t),\rho(\theta,t,s)\right)\leq \lim_{s\rightarrow\infty}\sup_{t\in[0,1]}\mathbb{E}\left(|\theta_\rho(s;t,\theta_0(t))-\theta_\rho(\infty;t,\theta_0(t))|^2\right)=0\,,
\]
which proves the convergence condition of Theorem~\ref{thm:globalminimal}. 

Next, consider any small ball $B_r\left(\widetilde{\theta}\right)\subset \mathbb{R}^k$, where $\widetilde{\theta}\in\mathbb{R}^k$ and $r>0$. According to Proposition~\ref{lemm:wholesupport}, we obtain that 
\[
\int_{B_{r/2}\left(\widetilde{\theta}\right)}d\rho(\theta,t_0,s)>0\,, \quad \mbox{for all $s>0$},
\]
which implies that
\[
\mathbb{P}\left(\theta_\rho(s;t,\theta_0(t))\in B_{r/2}\left(\widetilde{\theta}\right)\right)>0\,, 
\quad \mbox{for all $s>0$}.
\]
Because of the uniformly convergence, we can choose $S$ large enough so that
\[
\sup_{t\in[0,1], \, \theta\in\mathbb{R}^k}|\theta_\rho(S;t,\theta)-\theta_\rho(\infty;t,\theta)|\leq \frac{r}{2}\,.
\]
We then obtain
\[
\int_{B_r\left(\widetilde{\theta}\right)}d\rho_\infty(\theta,t_0)
 =\mathbb{P}\left(\theta_\rho(\infty;t,\theta_0(t))\in B_{r}\left(\widetilde{\theta}\right)\right)
\geq \mathbb{P}\left(\theta_\rho(S;t,\theta_0(t))\in B_{r/2}\left(\widetilde{\theta}\right)\right)>0\,.
\]
Since $B_r\left(\widetilde{\theta}\right)$ is a arbitrary ball in $\mathbb{R}^k$, this proves that the support of $\rho_\infty(\theta,t_0)$ is $\mathbb{R}^k$.
\end{proof}

\section{Proof of Theorem~\ref{thm:consistent}}\label{sec:proofofthmconsistent}
We start by giving the full definition of limit-admissible $\rho$.
\begin{definition}\label{def:meanadmissible2}
For an admissible $\rho(\theta,t)$, we say $\rho(\theta,t)$ is {\em limit-admissible} if the average of a large number of particle presentations is bounded and Lipschitz with high probability. That is, for an admissible $\rho(\theta,t)$, there are two constants $C_3$ and $C_4$, both greater than $\sup_{t\in[0,1]}\int_{\mathbb{R}^k}|\theta|^2d\rho(\theta,t)$ such that, for any $M$ stochastic process presentation $\{\theta_m(t)\}^M_{m=1}$ that are $i.i.d.$ drawn from $\rho(\theta,t)$, the following properties are satisfied for any $\eta>0$ and $M>\frac{C_3}{\eta}$:
\begin{itemize}
\item[1.] Second moment boundedness in time:
\begin{equation}\label{boundLdis0}
\mathbb{P}\left(\sup_{t\in[0,1]}\frac{1}{M}\sum^M_{m=1}|\theta_m(t)|^2 \leq C_4\right)\geq 1-\eta\,.
\end{equation}
\item[2.] For all $L>0$, we have
\begin{equation}\label{diffrencesmall}
\mathbb{P}\left(\frac{1}{M}\sum^{L-1}_{l=0}\sum^M_{m=1}\int^{\frac{l+1}{L}}_{\frac{l}{L}}\left|\theta_m(t)-\theta_m\left(\frac{l}{L}\right)\right|^2\rd t\leq \frac{C_4}{L^2}\right)\geq 1-\eta\,.
\end{equation}
\end{itemize}
\end{definition}

\begin{remark}\label{remark3.1}
We note that it is relatively easy to have $\rho(\theta,t)$ be limit-admissible and satisfy \eqref{boundLdis0} and \eqref{diffrencesmall}. For example, if $\rho(\theta,t)=\rho(\theta)$ and $\theta(t)=\theta$, then \eqref{diffrencesmall} is directly satisfied. Furthermore, if $\int_{\mathbb{R}^k}|\theta|^4\rd\rho(\theta)$ is finite, then using Chebyshev's inequality,
we have 
\[
\begin{aligned}
&\mathbb{P}\left(\left|\sup_{t\in[0,1]}\frac{1}{M}\sum^M_{m=1}|\theta_m(t)|^2-\int_{\mathbb{R}^k}|\theta|^2\rd\rho(\theta)\right|\leq \left(\int_{\mathbb{R}^k}|\theta|^4\rd\rho(\theta)\right)^{1/2}\right)\\
& =\mathbb{P}\left(\left|\frac{1}{M}\sum^M_{m=1}\left(|\theta_m|^2-\int_{\mathbb{R}^k}|\theta|^2d\rho(\theta)\right)\right|\leq \left(\int_{\mathbb{R}^k}|\theta|^4\rd\rho(\theta)\right)^{1/2}\right)\\
& \geq 1-\frac{\mathbb{E}\left(\left|\frac{1}{M}\sum^M_{m=1}\left(|\theta_m|^2-\int_{\mathbb{R}^k}|\theta|^2\rd\rho(\theta)\right)\right|^2\right)}{\int_{\mathbb{R}^k}|\theta|^4\rd\rho(\theta)}\\
& \geq 1-\frac{1}{M}\,.
\end{aligned}\,.
\]
Since $\int_{\mathbb{R}^k}|\theta|^2\rd\rho(\theta)\leq \left(\int_{\mathbb{R}^k}|\theta|^4\rd\rho(\theta)\right)^{1/2}$, this implies \eqref{boundLdis0} with 
\[
C_3=1,\quad  C_4=2\left(\int_{\mathbb{R}^k}|\theta|^4\rd\rho(\theta)\right)^{1/2}\,.
\]
\end{remark}

Theorem~\ref{thm:consistent} addresses the convergence in both $M$ and $L$. 
We organize the discussion into the following two theorems, which are subsequently proved in the latter two sections.

The first theorem addresses the limit in $M$, specifically,  the setup in which ResNet has infinite layers $L=\infty$ and finite width $M<\infty$, but we are interested in the limiting loss for lager $M$. 
We refer this part of analysis as ``mean-field analysis.''

\begin{theorem}\label{thm:mean-field}
Suppose that Assumption \ref{assum:f} holds, $\rho_{\ini}(\theta,t)$ is limit-admissible, and $\{\theta_m(0;t)\}^M_{m=1}$ are $i.i.d$ drawn from $\rho_{\ini}(\theta,t)$. Suppose too that
\begin{itemize}
\item $\rho(\theta,t,s)$ solves \eqref{eqn:Wassgradientflows} with the initial condition $\rho_{\ini}(\theta,t)$, and 
\item $\theta_m(s;t)$ solves \eqref{eqn:Wassgradientflowsdis} with the initial condition $\theta_m(0;t)$.
\end{itemize}
Then for any positive values of $\epsilon$, $\eta$, and $S$, there exists a constant $C(\rho_{\ini}(\theta,t),S)>0$ depending on $\rho_{\ini}(\theta,t)$ and $S$ such that when 
\[
M>\frac{C(\rho_{\ini}(\theta,t),S)}{\epsilon^2\eta}\,,
\]
we have
\[
\mathbb{P}\left(\left|E(\Theta(s;\cdot))-E(\rho(\cdot,\cdot,s))\right|\leq \epsilon\right)\geq 1-\eta\,,\quad \forall s<S\,.
\]
\end{theorem}
\begin{proof}
See Appendix~\ref{sec:proofofthmmean-field}.
\end{proof}

The conclusion of this result suggests that for a $1-\eta$ confidence of an $\epsilon$ accuracy, $M$ grows polynomially with respect to  $1/\epsilon$ and $1/\eta$.

The second result considers the convergence of discrete ResNet \eqref{eqn:disRes} to continuous ResNet  \eqref{eqn:contRes} as $L \to \infty$. This part of analysis is called ``continuous limiting analysis.''
\begin{theorem}\label{thm:contlimit}
Suppose that Assumption \ref{assum:f} holds, $\rho_{\ini}(\theta,t)$ is limit-admissible, and $\{\theta_m(0;t)\}^M_{m=1}$ are $i.i.d$ drawn from $\rho_{\ini}(\theta,t)$. 
Suppose too that
\begin{itemize}
\item $\theta_m(s;t)$ solves~\eqref{eqn:Wassgradientflowsdis} with initial condition $\theta_m(0;t)$, and
\item $\theta_{l,m}(s)$ solves~\eqref{eqn:classicalgradientflowfiniteL} with initial condition $\theta_m\left(0;\frac{l}{L}\right)$\,.
\end{itemize}
Then for any positive $\epsilon$, $\eta$, and $S$, there exists a constant $C(\rho_{\ini}(\theta,t),S)>0$ depending on  $\rho_{\ini}(\theta,t)$ and $S$ such that when
\[
M\geq \frac{C(\rho_{\ini}(\theta,t),S)}{\eta},\quad L\geq \frac{C(\rho_{\ini}(\theta,t),S)}{\epsilon}
\]
we have for all $s<S$ that 
\[
\mathbb{P}\left(\left|E(\Theta(s;\cdot))-E(\Theta_{L,M}(s))\right|\leq \epsilon\right)\geq 1-\eta\,.
\]
\end{theorem}
\begin{proof}
See Appendix~\ref{sec:proofofcontlimit}.
\end{proof}

This theorem shows that when the width $M$ is large enough, then with high probability, the difference of loss functions between the discrete ResNet and its continuous limit decreases to $0$ as $L\rightarrow\infty$.

\begin{remark}
According to the proof in Appendix~\ref{sec:proofofcontlimit}, the result of Theorem~\ref{thm:contlimit} can be strengthened: In particular, the lower bound of $M$ could be relaxed. If one incorporates $\eta$ dependence in $C_4$ (in~\eqref{boundLdis0}) and let it be large enough, the lower bound in $M$ in Definition~\ref{def:meanadmissible2} can be removed, relaxing the requirement on $M$ in the previous theorem.


\end{remark}

\section{Convergence to the mean-field PDE}\label{sec:proofofthmmean-field}
This section is dedicated to mean-field analysis and the proof of Theorem~\ref{thm:mean-field}. 
We first present some intuition.

\subsection{An intuitive argument for the equivalence}\label{sec:equivalence}

To intuitively see the equivalence between~\eqref{eqn:Wassgradientflowsdis} and~\eqref{eqn:Wassgradientflows}, we test them on the same smooth function $h(\theta)$. Testing~\eqref{eqn:Wassgradientflows} amounts to multiplying $h$ on both sides and performing integration by parts:
\[
\frac{\rd}{\rd s}\int_{\mathbb{R}^k}h\rd\rho(\theta) = -\int_{\mathbb{R}^k}\nabla_\theta h\nabla_\theta\frac{\delta E_s(\rho(s))}{\delta\rho}\rd{\rho}\,.
\]
This is to say $\frac{\rd}{\rd s}\mathbb{E}(h)=\EE(\nabla_\theta h\nabla_\theta\frac{\delta E_s(\rho(s))}{\delta \rho})$. Testing $h$ on~\eqref{eqn:Wassgradientflowsdis}, we should expect the same equivalence.

Setting $\rho = \frac{1}{M}\sum_{m=1}^M\delta_{\theta_m}$, we have from~\eqref{eqn:Wassgradientflowsdis} that
\[
\frac{\rd}{\rd s}\mathbb{E}(h) = \frac{1}{M}\sum_{m=1}^M\nabla_\theta h(\theta_m)\frac{\rd}{\rd s}\theta_m = -\sum_{m=1}^M\nabla_\theta h(\theta_m) \frac{\delta E_s}{\delta\theta_m}\,.
\]
The right hand side is also $\EE(\nabla_\theta h\nabla_\theta\frac{\delta E_s(\rho(s))}{\delta \rho})$ if and only if
\[
M\frac{\delta E_s(\Theta(s))}{\delta \theta_m}=\nabla_\theta\frac{\delta E_s(\rho(s))}{\delta\rho}(\theta_m,t)\,,
\]
a fact that we show rigorously below.

\begin{lemma}\label{lem:equivalence} For a list of continuous trajectories $\Theta(t)=\{\theta_{m}(t)\}_{m=1}^M$, denote 
\[
\rho_{\Theta}=\frac{1}{M}\sum^M_{m=1}\delta_{\theta_m(t)}(\theta)
\]
then for all $1\leq m\leq M$ and $t\in[0,1]$, we have
\begin{equation}\label{eqn:forceequivalence}
\frac{\delta E_s(\Theta)}{\delta \theta_m}(t)=\frac{1}{M}\nabla_\theta\frac{\delta E_s(\rho_\Theta)}{\delta \rho}(\theta_m(t),t)\,.
\end{equation}
\end{lemma}
\begin{proof} 
For every fixed $m$ with $1\leq m\leq M$, we denote by $\widetilde{\theta}_m(t)$ the direction of perturbation, and define the new system
\[
\widetilde{\Theta}_{\epsilon}(t)=\left\{\widetilde{\theta}_{i,\epsilon}(t)\right\}_{i=1}^M\,,
\]
by setting $\widetilde{\theta}_{i,\epsilon}(t)=\theta_{i}(t)$ for $i\neq m$ and $\widetilde{\theta}_{m,\epsilon}(t)=\theta_{m}(t)+\epsilon\widetilde{\theta}_m(t)$ for small $\epsilon>0$. Using the definition of the Fr\'echet derivative, we have
\[
\begin{aligned}
\int^1_0 \frac{\delta E_s(\Theta)}{\delta \theta_m}(t)\widetilde{\theta}_m(t)\rd t &=\lim_{\epsilon\rightarrow0} \frac{E_s\left(\widetilde{\Theta}_{\epsilon}\right)-E_s(\Theta)}{\epsilon}\\
& =\lim_{\epsilon\rightarrow0} \frac{E_s\left(\rho_{\widetilde{\Theta},\epsilon}\right)-E_s(\rho_{\Theta})}{\epsilon}\\
& \stackrel{\mathrm{(I)}}{=}\lim_{\epsilon\rightarrow0} \frac{1}{\epsilon}\left[\int^1_0\int_{\mathbb{R}^k}\frac{\delta E_s(\rho_\Theta)}{\delta \rho}\left(\rd \rho_{\widetilde{\Theta},\epsilon}-\rd \rho_{\Theta}\right)\right]\\
& =\lim_{\epsilon\rightarrow0} \frac{1}{M\epsilon}\left[\int^1_0\left(\frac{\delta E_s(\rho_\Theta)}{\delta \rho}(\widetilde{\theta}_{m,\epsilon}(t),t)-\frac{\delta E_s(\rho_\Theta)}{\delta \rho}(\theta_{m}(t),t)\right)\rd t\right]\\
& =\frac{1}{M}\int^1_0\nabla_\theta\frac{\delta E_s(\rho_\Theta)}{\delta \rho}(\theta_m(t),t)\widetilde{\theta}_m(t)\rd t\,.
\end{aligned}
\]
In (I), we use $d_1(\rho_{\widetilde{\Theta},\epsilon},\rho_{\Theta,\epsilon})=O(\epsilon)$ and a property of Fr\'echet derivative (\citep[Theorem 2]{pmlr-v119-lu20b}) to obtain that
\[
\left|E_s(\rho_{\widetilde{\Theta},\epsilon})-E_s(\rho_{\Theta})-\left[\int^1_0\int_{\mathbb{R}^k}\frac{\delta E_s(\rho_\Theta)}{\delta \rho}(\theta,t)\left(\rd \rho_{\widetilde{\Theta},\epsilon}-\rd \rho_{\Theta}\right)\rd t\right]\right|\leq o(d_1(\rho_{\widetilde{\Theta},\epsilon},\rho_{\Theta,\epsilon}))=o(\epsilon)\,.
\]
Because $\widetilde{\theta}_m(t)$ is an arbitrary perturbation, the result \eqref{eqn:forceequivalence} is proved.
\end{proof}

\begin{remark}\label{re:G.1}
As a direct consequence of this result, we can verify Remark~\ref{rmk:equivalence} rigorously as well.
\begin{proof}[Proof of Remark \ref{rmk:equivalence}]
From Lemma~\ref{lem:equivalence}, we see that~\eqref{eqn:Wassgradientflowsdis} can be written as
\begin{equation}\label{eqn:Wassgradientflowsdisequi}
\frac{\rd\theta_m(s;t)}{\rd s}=-\nabla_\theta\frac{\delta E_s(\rho^{\dis}_{\Theta})}{\delta \rho}(\theta_m(s;t),t)\,, 
\quad m=1,2,\dotsc,M\,.
\end{equation}

For any smooth test function $h$ and $t\in[0,1]$, we have
\[
\begin{aligned}
\frac{\rd}{\rd s}\left(\int_{\mathbb{R}^k}h(\theta)\rd\rho^{\dis}_{\Theta}(\theta,t,s)\right)
&=\frac{\rd }{\rd s}\left(\frac{1}{M}\sum^M_{m=1}h(\theta_m(s;t))\right)\\
& =\frac{1}{M}\sum^M_{m=1}-\nabla h(\theta_m(s;t))\cdot \nabla_\theta\frac{\delta E_s(\rho^{\dis}_{\Theta})}{\delta \rho}(\theta_m(s;t),t)\\
& =-\int_{\mathbb{R}^k}\nabla h(\theta)\cdot \nabla_\theta\frac{\delta E_s(\rho^{\dis}_{\Theta})}{\delta \rho}(\theta,t)\rd\rho^{\dis}_{\Theta}(\theta,t,s)\,,
\end{aligned}
\]
which implies from  integration by parts that $\rho^{\dis}_{\Theta}(\theta,t,s)$ is a weak solution to~\eqref{eqn:Wassgradientflows}.
\end{proof}
\end{remark}

\subsection{Discussion of stability in the mean-field regime}\label{sec:meanfieldstability}

We first provide a priori estimates to show the stability of $Z$ and $p$. Let $\rho(\theta,t)\in \mathcal{C}([0,1];\mathcal{P}^2)$, and $\left\{\widetilde{\theta}_m(t)\right\}^M_{m=1}$ and $\left\{\theta_m(t)\right\}^M_{m=1}$ be two distinct paths. Defining
\begin{equation} \label{eq:jw1}
\rho^{\dis}(\theta,t)=\frac{1}{M}\sum^M_{m=1}\delta_{\theta_m(t)}(\theta),\quad \widetilde{\rho}^{\dis}(\theta,t)=\frac{1}{M}\sum^M_{m=1}\delta_{\widetilde{\theta}_m(t)}(\theta)\,,
\end{equation}
we have the following lemma.
\begin{lemma}\label{leamm:meanfiledstability}
Let $\rho^{\dis}$ and $\widetilde{\rho}^{\dis}$ be as defined in \eqref{eq:jw1}. 
Suppose that  $Z_{\rho}$ and $Z_{\rho^{\dis}}$ are the solutions of \eqref{eqn:meancontRes} using $\rho$ and $\rho^{\dis}$, respectively, and $p_{\rho}$ and $p_{\rho^{\dis}}$ solve~\eqref{eqn:prho} using $\rho$ and $\rho^{\dis}$, respectively. Denote
\begin{equation}\label{eqn:Lmean-field}
\mathcal{L}^{\sup}=\sup_{t\in[0,1]}\left\{\int_{\mathbb{R}^k}|\theta|^2\rd\rho(\theta,u)\rd u,\ \frac{1}{M}\sum^M_{m=1}|\theta_m(t)|^2,\ \frac{1}{M}\sum^M_{m=1}\left|\widetilde{\theta}_m(t)\right|^2\right\}\,.
\end{equation}
There exists a constant $C(\mathcal{L}^{\sup})$ that depends only on $\mathcal{L}^{\sup}$ such that for all $t\in[0,1]$, we have
\begin{equation}\label{eqn:stabilityofZmean-field}
\begin{aligned}
& \left|Z_{\rho}(t;x)-Z_{\rho^{\dis}}(u;x)\right| \\
& \leq C(\mathcal{L}^{\sup})\left(\frac{1}{M}\sum^M_{m=1}\int^1_0\left|\theta_m(\tau)-\widetilde{\theta}_m(\tau)\right|^2\rd \tau\right)^{1/2}\\
& \quad +C(\mathcal{L}^{\sup})\left(\int^1_0\left|\int_{\mathbb{R}^k} f\left(Z_{\rho}(\tau;x),\theta\right)\rd(\rho(\theta,\tau)-\widetilde{\rho}^{\dis}(\theta,\tau))\right|^2\rd \tau\right)^{1/2}\,,
\end{aligned}
\end{equation}
and
\begin{equation}\label{eqn:stabilityofpmean-field}
\begin{aligned}
& \left|p_\rho(t;x)-p_{\rho^{\dis}}(t;x)\right| \\
& \leq C(\mathcal{L}^{\sup})\left(\frac{1}{M}\sum^M_{m=1}\int^1_0\left|\theta_m(\tau)-\widetilde{\theta}_m(\tau)\right|^2\rd \tau\right)^{1/2}\\
& \quad +C(\mathcal{L}^{\sup})\left(\int^1_0\left|\int_{\mathbb{R}^k} f\left(Z_{\rho}(\tau;x),\theta\right)\rd(\rho(\theta,\tau)-\widetilde{\rho}^{\dis}(\theta,\tau))\right|^2\rd \tau\right)^{1/2}\\
& \quad +C(\mathcal{L}^{\sup})\left(\int^1_0\left|\int_{\mathbb{R}^k} \partial_z f\left(Z_{\rho}(\tau;x),\theta\right)\rd(\rho(\theta,\tau)-\widetilde{\rho}^{\dis}(\theta,\tau))\right|^2\rd \tau\right)^{1/2}\,.
\end{aligned}
\end{equation}
\end{lemma}
\begin{proof} 
For ease of  notation, we define
\begin{alignat*}{3}
Z(t;x)&=Z_{\rho}(t;x),\quad & Z_{\dis}(t;x)&=Z_{\rho^{\dis}}(t;x),\quad & \widetilde{Z}_{\dis}(t;x)&=Z_{\widetilde{\rho}^{\dis}}(t;x)\,,\\
p(t;x)&=p_{\rho}(t;x),\quad & p_{\dis}(t;x)&=p_{\rho^{\dis}}(t;x),\quad & \widetilde{p}_{\dis}(t;x)&=p_{\widetilde{\rho}^{\dis}}(t;x)\,,
\end{alignat*}
and denote
\[
\widetilde{\Delta}(t;x)=Z(t;x)-\widetilde{Z}_{\dis}(t;x),\;\; \Delta(t;x)=\widetilde{Z}_{\dis}(t;x)-Z_{\dis}(t;x),\;\;
\Delta_p(t;x)=p(t;x)-p_{\dis}(t;x)\,.
\]
To prove \eqref{eqn:stabilityofZmean-field}, we have by similar reasoning to \eqref{boundofDeltatZrho} that
\[
\frac{\rd \left|\widetilde{\Delta}(t;x)\right|^2}{\rd t}\leq \left(\frac{2C_1}{M}\sum^M_{m=1}\left|\widetilde{\theta}_m(t)\right|+1\right)\left|\widetilde{\Delta}(t;x)\right|^2+\left|\int_{\mathbb{R}^k} f\left(Z(t;x),\theta\right)\rd(\rho(\theta,t)-\widetilde{\rho}^{\dis}(\theta,t))\right|^2\,.
\]
Using the Gr\"onwall inequality and the fact that $|\widetilde{\Delta}(0;x)|=0$, we obtain the following, for all $t\in[0,1]$:
\begin{equation}\label{eqn:stabilityofZmean-field1}
\left|\widetilde{\Delta}(t;x)\right|\leq C(\mathcal{L}^{\sup})\left(\int^1_0\left|\int_{\mathbb{R}^k} f\left(Z(\tau;x),\theta\right)\rd(\rho(\theta,\tau)-\widetilde{\rho}^{\dis}(\theta,\tau))\right|^2\rd \tau\right)^{1/2}\,.
\end{equation}
Similarly, we have
\[
\begin{aligned}
\frac{\rd \left|\Delta(t;x)\right|^2}{\rd t} & \leq \left(\frac{2C_1}{M}\sum^M_{m=1}|\theta_m(t)|^2+1\right)\left|\Delta(t;x)\right|^2\\
& \quad +\left|\int_{\mathbb{R}^k} f\left(\widetilde{Z}_{\dis}(t;x),\theta\right)d(\rho^{\dis}(\theta,t)-\widetilde{\rho}^{\dis}(\theta,t))\right|^2\\
& \leq C(\mathcal{L}^{\sup})\left|\Delta(t;x)\right|^2+C_1^2\left(\left|\widetilde{Z}_{\dis}(t;x)\right|+1\right)^2\left(\frac{1}{M}\sum^M_{m=1}\left|\theta_m(t)-\widetilde{\theta}_m(t)\right|\right)^2\\
& \leq C(\mathcal{L}^{\sup})\left[\left|\Delta(t;x)\right|^2+\left(\frac{1}{M}\sum^M_{m=1}\left|\theta_m(t)-\widetilde{\theta}_m(t)\right|^2\right)\right]\,,
\end{aligned}
\]
where we use Assumption \ref{assum:f} \eqref{eqn:derivativebound} in the second inequality and \eqref{boundofxsolution} in the last inequality. Since $\Delta(0;x)=0$, we use the Gr\"onwall's inequality to obtain the following, for all $t \in [0,1]$:
\begin{equation}\label{eqn:stabilityofZmean-field2}
\left|\Delta(t;x)\right|\leq C(\mathcal{L}^{\sup})\left(\frac{1}{M}\sum^M_{m=1}\int^1_0\left|\theta_m(\tau)-\widetilde{\theta}_m(\tau)\right|^2\rd \tau\right)^{1/2}.
\end{equation}
We obtain  \eqref{eqn:stabilityofZmean-field} by adding \eqref{eqn:stabilityofZmean-field1} and \eqref{eqn:stabilityofZmean-field2}.

To prove \eqref{eqn:stabilityofpmean-field}, we recall~\eqref{eqn:prho} to obtain
\begin{equation}\label{eqn:boundofDeltap}
\begin{aligned}
&\frac{\rd|\Delta_p(t;x)|^2}{\rd t}\\
&=2\left\langle \Delta_p(t;x),p^\top_{\rho}\int_{\mathbb{R}^k}\partial_zf(Z(t;x),\theta)\rd\rho(\theta,t)-p^\top_{\rho^\dis}\int_{\mathbb{R}^k}\partial_zf(Z_{\dis}(t;x),\theta)\rd\rho^{\dis}(\theta,t))\right\rangle\\
& = 2\left\langle \Delta_p(t;x),(p^\top_{\rho}-p^\top_{\rho^\dis})\int_{\mathbb{R}^k}\partial_zf(Z(t;x),\theta)\rd\rho(\theta,t)\right\rangle\\
& \quad +2\left\langle \Delta_p(t;x),p^\top_{\rho^\dis}\left(\int_{\mathbb{R}^k}\partial_zf(Z(t;x),\theta)\rd\rho(\theta,t)-\int_{\mathbb{R}^k}\partial_zf(Z_{\dis}(t;x),\theta)\rd\rho^{\dis}(\theta,t))\right)\right\rangle\\
& \leq 2\left|\int_{\mathbb{R}^k}\partial_zf(Z(t;x),\theta)\rd\rho(\theta,t)\right||\Delta_p(t;x)|^2\\
& \quad +2|\Delta_p(t;x)||p_{\rho^\dis}|\left|\int_{\mathbb{R}^k}\partial_zf(Z(t;x),\theta)\rd\rho(\theta,t)-\int_{\mathbb{R}^k}\partial_zf(Z_{\dis}(t;x),\theta)\rd\rho^{\dis}(\theta,t))\right|\\
& \leq C(\mathcal{L}^{\sup})|\Delta_p(t;x)|^2\\
& \quad +2\left|\int_{\mathbb{R}^k}\partial_zf(Z(t;x),\theta)\rd\rho(\theta,t)-\int_{\mathbb{R}^k}\partial_zf(Z_{\dis}(t;x),\theta)\rd\rho^{\dis}(\theta,t)\right|^2\\
& \leq C(\mathcal{L}^{\sup})|\Delta_p(t;x)|^2\\
& \quad +6\left|\int_{\mathbb{R}^k}\partial_zf(Z(t;x),\theta)\rd\rho(\theta,t)-\int_{\mathbb{R}^k}\partial_zf(Z(t;x),\theta)\rd\widetilde{\rho}^{\dis}(\theta,t)\right|^2\\
& \quad +6\underbrace{\left|\int_{\mathbb{R}^k}\partial_zf(Z(t;x),\theta)\rd\widetilde{\rho}^{\dis}(\theta,t)-\int_{\mathbb{R}^k}\partial_zf(Z(t;x),\theta)\rd\rho^{\dis}(\theta,t)\right|^2}_{\textrm{(I)}}\\
& \quad +6\underbrace{\left|\int_{\mathbb{R}^k}\partial_zf(Z(t;x),\theta)\rd\rho^{\dis}(\theta,t)-\int_{\mathbb{R}^k}\partial_zf(Z_{\dis}(t;x),\theta)\rd\rho^{\dis}(\theta,t)\right|^2}_{\textrm{(II)}}\,,\\
\end{aligned}
\end{equation}
where we use \eqref{eqn:derivativebound} from Assumption~\ref{assum:f} along with \eqref{boundofxsolution} and \eqref{boundofprho} in the second inequality. We then bound the last two terms. 
To bound $\textrm{(I)}$, we use \eqref{eqn:derivativebound} from Assumption~\ref{assum:f} along with \eqref{boundofxsolution}, and an analysis similar to \eqref{eqn:boundpartialzfexample} to obtain
\begin{equation}\label{eqn:boundofIIDeltap}
\begin{aligned}
\textrm{(I)} & \leq \left(\frac{1}{M}\sum^M_{m=1}\left|\partial_zf(Z(t;x),\theta_m(t))-\partial_zf(Z(t;x),\widetilde{\theta}_m(t))\right|\right)^2\\
& \leq C(\mathcal{L}^{\sup})\left(\frac{1}{M}\sum^M_{m=1}\left(|\theta_m(t)|+|\widetilde{\theta}_m(t)|\right)|\theta_m(t)-\widetilde{\theta}_m(t)|\right)^2\\
& \leq C(\mathcal{L}^{\sup})\left(\frac{1}{M}\sum^M_{m=1}\left(|\theta_m(t)|+|\widetilde{\theta}_m(t)|\right)^2\right)\left(\frac{1}{M}\sum^M_{m=1}\left|\theta_m(t)-\widetilde{\theta}_m(t)\right|^2\right)\\
& \leq C(\mathcal{L}^{\sup})\left(\frac{1}{M}\sum^M_{m=1}\left|\theta_m(t)-\widetilde{\theta}_m(t)\right|^2\right)\,,
\end{aligned}
\end{equation}
where we use H\"older's inequality in the second inequality. For $\textrm{(III)}$, we use \eqref{eqn:derivativebound} from  Assumption~\ref{assum:f} along with \eqref{eqn:Lmean-field}, to obtain
\begin{equation}\label{eqn:boundofIIIDeltap}
\begin{aligned}
\textrm{(II)} & \leq \left(\frac{1}{M}\sum^M_{m=1}\left|\partial_zf(Z(t;x),\theta_m(t))-\partial_zf(Z_{\dis}(t;x),\theta_m(t))\right|\right)^2\\
& \leq \left(\frac{C_1}{M}\sum^M_{m=1}\left|\theta_m(t)\right|^2|Z(t;x)-Z_{\dis}(t;x)|\right)^2 \\
& \leq C(\mathcal{L}^{\sup})|Z(t;x)-Z_{\dis}(t;x)|^2\,.
\end{aligned}
\end{equation}
By substituting \eqref{eqn:boundofIIDeltap} and 
\eqref{eqn:boundofIIIDeltap} into \eqref{eqn:boundofDeltap}, we obtain
\begin{equation}\label{Deltapiteration}
    \begin{aligned}
\frac{\rd|\Delta_p(t;x)|^2}{\rd t}&=C(\mathcal{L}^{\sup})|\Delta_p(t;x)|^2\\
& \quad +C(\mathcal{L}^{\sup})\left(\left(\frac{1}{M}\sum^M_{m=1}\left|\theta_m(t)-\widetilde{\theta}_m(t)\right|^2\right)+|Z(t;x)-Z_{\dis}(t;x)|^2\right)\\
& \quad +6\left|\int_{\mathbb{R}^k}\partial_zf(Z(t;x),\theta)\rd\rho(\theta,t)-\int_{\mathbb{R}^k}\partial_zf(Z(t;x),\theta)\rd\widetilde{\rho}^{\dis}(\theta,t)\right|^2\,.
\end{aligned}
\end{equation}
Considering the initial condition, we obtain in a similar fashion to \eqref{eqn:ODEforDEltaprhoinitial} that
\[
|\Delta_p(1;x)|\leq C(\mathcal{L}^{\sup})|Z(1;x)-Z_{\dis}(1;x)|\,,
\]
By substituting \eqref{eqn:stabilityofZmean-field} into \eqref{Deltapiteration} and using Gr\"onwall's inequality, we arrive at \eqref{eqn:stabilityofpmean-field}.
\end{proof}

\subsection{Proof of Theorem~\ref{thm:mean-field}}

According to the statement of the theorem, $\rho(\theta,t,s)$ solves~\eqref{eqn:Wassgradientflows} with admissible initial condition $\rho_{\ini}(\theta,t)$, and $\theta_m(s;t)$ solves \eqref{eqn:Wassgradientflowsdis} with initial conditions $\{\theta_m(t,0)\}^M_{m=1}$ that are $i.i.d.$ drawn from $\rho_{\ini}(\theta,t)$. We now define
\[ 
\rho^{\dis}_{\Theta_s}(\theta,t)=\frac{1}{M}\sum^M_{m=1}\delta_{\theta_m(s;t)}(\theta)\,,
\]
with
\begin{equation}\label{eqn:Wassgradientflowsdisappendix}
\frac{\rd\theta_m(s;t)}{\rd s}=-\nabla_{\theta}\frac{\delta E_s(\rho^{\dis}_{\Theta_s})}{\delta \rho}(\theta_m(s;t))-2e^{-s}\theta_m(s;t)\,,\quad s>0,\quad t\in[0,1]\,,
\end{equation}
where $\frac{\delta E_s(\rho^{\dis}_{\Theta_s})}{\delta \rho}$ is defined in \eqref{eqn:Frechetderivative}.

Throughout this section, we denote
\[
\mathcal{L}^{\dis,\sup}_\ini=\sup_{t\in[0,1]}\frac{1}{M}\sum^M_{m=1}|\theta_m(0;t)|^2,\quad \mathcal{L}^{\sup}_{\ini}=\sup_{t\in[0,1]}\int_{\mathbb{R}^k}|\theta|^2\rd\rho_{\ini}(\theta,t)\,.
\]
We note that when $M$ is large, $\mathcal{L}^{\dis,\sup}_\ini$ is close to $\mathcal{L}^{\sup}_{\ini}$ (which has no randomness) with high probability. We have the following lemma.
\begin{lemma}\label{lem:boundofdiscrete} For fixed $S>0$, any $s \in [0,S]$, $|x|<R$, and $t\in[0,1]$, there exists a constant $C\left(S,\mathcal{L}^{\dis,\sup}_\ini\right)$ such that
\begin{equation}\label{eqn:boundofLSdis}
\int^1_0\int_{\mathbb{R}^k}|\theta|^2 \rd\rho^{\dis}_{\Theta_s}(\theta,t)\rd t<C\left(S,\mathcal{L}^{\dis,\sup}_\ini\right)\,.
\end{equation}
Further, for any $x$ with $|x|<R$, the ODE solution is bounded as follows:
\begin{equation}\label{boundofxsolutiondis}
\left|Z_{\rho^{\dis}_{\Theta_s}}(t;x)\right|\leq C\left(S,\mathcal{L}^{\dis,\sup}_\ini\right)\,,
\end{equation}
while the following bound holds on $p_{\rho^{\dis}_{\Theta_s}}$:
\begin{equation}\label{boundofprhosolutiondis}
\left|p_{\rho^{\dis}_{\Theta_s}}(t;x)\right|\leq C\left(S,\mathcal{L}^{\dis,\sup}_\ini\right)\,.
\end{equation}
\end{lemma}
\begin{proof} First, using~\eqref{eqn:cost_M_s} and \eqref{boundofxsolution}, we have that
\[
E_0\left(\rho^{\dis}_{\Theta_0}\right)<C\left(\mathcal{L}^{\dis,\sup}_\ini\right)\,.
\]
According to the definition of $\rho^\dis_{\Theta_s}$ and $E_s(\rho)$ in~\eqref{eqn:dis_rho} and~\eqref{eqn:cost_s}, and recalling decay of energy, we have as shown in~\eqref{eqn:decayEsmean-field} that
\begin{align*}
\int^1_0\int_{\mathbb{R}^k}|
\theta|^2 \rd\rho^{\dis}_{\Theta_s}(\theta,t)\rd t&=\int^1_0\frac{1}{M}\sum^M_{m=1}|\theta_m(s;t)|^2\rd t\leq \exp(s)E_s\left(\rho^{\dis}_{\Theta_s}\right)\\
&\leq \exp(S)E_0\left(\rho^{\dis}_{\Theta_0}\right)<C\left(S,\mathcal{L}^{\dis,\sup}_\ini\right)\,,
\end{align*}
for $s \in [0,S]$, proving~\eqref{eqn:boundofLSdis}.



We obtain \eqref{boundofxsolutiondis} and \eqref{boundofprhosolutiondis} from the same argument that shows \eqref{boundofxsolution} and \eqref{boundofprho}, with the $\mathcal{L}_1$ term replaced by the bound~\eqref{eqn:boundofLSdis}.
\end{proof}

We are now ready to prove Theorem~\ref{thm:mean-field}.

\begin{proof}[Proof of Theorem~\ref{thm:mean-field}]
To start, we define a new system $\widetilde{\Theta}_s=\{\widetilde{\theta}_m(s;t)\}_{m=1}^M$ where each $\widetilde{\theta}_m$ solves
\begin{equation}\label{eqn:Wassgradientflowsdiscoupling}
\frac{\partial \widetilde{\theta}_m(s;t)}{\partial s}=-\nabla_{\theta}\frac{\delta E_s(\rho(s))}{\delta \rho}\left(\widetilde{\theta}_m(s;t)\right)\,,\quad\forall (t,s)\in[0,1]\times[0,\infty)\,,
\end{equation}
with initial conditions $\widetilde{\theta}_m(t,0) = \theta_m(t,0)$. As a consequence, we have
\[
\widetilde{\theta}_m(s;t)\sim \rho(\theta,t,s)\,,\quad\forall (t,s)\in[0,1]\times[0,\infty)\,.
\]
We further denote
\begin{equation}\label{def:LSmean-field}
\widetilde{\rho}^{\dis}(\theta,t,s)=\frac{1}{M}\sum^M_{m=1}\delta_{\widetilde{\theta}_m(s;t)}(\theta)\,,\quad\mathcal{L}^{\sup}_{0}=\max\left\{\mathcal{L}^{\sup}_{\ini},\mathcal{L}^{\dis,\sup}_\ini\right\}<\infty\,,
\end{equation}
We first bound the second moment of $\widetilde{\rho}^{\dis}(\theta,t,s)$ for all $t\in[0,1]$ and $s\in[0,S]$ with initial condition $\widetilde{\rho}^{\dis}(\theta,t,0)=\frac{1}{M}\sum^M_{m=1}\delta_{\widetilde{\theta}_m(0;t)}(\theta)$. 
Similar to the proof of Proposition~\ref{proposition:deltaEODE}, we multiply~\eqref{eqn:Wassgradientflowsdiscoupling} by $\widetilde{\theta}_m(s;t)$ on both sides, and utilize the bound \eqref{bound} from Lemma~\ref{lemmadeltaEs}, where $\mathcal{L}$ in~\eqref{bound} is replaced by $C(\mathcal{L}^{\sup}_{\ini},S)$ according to \eqref{eqn:secondmomentbound2} in Corollary~\ref{cor:C.3}. 
We thus obtain
\[
|\widetilde{\theta}_m(s;t)|\leq C(\mathcal{L}^{\sup}_{\ini},S)\left(|\widetilde{\theta}_m(0;t)|+1\right)\,,
\]
which implies that
\[
\sup_{t\in[0,1],0\leq s\leq S}\frac{1}{M}\sum^M_{m=1}\left|\widetilde{\theta}_m(s;t)\right|^2\leq C(\mathcal{L}^{\sup}_{0},S)\,.
\]
By combining this bound with \eqref{eqn:secondmomentbound2} and \eqref{eqn:boundofLSdis}, we obtain that there is a constant $C(\mathcal{L}^{\sup}_{0},S)$ depending only on $\mathcal{L}^{\sup}_{0}$ and $S$ such that
\[
\sup_{t\in[0,1],0\leq s\leq S}\left\{\int_{\mathbb{R}^k}|\theta|^2d\rho(\theta,t,s),\frac{1}{M}\sum^M_{m=1}|\theta_m(s;t)|^2,\frac{1}{M}\sum^M_{m=1}\left|\widetilde{\theta}_m(s;t)\right|^2\right\}\leq C(\mathcal{L}^{\sup}_{0},S)\,.
\]

To prove the theorem, we have from the definition \eqref{eqn:costfunctioncont2} of $E$ that
\begin{equation}\label{eqn:differenceofE}
\begin{aligned}
&\left|E(\rho(s))-E(\rho^{\dis}_{\Theta_s})\right|\\
&=\mathbb{E}_{x\sim\mu}\left[\frac{1}{2}\left(g(Z_{\rho(s)}(1;x))-y(x)\right)^2-\frac{1}{2}\left(g(Z_{\rho^{\dis}_{\Theta_s}}(1;x))-y(x)\right)^2\right]\\
&\leq \mathbb{E}_{x\sim\mu}\left[\left|g(Z_{\rho(s)}(1;x))-g(Z_{\rho^{\dis}_{\Theta_s}}(1;x))\right|\left(\left|g(Z_{\rho(s)}(1;x))+g(Z_{\rho^{\dis}_{\Theta_s}}(1;x))\right|+|y(x)|\right)\right]\\
& \leq C(\mathcal{L}^{\sup}_{0},S)\EE_{x\sim\mu}\left(\left|g(Z_{\rho(s)}(1;x))-g(Z_{\rho^{\dis}_{\Theta_s}}(1;x))\right|\right)\\
& \leq C(\mathcal{L}^{\sup}_{0},S)\EE_{x\sim\mu}\left(\left|Z_{\rho(s)}(1;x)-Z_{\rho^{\dis}_{\Theta_s}}(1;x)\right|\right)\,,
\end{aligned}
\end{equation}
where the second inequality arises from boundedness of  $Z$ in \eqref{boundofxsolution} and local boundedness of $g$ and $y$, by Assumption~\ref{assum:f}.

To estimate $\left|Z_{\rho(s)}(t;x)-Z_{\rho^{\dis}_{\Theta_s}}(t;x)\right|$, we recall~\eqref{eqn:stabilityofZmean-field} to obtain
\begin{equation}\label{eqn:stabilityofZmean-fielduseit}
\begin{aligned}
& \left|Z_{\rho(s)}(t;x)-Z_{\rho^{\dis}_{\Theta_s}}(t;x)\right| \\
& \leq C(\mathcal{L}^{\sup}_{0},S)\left(\frac{1}{M}\sum^M_{m=1}\int^1_0\left|\theta_m(s;\tau)-\widetilde{\theta}_m(s;\tau)\right|^2\rd \tau\right)^{1/2}\\
& \quad +C(\mathcal{L}^{\sup}_{0},S)\left(\int^1_0\left|\int_{\mathbb{R}^k} f\left(Z_{\rho(s)}(\tau;x),\theta\right)\rd(\rho(\theta,\tau,s)-\widetilde{\rho}^{\dis}(\theta,\tau,s))\right|^2\rd \tau\right)^{1/2}\,.
\end{aligned}
\end{equation}
Since $\widetilde{\theta}_m(s;t)\sim \rho(\theta,t,s)$, the second term of \eqref{eqn:stabilityofZmean-fielduseit} can be bounded by law of large numbers, with high probability. 
Thus we need only to control the first term. 
In the analysis below we control this term in Step 1 and utilize the law of large numbers in Step 2.

\begin{enumerate}[wide,   labelindent=0pt]
\item[Step 1: Estimating $\frac{1}{M}\sum^M_{m=1}\int^1_0\left|\theta_m(s;t)-\widetilde{\theta}_m(s;t)\right|^2\rd t$.] Defining
\[
\Delta_{t,m}(s)=\theta_m(s;t)-\widetilde{\theta}_m(s;t)\,,
\]
we note that $|\Delta_{t,m}(0)|=0$. 
By taking the difference of \eqref{eqn:Wassgradientflowsdisappendix} and~\eqref{eqn:Wassgradientflowsdiscoupling}, we obtain
\begin{equation}\label{iterationmean-field}
\begin{aligned}
&\frac{\rd|\Delta_{t,m}(s)|^2}{\rd s}\\
& =-2\exp(-s)|\Delta_{t,m}(s)|^2\\
& \quad -2\left\langle\Delta_{t,m}(s),\EE_{x\sim\mu}\left(\partial_\theta f(Z_{\rho(s)}(t;x),\widetilde{\theta}_m)p_{\rho(s)}(t;x)-\partial_\theta f(Z_{\rho^{\dis}_{\Theta_s}}(t;x),\theta_m)p_{\rho^{\dis}_{\Theta_s}}(t;x)\right)\right\rangle\\
& =-2\exp(-s)|\Delta_{t,m}(s)|^2\\
& \quad -2\left\langle\Delta_{t,m}(s),\underbrace{\EE_{x\sim\mu}\left(\partial_\theta f(Z_{\rho(s)}(t;x),\widetilde{\theta}_m)p_{\rho(s)}(t;x)-\partial_\theta f(Z_{\rho(s)}(t;x),\widetilde{\theta}_m)p_{\rho^{\dis}_{\Theta_s}}(t;x)\right)}_{\textrm{(I)}}\right\rangle\\
& \quad -2\left\langle\Delta_{t,m}(s),\underbrace{\EE_{x\sim\mu}\left(\partial_\theta f(Z_{\rho(s)}(t;x),\widetilde{\theta}_m)p_{\rho^{\dis}_{\Theta_s}}(t;x)-\partial_\theta f(Z_{\rho^{\dis}_{\Theta_s}}(t;x),\theta_m)p_{\rho^{\dis}_{\Theta_s}}(t;x)\right)}_{\textrm{(II)}}\right\rangle\,.
\end{aligned}
\end{equation}
For term $\textrm{(I)}$, we have from the bound on $Z_{\rho(s)}$ in~\eqref{boundofxsolution} that
\begin{equation*}
\left|\textrm{(I)}\right|\leq C(\mathcal{L}^{\sup}_{0},S)\EE_{x\sim\mu}\left(\left|p_{\rho(s)}(t;x)-p_{\rho^{\dis}_{\Theta_s}}(t;x)\right|\right)\,.
\end{equation*}
For $\textrm{(II)}$, we have
\begin{equation*}
\begin{aligned}
\left|\textrm{(II)}\right| &\leq C(\mathcal{L}^{\sup}_{0},S)\EE_{x\sim\mu}\left(\left|\partial_\theta f(Z_{\rho(s)}(t;x),\widetilde{\theta}_m)-\partial_\theta f(Z_{\rho^{\dis}_{\Theta_s}}(t;x),\theta_m)\right|\right)\\
& \leq C(\mathcal{L}^{\sup}_{0},S)\left[\left(\left|\widetilde{\theta}_m\right|+\left|\theta_m\right|\right)\EE_{x\sim\mu}\left(\left|Z_{\rho(s)}(t;x)-Z_{\rho^{\dis}_{\Theta_s}}(t;x)\right|\right)+|\widetilde{\theta}_m-\theta_m|\right]\,.
\end{aligned}
\end{equation*}
In both estimates we used the property of $f$ in~\eqref{eqn:derivativebound}, and the bound on $Z_{\rho,\rho^\dis_{\Theta}}$
By substituting these estimates into \eqref{iterationmean-field}, we obtain
\[
\begin{aligned}
&\frac{\rd|\Delta_{t,m}(s)|^2}{\rd s}\\
& \leq C(\mathcal{L}^{\sup}_{0},S)|\Delta_{t,m}(s)|^2\\
& \quad +C(\mathcal{L}^{\sup}_{0},S)|\Delta_{t,m}(s)|\EE_{x\sim\mu}\left(\left(\left|\widetilde{\theta}_m\right|+\left|\theta_m\right|\right)\left|Z_{\rho(s)}(t;x)-Z_{\rho^{\dis}_{\Theta_s}}(t;x)\right|+\left|p_{\rho(s)}(t;x)-p_{\rho^{\dis}_{\Theta_s}}(t;x)\right|\right)\,,
\end{aligned}
\]
which implies
\begin{align*}
&\frac{1}{M}\sum^M_{m=1}\frac{\rd|\Delta_{t,m}(s)|^2}{\rd s}\\
& \leq C(\mathcal{L}^{\sup}_{0},S)\left(\frac{1}{M}\sum^M_{m=1}|\Delta_{t,m}(s)|^2\right)\\
& \quad +C(\mathcal{L}^{\sup}_{0},S)\left(\frac{1}{M}\sum^M_{m=1}|\Delta_{t,m}(s)|\left(\left|\widetilde{\theta}_m\right|+\left|\theta_m\right|\right)\right)\EE_{x\sim\mu}\left(\left|Z_{\rho(s)}(t;x)-Z_{\rho^{\dis}_{\Theta_s}}(t;x)\right|\right)\\
& \quad +C(\mathcal{L}^{\sup}_{0},S)\left(\frac{1}{M}\sum^M_{m=1}|\Delta_{t,m}(s)|\right)\EE_{x\sim\mu}\left(\left|p_{\rho(s)}(t;x)-p_{\rho^{\dis}_{\Theta_s}}(t;x)\right|\right)\\
& \leq C(\mathcal{L}^{\sup}_{0},S)\left(\frac{1}{M}\sum^M_{m=1}|\Delta_{t,m}(s)|^2\right)+C(\mathcal{L}^{\sup}_{0},S)\EE_{x\sim\mu}\left(\left|p_{\rho(s)}(t;x)-p_{\rho^{\dis}_{\Theta_s}}(t;x)\right|^2\right)\\
& \quad +C(\mathcal{L}^{\sup}_{0},S)\EE_{x\sim\mu}\left(\left|Z_{\rho(s)}(t;x)-Z_{\rho^{\dis}_{\Theta_s}}(t;x)\right|^2\right)\,,
\end{align*}
where we used the H\"older's inequality and
\begin{align*}
\left(\frac{1}{M}\sum^M_{m=1}|\Delta_{t,m}(s)|\left(\left|\widetilde{\theta}_m\right|+\left|\theta_m\right|\right)\right)^2&\leq  \left(\frac{1}{M}\sum^M_{m=1}|\Delta_{t,m}(s)|^2\right)\left(\frac{1}{M}\sum^M_{m=1}\left(\left|\widetilde{\theta}_m\right|+\left|\theta_m\right|\right)^2\right)\\
&\leq C(\mathcal{L}^{\sup}_{0},S)\left(\frac{1}{M}\sum^M_{m=1}|\Delta_{t,m}(s)|^2\right)\,.
\end{align*}

Noting the estimate in Lemma~\ref{leamm:meanfiledstability}, we obtain that
\begin{align*}
&\frac{\rd \left(\frac{1}{M}\sum^M_{m=1}|\Delta_{t,m}(s)|^2\right)}{\rd s}\\
& \leq C(\mathcal{L}^{\sup}_{0},S)\left(\frac{1}{M}\sum^M_{m=1}|\Delta_{t,m}(s)|^2\right)\\
& \quad +C(\mathcal{L}^{\sup}_{0},S)\EE_{x\sim\mu}\left(\int^1_0\left|\int_{\mathbb{R}^k} f\left(Z_{\rho(s)}(\tau;x),\theta\right)\rd(\rho(\theta,\tau,s)-\widetilde{\rho}^{\dis}(\theta,\tau,s))\right|^2\rd \tau\right)\\
& \quad +C(\mathcal{L}^{\sup}_{0},S)\EE_{x\sim\mu}\left(\int^1_0\left|\int_{\mathbb{R}^k} \partial_z f\left(Z_{\rho(s)}(\tau;x),\theta\right)\rd(\rho(\theta,\tau,s)-\widetilde{\rho}^{\dis}(\theta,\tau,s))\right|^2\rd \tau\right)\,,
\end{align*}
which implies, using Gr\"onwall's inequality, that
\begin{equation}\label{boundofcoupling}
\begin{aligned}
&\frac{1}{M}\sum^M_{m=1}\int^1_0|\Delta_{t,m}(s)|^2\rd t\\
& \leq C(\mathcal{L}^{\sup}_{0},S)\EE_{x\sim\mu}\left(\int^S_0\int^1_0\left|\int_{\mathbb{R}^k} f\left(Z_{\rho(s)}(\tau;x),\theta\right)\rd(\rho(\theta,\tau,s)-\widetilde{\rho}^{\dis}(\theta,\tau,s))\right|^2\rd \tau\rd s\right)\\
& \quad +C(\mathcal{L}^{\sup}_{0},S)\EE_{x\sim\mu}\left(\int^S_0\int^1_0\left|\int_{\mathbb{R}^k} \partial_z f\left(Z_{\rho(s)}(\tau;x),\theta\right)\rd(\rho(\theta,\tau,s)-\widetilde{\rho}^{\dis}(\theta,\tau,s))\right|^2 \rd \tau\rd s\right)\,.
\end{aligned}
\end{equation}

\item[Step 2: Complete the proof.]
Since $\rho_\ini$ satisfies the limit-admissible condition, according to~\eqref{boundLdis0}, there is a constant $C_4\geq \mathcal{L}^{\sup}_\ini$ such that $\mathbb{P}\left(\mathcal{L}^{\dis,\sup}_{\ini}\leq C_4\right)\geq 1-\eta/2$ for  $M$ sufficiently large. 
Using the definition of $\mathcal{L}^{\sup}_{0}$ in~\eqref{def:LSmean-field}, and for some constant $C_3\geq \mathcal{L}^{\sup}_\ini$ , this translates to:
\begin{equation}\label{boundLdis01}
\mathbb{P}\left(\mathcal{L}^{\sup}_{0}\leq C_4\right)\geq 1-\eta/2\,, \quad \mbox{for} \;\; M>\frac{2C_3}{\eta},
\end{equation}
so that $\mathcal{L}^{\sup}_{0}$ is bounded with high probability.

By substituting \eqref{boundofcoupling} into \eqref{eqn:stabilityofZmean-field}, we obtain
\begin{equation}\label{boundofcoupling2}
\begin{aligned}
&\left|Z_{\rho(s)}(t;x)-Z_{\rho^{\dis}_{\Theta_s}}(t;x)\right|\\
& \leq C(\mathcal{L}^{\sup}_{0},S)\left(\underbrace{\EE_{x\sim\mu}\left(\int^S_0\int^1_0\left|\int_{\mathbb{R}^k} f\left(Z_{\rho(s)}(\tau;x),\theta\right)\rd(\rho(\theta,\tau,s)-\widetilde{\rho}^{\dis}(\theta,\tau,s))\right|^2\rd \tau\rd s\right)}_{\textrm{(I)}}\right)^{1/2}\\
& \quad +C(\mathcal{L}^{\sup}_{0},S)\left(\underbrace{\EE_{x\sim\mu}\left(\int^S_0\int^1_0\left|\int_{\mathbb{R}^k} \partial_z f\left(Z_{\rho(s)}(\tau;x),\theta\right)\rd(\rho(\theta,\tau,s)-\widetilde{\rho}^{\dis}(\theta,\tau,s))\right|^2\rd \tau\rd s\right)}_{\textrm{(II)}}\right)^{1/2}\\
& \quad +C(\mathcal{L}^{\sup}_{0},S)\left(\underbrace{\int^1_0\left|\int_{\mathbb{R}^k} f\left(Z_{\rho(s)}(\tau;x),\theta\right)\rd(\rho(\theta,\tau,s)-\widetilde{\rho}^{\dis}(\theta,\tau,s))\right|^2\rd \tau}_{\textrm{(III)}}\right)^{1/2}\,.
\end{aligned}
\end{equation}
All terms $\textrm{(I)}$,  $\textrm{(II)}$, and $\textrm{(III)}$ can be controlled. We have
\begin{align*}
\mathbb{E}\textrm{(I)}& =\EE_{x\sim\mu}\left(\mathbb{E}\left(\int^S_0\int^1_0\left|\int_{\mathbb{R}^k} f\left(Z_{\rho(s)}(\tau;x),\theta\right)\rd(\rho(\theta,\tau,s)-\widetilde{\rho}^{\dis}(\theta,\tau,s))\right|^2\rd\tau\rd s\right)\right)\\
&=\EE_{x\sim\mu}\left(\int^S_0\int^1_0\mathbb{E}\left(\left|\int_{\mathbb{R}^k} f\left(Z_{\rho(s)}(\tau;x),\theta\right)\rd(\rho(\theta,\tau,s)-\widetilde{\rho}^{\dis}(\theta,\tau,s))\right|^2\right)\rd \tau\rd s\right)\\
& \leq \frac{C(\mathcal{L}^{\sup}_{\ini},S)}{M}\EE_{x\sim\mu}\left(\int^S_0\int^1_0\int_{\mathbb{R}^k}|f\left(Z_{\rho(s)}(\tau;x),\theta\right)|^2d\rho(\theta,\tau,s)\rd\tau\rd s\right)\\
& \leq \frac{C(\mathcal{L}^{\sup}_{\ini},S)}{M}\int^S_0\int^1_0\int_{\mathbb{R}^k}(|\theta|^2+1)d\rho(\theta,\tau,s)\rd\tau\rd s\leq \frac{C(C_4,S)}{M}\,,
\end{align*}
where we use $\widetilde{\theta}_m(s;t)\sim \rho(\theta,t,s)$ in the first inequality, \eqref{eqn:boundoff} with $|Z_{\rho(s)}|\leq C(\mathcal{L}^{\sup}_{\ini},S)$ in the second inequality, and $\mathcal{L}^{\sup}_{\ini}\leq C_4$ in the third inequality. 
By similar reasoning, we obtain
\begin{align*}
\mathbb{E}\textrm{(II)}\leq \frac{C(C_4,S)}{M},\quad \mathbb{E}\textrm{(III)}\leq \frac{C(C_4,S)}{M}\,.
\end{align*}
From Markov's inequality, these bounds imply that when $M>\frac{C(C_4,S)}{\epsilon^2\eta}$, we have
\begin{equation}\label{listofboundinprobability}
\mathbb{P}\left(\left\{\textrm{(I)}<\epsilon^2\right\}\cap\left\{\textrm{(II)}<\epsilon^2\right\}\cap\left\{\textrm{(III)}<\epsilon^2\right\} \right)>1-\eta/2\,.
\end{equation}

By substituting \eqref{boundLdis01} and \eqref{listofboundinprobability} into \eqref{boundofcoupling2}, we see that there exists a constant $C(C_4,S)$ such that for any $\epsilon,\eta>0$, when $M>\frac{C(C_4,S)}{\epsilon^2\eta}$ we obtain that
\[
\mathbb{P}\left(\left|Z_{\rho(s)}(1;x)-Z_{\rho^{\dis}_{\Theta_s}}(1;x)\right|<\epsilon\right)>1-\eta\,.
\]
The proof is completed by substituting this bound into \eqref{eqn:differenceofE}.
\end{enumerate}
\end{proof}

\section{Convergence to the continuous limit}\label{sec:proofofcontlimit}

This section is dedicated to the continuous limit and proof of Theorem~\ref{thm:contlimit}.

\subsection{Stability with discretization}
Before proving Theorem~\ref{thm:contlimit}, and similarly to Appendix \ref{sec:meanfieldstability}, we first consider the stability $Z$ and $p$ under discretization. 
Defining the path of parameters $\Theta(t)=\left\{\theta_m(t)\right\}^M_{m=1}$ and the set of parameters ${\Theta}_{L,M}=\left\{\theta_{l,m}\right\}^{L-1,M}_{l=0,m=1}$, we have the following lemma.
\begin{lemma}
Suppose that Assumption~\ref{assum:f} holds and that $x$ is in the support of $\mu$. 
Denoting 
\begin{equation}\label{eqn:Lmean-fieldfiniteL}
\mathcal{L}^{\sup}=\sup_{t\in[0,1],l}\left\{\frac{1}{M}\sum^M_{m=1}|\theta_m(t)|^2,\ \frac{1}{M}\sum^M_{m=1}|\theta_{l,m}|^2\right\}\,,
\end{equation}
there exists a constant $C(\mathcal{L}^{\sup})$  depending only on $\mathcal{L}^{\sup}$ such that for any $0\leq l\leq L-1$, we have
\begin{equation}\label{eqn:stabilityofZmean-fieldfiniteL}
\begin{aligned}
&\sup_{\frac{l}{L}\leq t\leq\frac{l+1}{L}}\left\{\left|Z_{\Theta}(t;x)-Z_{\Theta_{L,M}}(l;x)\right|,\left|Z_{\Theta}(t;x)-Z_{\Theta_{L,M}}(l+1;x)\right|\right\}\\
& \leq C(\mathcal{L}^{\sup})\left(\frac{1}{M}\sum^{L-1}_{l=0}\sum^M_{m=1}\int^{\frac{l+1}{L}}_{\frac{l}{L}}|\theta_{l,m}-\theta_m(\tau)|^2\rd\tau\right)^{1/2}+\frac{C(\mathcal{L}^{\sup})}{L}\,,
\end{aligned}
\end{equation}
and
\begin{equation}\label{eqn:stabilityofpmean-fieldfiniteL}
\begin{aligned}
& \sup_{\frac{l}{L}\leq t\leq\frac{l+1}{L}}\left|p_{\Theta}(t;x)-p_{\Theta_{L,M}}(l;x)\right| \\
& \leq C(\mathcal{L}^{\sup})\left(\frac{1}{M}\sum^{L-1}_{l=0}\sum^M_{m=1}\int^{\frac{l+1}{L}}_{\frac{l}{L}}|\theta_{l,m}-\theta_m(\tau)|^2\rd\tau\right)^{1/2}+\frac{C(\mathcal{L}^{\sup})}{L}\,.
\end{aligned}
\end{equation}
\end{lemma}
\begin{proof} 
Define
\[
Z(t;x)=Z_{\Theta}(t;x),\quad p(t;x)=p_{\Theta}(t;x)\,
\]
and
\[
\widetilde{Z}(t;x)=\sum^{L-1}_{l=0}Z_{{\Theta}_{L,M}}(l;x)\textbf{1}_{\frac{l}{L}\leq t<\frac{l+1}{L}},\quad \widetilde{p}(t;x)=\sum^{L-1}_{l=0}p_{{\Theta}_{L,M}}(l;x)\textbf{1}_{\frac{l}{L}<t\leq \frac{l+1}{L}},
\]
with
\[
\widetilde{Z}(1;x)=Z_{{\Theta}_{L,M}}(L;x),\quad \widetilde{p}(0;x)=p_{{\Theta}_{L,M}}(0;x)\,.
\]
Using \eqref{eqn:disRes}, \eqref{eqn:prhofiniteL}, and Lemma~\ref{prop:wmean-fieldlimitfiniteL}, we obtain 
\begin{equation}\label{intervalbound2}
\begin{aligned}
\left|Z_{\Theta_{L,M}}(l+1;x)-Z_{\Theta_{L,M}}(l;x)\right| &<\frac{C(\mathcal{L}^{\sup})}{L},\\ \left|p_{\Theta_{L,M}}(l+1;x)-p_{\Theta_{L,M}}(l;x)\right| &<\frac{C(\mathcal{L}^{\sup})}{L},
\end{aligned}
\end{equation}
for $0\leq l\leq L-1$.  Now define $\Delta_t$ by
\[
\Delta_t=Z(t;x)-\widetilde{Z}(t;x)\,.
\]
For $t \in \left[ \frac{l}{L}, \frac{l+1}{L}\right]$, we have from \eqref{eqn:contRes} that
\begin{equation}\label{intervalbound}
\begin{aligned}
|\Delta_t|&\leq \left|\Delta_{\frac{l}{L}}\right|+\frac{1}{M}\sum^M_{m=1}\int^{t}_{\frac{l}{L}}|f(Z(\tau;x),\theta_m(\tau))|\rd\tau\\
&\leq \left|\Delta_{\frac{l}{L}}\right|+\frac{C(\mathcal{L}^{\sup})}{M}\sum^M_{m=1}\int^{\frac{l+1}{L}}_{\frac{l}{L}}(|\theta_m(\tau)|+1)\rd\tau\\
&\leq \left|\Delta_{\frac{l}{L}}\right|+\frac{C(\mathcal{L}^{\sup})}{L}\,,
\end{aligned}
\end{equation}
where we use \eqref{eqn:boundoff} and \eqref{boundofxsolutiondis} in the second inequality and \eqref{eqn:Lmean-fieldfiniteL} in the third inequality. 
From \eqref{eqn:disRes} and \eqref{eqn:contRes}, we obtain further that
\[
\begin{aligned}
\left|\Delta_{\frac{l+1}{L}}\right|& =\left|\Delta_{\frac{l}{L}}\right|+\left|\frac{1}{M}\sum^M_{m=1}\int^{\frac{l+1}{L}}_{\frac{l}{L}} f(Z(\tau;x),\theta_m(\tau))-f\left(\widetilde{Z}(\tau;x),\theta_{l,m}\right)\rd\tau\right|\\
& \leq \left|\Delta_{\frac{l}{L}}\right|+\left|\frac{1}{M}\sum^M_{m=1}\int^{\frac{l+1}{L}}_{\frac{l}{L}} f(Z(\tau;x),\theta_m(\tau))-f\left(Z(\tau;x),\theta_{l,m}\right)\rd\tau\right|\\
& \quad +\left|\frac{1}{M}\sum^M_{m=1}\int^{\frac{l+1}{L}}_{\frac{l}{L}} f(Z(\tau;x),\theta_{l,m})-f\left(\widetilde{Z}(\tau;x),\theta_{l,m}\right)\rd\tau\right|\\
& \stackrel{{\textrm{(I)}}}{\leq} \left|\Delta_{\frac{l}{L}}\right|+C(\mathcal{L}^{\sup})\left(\frac{1}{M}\sum^M_{m=1}\int^{\frac{l+1}{L}}_{\frac{l}{L}}|\theta_m(\tau)-\theta_{l,m}|\rd\tau\right)+C_1|\Delta_\xi|\left(\frac{1}{ML}\sum^M_{m=1}|\theta_{l,m}|\right)\\
& \stackrel{{\textrm{(II)}}}{\leq} \left(1+\frac{C(\mathcal{L}^{\sup})}{L}\right)\left|\Delta_{\frac{l}{L}}\right|+\frac{C(\mathcal{L}^{\sup})}{M}\sum^M_{m=1}\int^{\frac{l+1}{L}}_{\frac{l}{L}}|\theta_m(\tau)-\theta_{l,m}|\rd\tau+\frac{C(\mathcal{L}^{\sup})}{L^2}
\end{aligned}
\]
where $\xi\in[\frac{l}{L},\frac{l+1}{L}]$, and using \eqref{eqn:derivativebound},  \eqref{boundofxsolutiondis} in (I), and \eqref{eqn:Lmean-fieldfiniteL} and \eqref{intervalbound} in (II). 
By applying this bound iteratively, we obtain
\[
\left|\Delta_{\frac{l}{L}}\right|\leq C(\mathcal{L}^{\sup})\left|\Delta_{0}\right|+\frac{C(\mathcal{L}^{\sup})}{M}\sum^{l-1}_{j=0}\sum^M_{m=1}\int^{\frac{j+1}{L}}_{\frac{j}{L}}|\theta_{l,m}-\theta_m(\tau)|\rd\tau+\frac{C(\mathcal{L}^{\sup})}{L}\,,
\]
where $|\Delta_{0}|=0$. Combining this with \eqref{intervalbound} and using H\"older's inequality, we obtain that
\[
\left|\Delta_t\right|\leq C(\mathcal{L}^{\sup})\left(\frac{1}{M}\sum^{L-1}_{l=0}\sum^M_{m=1}\int^{\frac{l+1}{L}}_{\frac{l}{L}}|\theta_{l,m}-\theta_m(\tau)|^2\rd\tau\right)^{1/2}+\frac{C(\mathcal{L}^{\sup})}{L}\,.
\]
By combining 
this bound with \eqref{intervalbound2},  we prove \eqref{eqn:stabilityofZmean-fieldfiniteL}.

To prove \eqref{eqn:stabilityofpmean-fieldfiniteL}, we define
\[
\Delta_p(t;x)=p(t;x)-\widetilde{p}(t;x)\,.
\]
Similarly to \eqref{eqn:ODEforDEltaprhoinitial}, we obtain 
\begin{equation}\label{eqn:ODEforDEltaprhoinitial2}
\begin{aligned}
|\Delta_p(1;x)| & \leq C(\mathcal{L}^{\sup})\left|\widetilde{Z}(1;x)-Z(1;x)\right| \\
& \leq C(\mathcal{L}^{\sup})\left(\frac{1}{M}\sum^{L-1}_{l=0}\sum^M_{m=1}\int^{\frac{l+1}{L}}_{\frac{l}{L}}|\theta_{l,m}-\theta_m(\tau)|^2\rd\tau\right)^{1/2}+\frac{C(\mathcal{L}^{\sup})}{L}.
\end{aligned}
\end{equation}
For $t \in \left( \frac{l}{L}\frac{l+1}{L}\right]$, using \eqref{eqn:prho}, we obtain that
\begin{equation}\label{intervalbound3}
\begin{aligned}
|\Delta_p\left(t;x\right)|&\leq \left|\Delta_p\left(\frac{l+1}{L};x\right)\right|+\frac{1}{M}\sum^M_{m=1}\int^{\frac{l+1}{L}}_t \left|\partial_zf(Z(\tau;x),\theta_m(\tau))\right||p(\tau;x)|\rd\tau\\
&\leq \left|\Delta_p\left(\frac{l+1}{L};x\right)\right|+\frac{C(\mathcal{L}^{\sup})}{M}\sum^M_{m=1}\int^{\frac{l+1}{L}}_{\frac{l}{L}}|\theta_m(\tau)|\rd\tau\\
&\leq \left|\Delta_p\left(\frac{l+1}{L};x\right)\right|+\frac{C(\mathcal{L}^{\sup})}{L}\,,
\end{aligned}
\end{equation}
where we use \eqref{eqn:derivativebound} and \eqref{boundofprhosolutiondis} in the second inequality, and H\"older's inequality togerther with the definition of $\mathcal{L}$ from~\eqref{eqn:Lmean-fieldfiniteL} in the third inequality.

From \eqref{eqn:prho}, we obtain that
\[
p^\top\left(\frac{l}{L};x\right)=p^\top\left(\frac{l+1}{L};x\right)+\frac{1}{M}\sum^M_{m=1}\int^{\frac{l+1}{L}}_\frac{l}{L}p^\top\left(\tau;x\right)\partial_zf(Z(\tau;x),\theta_m(\tau))\rd \tau\,,
\]
and from \eqref{eqn:prhofiniteL} that
\[
\widetilde{p}^\top\left(\frac{l}{L};x\right)=\widetilde{p}^\top\left(\frac{l+1}{L};x\right)+\frac{1}{M}\sum^M_{m=1}\int^{\frac{l+1}{L}}_\frac{l}{L}\widetilde{p}^\top\left(\frac{l+1}{L};x\right)\partial_zf(Z_{\Theta_{L,M}}(l;x),\theta_{l,m})\rd \tau\,.
\]
By bounding differences of these two expressions, we have
\begin{equation}\label{eqn:all}
\begin{aligned}
&\left|\Delta_p\left(\frac{l}{L};x\right)\right|\\
& \leq \left|\Delta_p\left(\frac{l+1}{L};x\right)\right|+\underbrace{\frac{1}{M}\sum^M_{m=1}\int^{\frac{l+1}{L}}_\frac{l}{L}\left|p^\top\left(\tau;x\right)\partial_zf(Z(\tau;x),\theta_m(\tau))\rd\tau-\widetilde{p}^\top\left(\frac{l+1}{L};x\right)\partial_zf(Z(\tau;x),\theta_m(\tau))\right|\rd \tau}_{{\textrm{(I)}}}\\
& \quad +\underbrace{\frac{1}{M}\sum^M_{m=1}\int^{\frac{l+1}{L}}_\frac{l}{L}\left|\widetilde{p}^\top\left(\frac{l+1}{L};x\right)\partial_zf(Z(\tau;x),\theta_m(\tau))-\widetilde{p}^\top\left(\frac{l+1}{L};x\right)\partial_zf(Z_{\Theta_{L,M}}(l;x),\theta_m(\tau))\right|\rd\tau}_{{\textrm{(II)}}}\\
& \quad +\underbrace{\frac{1}{M}\sum^M_{m=1}\int^{\frac{l+1}{L}}_\frac{l}{L}\left|\widetilde{p}^\top\left(\frac{l+1}{L};x\right)\partial_zf(Z_{\Theta_{L,M}}(l;x),\theta_m(\tau))-\widetilde{p}^\top\left(\frac{l+1}{L};x\right)\partial_zf(Z_{\Theta_{L,M}}(l;x),\theta_{l,m})\right|\rd \tau}_{{\textrm{(III)}}}\,.
\end{aligned}
\end{equation}
We bound (I), (II), and (III) as follows.
\begin{enumerate}[wide,   labelindent=0pt]
\item[(I):] Using \eqref{eqn:derivativebound} and \eqref{eqn:Lmean-fieldfiniteL}, we obtain that
\[
{\textrm{(I)}}\leq |\Delta_p(t;x)|\frac{C(\mathcal{L}^{\sup})}{M}\sum^M_{m=1}\int^{\frac{l+1}{L}}_{\frac{l}{L}}|\theta_m(\tau)|\rd \tau\leq \frac{C(\mathcal{L}^{\sup})}{L}|\Delta_p(t;x)|.
\]
\item[(II):] Using \eqref{eqn:derivativebound}, \eqref{boundofprhofiniteL}, and \eqref{eqn:Lmean-fieldfiniteL}, we obtain
\[
\begin{aligned}
{\textrm{(II)}}&\leq \frac{C(\mathcal{L}^{\sup})}{M}\sum^M_{m=1}\int^{\frac{l+1}{L}}_{\frac{l}{L}}|Z(\tau;x)-Z_{\Theta_{L,M}}(l;x)||\theta_m(\tau)|^2\rd \tau\\
&\leq \frac{C(\mathcal{L}^{\sup})}{L}\left(\left(\frac{1}{M}\sum^{L-1}_{l'=0}\sum^M_{m=1}\int^{\frac{l'+1}{L}}_{\frac{l'}{L}}|\theta_{l',m}-\theta_m(\tau)|^2\rd \tau\right)^{1/2}+\frac{1}{L}\right)\,,
\end{aligned}
\]
where we make use of \eqref{eqn:stabilityofZmean-fieldfiniteL} in the final inequality.
\item[{\textrm{(III)}}:]  Using \eqref{eqn:derivativebound}, \eqref{boundofxsolutionfiniteL}, \eqref{boundofprhofiniteL}, and \eqref{eqn:Lmean-fieldfiniteL}, we obtain that
\[
\begin{aligned}
{\textrm{(III)}}&\leq \frac{C(\mathcal{L}^{\sup})}{M}\sum^M_{m=1}\int^{\frac{l+1}{L}}_\frac{l}{L}\left(|\theta_m(\tau)|+|\theta_{l,m}|\right)|\theta_m(\tau)-\theta_{l,m}|\rd\tau\\
&\leq C(\mathcal{L}^{\sup})\int^{\frac{l+1}{L}}_\frac{l}{L}\left(\frac{1}{M}\sum^M_{m=1}\left(|\theta_m(\tau)|+|\theta_{l,m}|\right)^2\right)^{1/2}\left(\frac{1}{M}\sum^M_{m=1}|\theta_m(\tau)-\theta_{l,m}|^2\right)^{1/2}\rd\tau\\
&\leq C(\mathcal{L}^{\sup})\int^{\frac{l+1}{L}}_\frac{l}{L}\left(\frac{1}{M}\sum^M_{m=1}|\theta_m(\tau)-\theta_{l,m}|^2\right)^{1/2}\rd\tau\,.
\end{aligned}
\]
\end{enumerate}
By substituting these three inequalities and \eqref{intervalbound3} into \eqref{eqn:all}, we obtain
\[
\begin{aligned}
\left|\Delta_p\left(\frac{l}{L};x\right)\right|
& \leq \left(1+\frac{C(\mathcal{L}^{\sup})}{L}\right)\left|\Delta_p\left(\frac{l+1}{L};x\right)\right|\\
& \quad +\frac{C(\mathcal{L}^{\sup})}{L}\left(\left(\frac{1}{M}\sum^{L-1}_{l'=0}\sum^M_{m=1}\int^{\frac{l+1}{L}}_{\frac{l}{L}}|\theta_{l',m}-\theta_m(\tau)|^2\rd\tau\right)^{1/2}+\frac{1}{L}\right)\\
& \quad +C(\mathcal{L}^{\sup})\int^{\frac{l+1}{L}}_\frac{l}{L}\left(\frac{1}{M}\sum^M_{m=1}|\theta_m(\tau)-\theta_{l,m}|^2\right)^{1/2}\rd\tau\,.
\end{aligned}
\]
By applying this bound iteratively, and using \eqref{eqn:ODEforDEltaprhoinitial2} and \eqref{intervalbound3}, we obtain
\begin{equation}\label{eqn:stabilityofprhomean-fieldfiniteLfinitet}
\begin{aligned}
\left|\Delta_p(t;x)\right|\leq C(\mathcal{L}^{\sup})\left(\left(\frac{1}{M}\sum^{L-1}_{l=0}\sum^M_{m=1}\int^{\frac{l+1}{L}}_{\frac{l}{L}}|\theta_{l,m}-\theta_m(\tau)|^2\rd\tau\right)^{1/2}+\frac{1}{L}\right)\,,
\end{aligned}
\end{equation}
where we also use H\"older's inequality to write
\begin{align*}
& C(\mathcal{L}^{\sup})\sum^{L-1}_{l=0}\int^{\frac{l+1}{L}}_\frac{l}{L}\left(\frac{1}{M}\sum^M_{m=1}|\theta_m(\tau)-\theta_{l,m}|^2\right)^{1/2}\rd\tau \\
& \quad \leq C(\mathcal{L}^{\sup})\left(\frac{1}{M}\sum^{L-1}_{l=0}\sum^M_{m=1}\int^{\frac{l+1}{L}}_{\frac{l}{L}}|\theta_{l,m}-\theta_m(\tau)|^2\rd\tau\right)^{1/2}.
\end{align*}
We obtain \eqref{eqn:stabilityofpmean-fieldfiniteL} by combining \eqref{eqn:stabilityofprhomean-fieldfiniteLfinitet} with \eqref{intervalbound3}.
\end{proof}


\subsection{Proof of Theorem~\ref{thm:contlimit}}

We denote by $\theta_m(s;t)$ the solution to \eqref{eqn:Wassgradientflowsdis} with initial $\{\theta_m(0;t)\}^M_{m=1}$  $i.i.d.$ drawn from $\rho_{\ini}(\theta,t)$\,. 
Further, $\theta_{l,m}(s)$ is a solution to \eqref{eqn:classicalgradientflowfiniteL} with initial $\theta_{l,m}(0)=\theta_m\left(0;\frac{l}{L}\right)$ for $0\leq l\leq L-1$ and $1\leq i\leq M$. Denoting
\[
\mathcal{L}^{\dis,\sup}_\ini=\sup_{t\in[0,1]}\frac{1}{M}\sum^M_{m=1}|\theta_m(0;t)|^2\,,\quad \Theta_{L,M}(s)=\left\{\theta_{l,m}(s)\right\}^{L-1,M}_{l=0,i=0},
\]
we have the following lemma.
\begin{lemma} For fixed $S>0$, any $s \in [0,S]$, $x$ in the support of $\mu$, and integer $l$ with $0\leq l\leq L-1$, there exists a constant $C\left(S,\mathcal{L}^{\dis,\sup}_\ini\right)$ depending only on $S$ and $\mathcal{L}^{\dis,\sup}_\ini$ such that 
\begin{equation}\label{boundofthetasolutiondisfiniteL}
\sup_{s\in[0,S],l}\frac{1}{M}\sum^M_{m=1}\left|\theta_{l,m}(s)\right|^2<C\left(S,\mathcal{L}^{\dis,\sup}_\ini\right)\,.
\end{equation}
Further, the ODE solution and $p_{\Theta_{L,M}(s)}$ are bounded as follows:
\begin{equation}\label{boundofxsolutiondisfiniteL}
\left|Z_{\Theta_{L,M}(s)}(l+1;x)\right|\leq C\left(S,\mathcal{L}^{\dis,\sup}_\ini\right)\,,
\end{equation}
\begin{equation}\label{boundofprhosolutiondisfiniteL}
\left|p_{\Theta_{L,M}(s)}(l;x)\right|\leq C\left(S,\mathcal{L}^{\dis,\sup}_\ini\right)\,.
\end{equation}
\end{lemma}
\begin{proof} Since the proof is quite similar to Lemma~\ref{lem:boundofdiscrete}, we omit it.
\end{proof}

We are now ready to prove Theorem~\ref{thm:contlimit}.
\begin{proof}[Proof of Theorem~\ref{thm:contlimit}] Define
\[
\mathcal{L}^{\sup}_{0}=\max\left\{\sup_{t\in[0,1]}\int_{\mathbb{R}^k}|\theta|^2\rd\rho_{\ini}(\theta,t),\mathcal{L}^{\dis,\sup}_\ini\right\}<\infty\,.
\]
From \eqref{eqn:boundofLSdis} and \eqref{boundofthetasolutiondisfiniteL}, we obtain 
\begin{equation}\label{boundofLsupappendix}
\sup_{(t,s)\in[0,1]\times[0,S],l}\left\{\frac{1}{M}\sum^M_{m=1}|\theta_m(s;t)|^2,\frac{1}{M}\sum^M_{m=1}\left|\theta_{l,m}(s)\right|^2\right\}\leq C(\mathcal{L}^{\sup}_{0},S)\,,
\end{equation}
where $C(\mathcal{L}^{\sup}_{0},S)$ is a constant depends on $\mathcal{L}^{\sup}_{0},S$.

Using a similar derivation to \eqref{eqn:differenceofE}, we have 
\begin{equation}\label{eqn:differenceofE2}
\left|E(\Theta(s;\cdot))-E(\Theta_{L,M}(s))\right|\leq C(\mathcal{L}^{\sup}_{0},S)\left|Z_{\Theta(s)}(1;x)-Z_{\Theta_{L,M}(s)}(L;x)\right|\,.
\end{equation}
Thus, to prove the theorem, it suffices to prove that $\left|Z_{\Theta_{L,M}(s)}(L;x)-Z_{\Theta(s)}(1;x)\right|$ is small. 
According to \eqref{eqn:stabilityofZmean-fieldfiniteL}, this requires us to bound the quantity 
\begin{equation}  \label{eq:fg2}
\frac{1}{M}\sum^M_{m=1}\int^{\frac{l+1}{L}}_{\frac{l}{L}}|\Delta_{t,m}(s)|^2\rd t\,, \quad
\mbox{where} \;\; \Delta_{t,m}(s)=\theta_{l,m}(s)-\theta_m(s;t)\,.
\end{equation}
The next part of the proof obtains the required bound.



First, using \eqref{eqn:classicalgradientflowfiniteL} and \eqref{eqn:Wassgradientflowsdis}, we obtain that
\begin{equation}\label{iterationmean-fieldfiniteL}
\begin{aligned}
&\frac{\rd|\Delta_{t,m}(s)|^2}{\rd s}\\
& = -2\exp(-s)|\Delta_{t,m}(s)|^2\\
& \quad -2\left\langle \Delta_{t,m}(s),\mathbb{E}_{x\sim\mu}\left(\partial_\theta f(Z_{\Theta_{L,M}(s)}(l;x),\theta_{l,m})p_{\Theta_{L,M}(s)}(l;x)-\partial_\theta f(Z_{\Theta(s)}(t;x),\theta_m)p_{\Theta(s)}(t;x)\right)\right\rangle\\
&= -2\exp(-s)|\Delta_{t,m}(s)|^2\\
& \quad -2\left\langle \Delta_{t,m}(s),\underbrace{\mathbb{E}_{x\sim\mu}\left(\partial_\theta f(Z_{\Theta_{L,M}(s)}(l;x),\theta_{l,m})p_{\Theta_{L,M}(s)}(l;x)-\partial_\theta f(Z_{\Theta(s)}(t;x),\theta_m)p_{\Theta_{L,M}(s)}(l;x)\right)}_{{\textrm{(I)}}}\right\rangle\\
& \quad -2\left\langle \Delta_{t,m}(s),\underbrace{\mathbb{E}_{x\sim\mu}\left(\partial_\theta f(Z_{\Theta(s)}(t;x),\theta_m)p_{\Theta_{L,M}(s)}(l;x)-\partial_\theta f(Z_{\Theta(s)}(t;x),\theta_m)p_{\Theta(s)}(t;x)\right)}_{{\textrm{(II)}}}\right\rangle
\end{aligned}
\end{equation}
To bound (I), we use \eqref{boundofprhosolutiondisfiniteL} to obtain
\begin{equation}\label{boundofIImean-fieldfiniteL}
\begin{aligned}
|{\textrm{(I)}}| & \leq C(\mathcal{L}^{\sup}_{0},S)\EE_{x\sim\mu}\left(\left|\partial_\theta f(Z_{\Theta_{L,M}(s)}(l;x),\theta_{l,m})-\partial_\theta f(Z_{\Theta(s)}(t;x),\theta_m)\right|\right)\\
& \leq C(\mathcal{L}^{\sup}_{0},S)\left[\left(\left|\theta_{l,m}\right|+\left|\theta_m\right|\right)\EE_{x\sim\mu}\left(\left|Z_{\Theta_{L,M}(s)}(l;x)-Z_{\Theta(s)}(t;x)\right|\right)+|\theta_{l,m}-\theta_m|\right]\,,
\end{aligned}
\end{equation}
where we use \eqref{eqn:derivativebound}, \eqref{boundofxsolutiondis}, and \eqref{boundofxsolutiondisfiniteL} in the second inequality.
To bound (II), we use  \eqref{eqn:derivativebound} and  \eqref{boundofxsolutiondis} to obtain
\begin{equation}\label{boundofImean-fieldfiniteL}
\begin{aligned}
|{\textrm{(II)}}|\leq &C(\mathcal{L}^{\sup}_{0},S)\EE_{x\sim\mu}\left(\left|p_{\Theta_{L,M}(s)}(l;x)-p_{\Theta(s)}(t;x)\right|\right)\,.
\end{aligned}
\end{equation}

By substituting \eqref{boundofImean-fieldfiniteL} and \eqref{boundofIImean-fieldfiniteL} into \eqref{iterationmean-fieldfiniteL}, we obtain 
\[
\begin{aligned}
&\frac{\rd|\Delta_{t,m}(s)|^2}{\rd s}\\
& \leq C(\mathcal{L}^{\sup}_{0},S)|\Delta_{t,m}(s)|^2\\
& \quad +C(\mathcal{L}^{\sup}_{0},S)|\Delta_{t,m}(s)|\EE_{x\sim\mu}\left(\left(\left|\theta_{l,m}\right|+\left|\theta_m\right|\right)\left|Z_{\Theta_{L,M}(s)}(l;x)-Z_{\Theta(s)}(t;x)\right|+\left|p_{\Theta_{L,M}(s)}(l;x)-p_{\Theta(s)}(t;x)\right|\right)\,,
\end{aligned}
\]
which implies that
\begin{equation}\label{eqn:finiteML}
\begin{aligned}
&\frac{\rd \left(\frac{1}{M}\sum^M_{m=1}|\Delta_{t,m}(s)|^2\right)}{\rd s}\\
& \leq C(\mathcal{L}^{\sup}_{0},S)\left(\frac{1}{M}\sum^M_{m=1}|\Delta_{t,m}(s)|^2\right)\\
&\quad +C(\mathcal{L}^{\sup}_{0},S)\left(\frac{1}{M}\sum^M_{m=1}|\Delta_{t,m}(s)|\left(\left|\theta_{l,m}\right|+\left|\theta_m\right|\right)\right)\EE_{x\sim\mu}\left(\left|Z_{\Theta_{L,M}(s)}(l;x)-Z_{\Theta(s)}(t;x)\right|\right)\\
& \quad +C(\mathcal{L}^{\sup}_{0},S)\left(\frac{1}{M}\sum^M_{m=1}|\Delta_{t,m}(s)|\right)\EE_{x\sim\mu}\left(\left|p_{\Theta_{L,M}(s)}(l;x)-p_{\Theta(s)}(t;x)\right|\right)\\
& \leq C(\mathcal{L}^{\sup}_{0},S)\left(\frac{1}{M}\sum^M_{m=1}|\Delta_{t,m}(s)|^2\right)+C(\mathcal{L}^{\sup}_{0},S)\EE_{x\sim\mu}\left(\left|p_{\Theta_{L,M}(s)}(t;x)-p_{\Theta(s)}(t;x)\right|^2\right)\\
& \quad +C(\mathcal{L}^{\sup}_{0},S)\EE_{x\sim\mu}\left(\left|Z_{\Theta_{L,M}(s)}(t;x)-Z_{\Theta(s)}(t;x)\right|^2\right)\,.
\end{aligned}
\end{equation}
The second inequality above uses H\"older's inequality together with the bound $2ab \le a^2+b^2$ and 
\begin{align*}
\left(\frac{1}{M}\sum^M_{m=1}|\Delta_{t,m}(s)|\left(\left|\theta_{l,m}\right|+\left|\theta_m\right|\right)\right)^2&\leq  \left(\frac{1}{M}\sum^M_{m=1}|\Delta_{t,m}(s)|^2\right)\left(\frac{1}{M}\sum^M_{m=1}\left(\left|\theta_{l,m}\right|+\left|\theta_m\right|\right)^2\right)\\
&\leq C(\mathcal{L}^{\sup}_{0},S)\left(\frac{1}{M}\sum^M_{m=1}|\Delta_{t,m}(s)|^2\right)\,.
\end{align*}

By substituting \eqref{eqn:stabilityofZmean-fieldfiniteL} and \eqref{eqn:stabilityofpmean-fieldfiniteL} into \eqref{eqn:finiteML}, we obtain
\[
\frac{\rd\frac{1}{M}\sum^{L-1}_{l=0}\sum^M_{m=1}\int^{\frac{l+1}{L}}_{\frac{l}{L}}|\Delta_{t,m}(s)|^2\rd t}{\rd s}\leq C(\mathcal{L}^{\sup}_{0},S)\left(\frac{1}{M}\sum^{L-1}_{l=0}\sum^M_{m=1}\int^{\frac{l+1}{L}}_{\frac{l}{L}}|\Delta_{t,m}(s)|^2\rd t+\frac{1}{L^2}\right)\,,
\]
which implies that
\begin{equation}\label{eqn:finiteLfinal}
\frac{1}{M}\sum^{L-1}_{l=0}\sum^M_{m=1}\int^{\frac{l+1}{L}}_{\frac{l}{L}}|\Delta_{t,m}(s)|^2\rd t\leq C(\mathcal{L}^{\sup}_{0},S)\left(\frac{1}{M}\sum^{L-1}_{l=0}\sum^M_{m=1}\int^{\frac{l+1}{L}}_{\frac{l}{L}}|\Delta_{t,m}(0)|^2\rd t+\frac{1}{L^2}\right)
\end{equation}
by Gr\"onwall's inequality. This is the bound we were seeking on  \eqref{eq:fg2}. 
Here, we also have
\begin{equation}\label{differenceinitialfiniteL}
\frac{1}{M}\sum^{L-1}_{l=0}\sum^M_{m=1}\int^{\frac{l+1}{L}}_{\frac{l}{L}}|\Delta_{t,m}(0)|^2\rd t=\frac{1}{M}\sum^{L-1}_{l=0}\sum^M_{m=1}\int^{\frac{l+1}{L}}_{\frac{l}{L}}\left|\theta_m\left(0;\frac{l}{L}\right)-\theta_m(0;t)\right|^2\rd t\,.
\end{equation}

To complete the proof, we first use \eqref{boundLdis0} and \eqref{diffrencesmall}, and take  $M\geq \frac{2C_3}{\eta}$, 
to obtain
\begin{equation}\label{eqn:highprobbound}
\mathbb{P}\left(\frac{1}{M}\sum^{L-1}_{l=0}\sum^M_{m=1}\int^{\frac{l+1}{L}}_{\frac{l}{L}}|\Delta_{t,m}(0)|^2\rd t\leq \frac{C_4}{L^2}\right)\geq 1-\eta/2\,.
\end{equation}
Similarly to \eqref{boundLdis01}, we also have
\begin{equation}\label{eqn:highprobbound2}
\mathbb{P}\left(\mathcal{L}^{\sup}_{0}\leq C_4\right)\geq 1-\eta/2\,.
\end{equation}
Using \eqref{eqn:highprobbound}, \eqref{eqn:highprobbound2} to substitute $\frac{1}{M}\sum^{L-1}_{l=0}\sum^M_{m=1}\int^{\frac{l+1}{L}}_{\frac{l}{L}}|\Delta_{t,m}(0)|^2$, $\mathcal{L}^{\sup}_{0}$ in \eqref{eqn:finiteLfinal}, we find that there exists a constant $C'(C_4,S)$ depending on $C_4$ and $S$ such that if 
\[
M\geq \frac{2C_3}{\eta},\quad L\geq \frac{C'(C_4,S)}{\epsilon},
\]
then we have
\begin{equation}\label{eqn:highprobbound3}
\mathbb{P}\left(\frac{1}{M}\sum^{L-1}_{l=0}\sum^M_{m=1}\int^{\frac{l+1}{L}}_{\frac{l}{L}}|\Delta_{t,m}(s)|^2\rd t\leq \epsilon\right)\geq 1-\eta\,.
\end{equation}
Using \eqref{eqn:highprobbound3}, \eqref{boundofLsupappendix} and \eqref{eqn:highprobbound2} to bound the right hand side of \eqref{eqn:stabilityofZmean-fieldfiniteL}, we find that there exists another constant $C''(C_4,S)$ depending on $C_4$ and $S$ such that if 
\[
M\geq \frac{2C_3}{\eta},\quad L\geq \frac{C''(C_4,S)}{\epsilon},
\]
then we have
\[
\mathbb{P}\left(\left|Z_{\Theta(s)}(1;x)-Z_{\Theta_{L,M}(s)}(L;x)\right|\leq \epsilon\right)\geq 1-\eta\,,
\]
By using this result in conjunction with \eqref{eqn:differenceofE2}, we complete the proof.
\end{proof}

\section{Rigorous Proof of Lemma~\ref{lem:decayofcost}}\label{proofoflemdecayofcost}
\begin{proof}[Proof of Lemma~\ref{lem:decayofcost}]
According to the brief proof in Appendix \ref{sec:Step3}, it suffices to prove
\begin{equation}\label{eqn:derivative}
\frac{\rd E_{s}(\phi^\ast(s))}{\rd s}=-\int^1_0\mathbb{E}\left(\left|\nabla_{\theta}\frac{\delta E_s(\phi^\ast(s))}{\delta \rho}\left(\theta^\ast(s;t),t\right)\right|^2\right)\rd t-\exp^{-s}\int^1_0\mathbb{E}(|\theta^\ast(s;t)|^2)\rd t\,.
\end{equation}

To show~\eqref{eqn:derivative}, we essentially need to show:
\begin{equation}\label{eqn:derivativeproof}
\left|E_s(\phi^\ast(s))-E_{s_0}(\phi^\ast({s_0}))+D(s_0)(s-s_0)\right|=o(|s-s_0|)\,,
\end{equation}
where 
\begin{equation}\label{eqn:Ds_0}
D(s_0)=\int^1_0\mathbb{E}\left(\left|\nabla_{\theta}\frac{\delta E_s(\phi^\ast(s_0))}{\delta \rho}\left(\theta^\ast(s_0;t),t\right)\right|^2\right)\rd t+\exp^{-s_0}\int^1_0\mathbb{E}(|\theta^\ast(s_0;t)|^2)\rd t
\end{equation}
Noticing that according to \eqref{bound}, 
\[
\mathbb{E}\left(\left|\nabla_{\theta}\frac{\delta E_s(\phi^\ast(s))}{\delta \rho}\left(\theta^\ast(s;t),t\right)\right|^2\right)\leq 8\mathbb{E}|\theta^\ast(s;t)|^2+2C(\mathcal{L}_{S,\phi^\ast})<\infty\,,
\]
the term~\eqref{eqn:Ds_0} is finite and well-defined. To prove \eqref{eqn:derivativeproof}, recall \eqref{eqn:cost_s}, we first write
\[
\begin{aligned}
&E_s(\phi^\ast(s))-E_{s_0}(\phi^\ast({s_0}))\\
=&\underbrace{\mathbb{E}_{x\sim \mu}\left(\frac{1}{2}\left(\left\langle w,Z_{\phi^\ast(s)}(1;x)\right\rangle-y(x)\right)^2-\frac{1}{2}\left(\left\langle w,Z_{\phi^\ast(s_0)}(1;x)\right\rangle-y(x)\right)^2\right)}_{{\textrm{(I)}}}\\
&+\underbrace{\exp^{-s}\int^1_0\mathbb{E}(|\theta^\ast(s;t)|^2)\rd t-\exp^{-s_0}\int^1_0\mathbb{E}(|\theta^\ast(s_0;t)|^2)\rd t}_{{\textrm{(II)}}}\,.
\end{aligned}
\]
Then we consider (I) and (II) seperately:
\begin{enumerate}[wide,   labelindent=0pt]
\item[(I):] 

First, (I) can be written as
\begin{align*}
{\textrm{(I)}}=&\mathbb{E}_{x\sim \mu}\left(\left(\left\langle w,Z_{\phi^\ast(s_0)}(1;x)\right\rangle-y(x)\right)\left\langle w,Z_{\phi^\ast(s)}(1;x)-Z_{\phi^\ast(s_0)}(1;x)\right\rangle\right)\\
&+o(|Z_{\phi^\ast(s)}(1;x)-Z_{\phi^\ast(s_0)}(1;x)|)\\
=&\mathbb{E}_{x\sim \mu}\left(\left(\left\langle w,Z_{\phi^\ast(s_0)}(1;x)\right\rangle-y(x)\right)\left\langle w,Z_{\phi^\ast(s)}(1;x)-Z_{\phi^\ast(s_0)}(1;x)\right\rangle\right)\\
&+o(s-s_0)
\end{align*}
where we use $|Z_{\phi^\ast(s)}(1;x)-Z_{\phi^\ast(s_0)}(1;x)|\leq C(\mathcal{L}_{S,\phi^\ast})|s-s_0|$ according to \eqref{stabilityofxsolution} and $\eqref{stabilityofODEsolutiontheta}$. 

Similar to \citep{pmlr-v119-lu20b}, using \eqref{eqn:derivativebound}, \eqref{stabilityofxsolution}, and $\eqref{stabilityofODEsolutiontheta}$, we have
\begin{equation}\label{Zdifferencebig}
\begin{aligned}
&\frac{\rd(Z_{\phi^\ast(s)}(t;x)-Z_{\phi^\ast(s_0)}(t;x))}{\rd t}\\
=&\mathbb{E}\left(f(Z_{\phi^\ast(s)}(t;x),\theta^\ast(s;t))-f(Z_{\phi^\ast(s_0)}(t;x),\theta^\ast(s_0;t))\right)\\
=&\mathbb{E}\left(f(Z_{\phi^\ast(s)}(t;x),\theta^\ast(s;t))-f(Z_{\phi^\ast(s_0)}(t;x),\theta^\ast(s;t))\right)\\
&+\mathbb{E}\left(f(Z_{\phi^\ast(s_0)}(t;x),\theta^\ast(s;t))-f(Z_{\phi^\ast(s_0)}(t;x),\theta^\ast(s_0;t))\right)\\
=&\left\{\mathbb{E}\left(f(Z_{\phi^\ast(s)}(t;x),\theta^\ast(s;t))-f(Z_{\phi^\ast(s_0)}(t;x),\theta^\ast(s;t))\right)\right.\\
&\left.-\mathbb{E}\left(f(Z_{\phi^\ast(s)}(t;x),\theta^\ast(s_0;t))-f(Z_{\phi^\ast(s_0)}(t;x),\theta^\ast(s_0;t))\right)\right\}\\
&+\mathbb{E}\left(f(Z_{\phi^\ast(s)}(t;x),\theta^\ast(s_0;t))-f(Z_{\phi^\ast(s_0)}(t;x),\theta^\ast(s_0;t))\right)\\
&+\mathbb{E}\left(f(Z_{\phi^\ast(s_0)}(t;x),\theta^\ast(s;t))-f(Z_{\phi^\ast(s_0)}(t;x),\theta^\ast(s_0;t))\right)\\
\stackrel{(a)}{=}&\left\{\left[\mathbb{E}\left(\partial_zf(Z_{\phi^\ast(s_0)}(t;x),\theta^\ast(s;t))\right)\right](Z_{\phi^\ast(s)}(t;x)-Z_{\phi^\ast(s_0)}(t;x))\right.\\
&-\left.\left[\mathbb{E}\left(\partial_zf(Z_{\phi^\ast(s_0)}(t;x),\theta^\ast(s_0;t))\right)\right](Z_{\phi^\ast(s)}(t;x)-Z_{\phi^\ast(s_0)}(t;x))\right\}\\
&+\left[\mathbb{E}\left(\partial_zf(Z_{\phi^\ast(s_0)}(t;x),\theta^\ast(s_0;t))\right)\right](Z_{\phi^\ast(s)}(t;x)-Z_{\phi^\ast(s_0)}(t;x))\\
&+\mathbb{E}\left(f(Z_{\phi^\ast(s_0)}(t;x),\theta^\ast(s;t))-f(Z_{\phi^\ast(s_0)}(t;x),\theta^\ast(s_0;t))\right)\\
&+o(|s-s_0|)\\
\stackrel{(b)}{=}&\left[\mathbb{E}\left(\partial_zf(Z_{\phi^\ast(s_0)}(t;x),\theta^\ast(s_0;t))\right)\right](Z_{\phi^\ast(s)}(t;x)-Z_{\phi^\ast(s_0)}(t;x))\\
&+\mathbb{E}\left(f(Z_{\phi^\ast(s_0)}(t;x),\theta^\ast(s;t))-f(Z_{\phi^\ast(s_0)}(t;x),\theta^\ast(s_0;t))\right)\\
&+o(|s-s_0|)\\
\stackrel{(c)}{=}&\mathbb{E}\left(\partial_zf(Z_{\phi^\ast(s_0)}(t;x),\theta^\ast(s_0;t))\right)(Z_{\phi^\ast(s)}(t;x)-Z_{\phi^\ast(s_0)}(t;x))\\
&+\mathbb{E}\left(\partial_\theta f(Z_{\phi^\ast(s_0)}(t;x),\theta^\ast(s_0;t))(\theta^\ast(s;t)-\theta^\ast(s_0;t))\right)\\
&+o(|s-s_0|)\\
\stackrel{(d)}{=}&\mathbb{E}\left(\partial_zf(Z_{\phi^\ast(s_0)}(t;x),\theta^\ast(s_0;t))\right)(Z_{\phi^\ast(s)}(t;x)-Z_{\phi^\ast(s_0)}(t;x))\\
&-\mathbb{E}\left(\partial_\theta f(Z_{\phi^\ast(s_0)}(t;x),\theta^\ast(s_0;t))\nabla_\theta\frac{\delta E_s(\phi^\ast(s_0))}{\delta \rho}(\theta^\ast(s_0;t),t)(s-s_0)\right)\\
&+o(|s-s_0|)\,,
\end{aligned}
\end{equation}
where we use $|\partial^2_xf|\leq C_1|\theta|^2$ by \eqref{eqn:derivativebound} and $\EE(|\theta^\ast(s;t)|^2)<C(\mathcal{L}_{S,\phi^\ast})$ by \eqref{eqn:secondmomentbound} in (a). In (b), we use $|\partial_z\partial_{\theta}f|\leq C_1|\theta|(|x|+1)$ by \eqref{eqn:derivativebound}, $|Z_{\phi^\ast(s_0)}|<C(\mathcal{L}_{S,\phi^\ast})$ by \eqref{boundofxsolution}, and $\EE(|\theta^\ast(s;t)-\theta^\ast(s_0;t)|)<C(\mathcal{L}_{S,\phi^\ast})(s-s_0)$ by \eqref{stabilityofODEsolutiontheta} to obtain that first two terms are $o(s-s_0)$. In (c), we also use $|\partial_{\theta}f|\leq C_1(|x|+1)$ by \eqref{eqn:derivativebound}, $|Z_{\phi^\ast(s_0)}|<C(\mathcal{L}_{S,\phi^\ast})$ by \eqref{boundofxsolution}, and $\EE(|\theta^\ast(s;t)-\theta^\ast(s_0;t)|)<C(\mathcal{L}_{S,\phi^\ast})(s-s_0)$ by \eqref{stabilityofODEsolutiontheta}. In (d), we first use \eqref{Lipshictz}, \eqref{eqn:Deltaiteration}, and \eqref{eqn:Deltabound2} to obtain 
\begin{align*}
&\left|\nabla_\theta\frac{\delta E_s(\phi^\ast(s))}{\delta \rho}(\theta^\ast(s;t),t)-\nabla_\theta\frac{\delta E_s(\phi^\ast(s_0))}{\delta \rho}(\theta^\ast(s_0;t),t)\right|\\
\leq &C\left(\left|\theta^\ast(s;t)-\theta^\ast(s_0;t)\right|+d_2(\phi^\ast(s),\phi^\ast(s_0))(|\theta^\ast(s_0;t)|+1)\right)\,
\end{align*}
which implies
\begin{equation}\label{closetotheta}
\mathbb{E}\left|\theta^\ast(s;t)-\theta^\ast(s_0;t)-\nabla_\theta\frac{\delta E_s(\phi^\ast(s_0))}{\delta \rho}(\theta^\ast(s_0;t),t)(s-s_0)\right|=o(|s-s_0|)
\end{equation}
because $\EE(|\theta^\ast(s;t)|^2)<C(\mathcal{L}_{S,\phi^\ast})$ by \eqref{eqn:secondmomentbound}. Combining \eqref{closetotheta} with $|\partial_{\theta}f|\leq C_1(|x|+1)$ by \eqref{eqn:derivativebound} and $|Z_{\phi^\ast(s_0)}|<C(\mathcal{L}_{S,\phi^\ast})$, we obtain (d). Since $Z_{\phi^\ast(s)}(0;x)=Z_{\phi^\ast(s_0)}(0;x)$, using the last equality, we obtain that
\begin{equation}\label{differenceofZ1}
\begin{aligned}
&Z_{\phi^\ast(s)}(1;x)-Z_{\phi^\ast(s)}(1;x)\\
=-\int^1_0 &e^{\int^1_t[\mathbb{E}\left(\partial_zf(Z_{\phi^\ast(s_0)}(\tau;x),\theta^\ast(s_0;\tau))\right)\rd\tau}\\
&\cdot\mathbb{E}\left(\partial_\theta f(Z_{\phi^\ast(s_0)}(t;x),\theta^\ast(s_0;t))\nabla_\theta\frac{\delta E_s(\phi^\ast(s_0))}{\delta \rho}(\theta^\ast(s_0;t),t)\right)\rd t(s-s_0)\\
\end{aligned}
\end{equation}

Using \eqref{eqn:prho}, we also have
\begin{equation}\label{differenceofprho1}
\begin{aligned}
p_{\phi^\ast}(t;x)=&e^{\int^1_t\mathbb{E}\left(\partial_zf(Z_{\phi^\ast(s_0)}(\tau;x),\theta^\ast(s_0;\tau))\right)\rd \tau}p_{\phi^\ast}(1;x)\\
=&e^{\int^1_t\mathbb{E}\left(\partial_zf(Z_{\phi^\ast(s_0)}(\tau;x),\theta^\ast(s_0;\tau))\right)\rd \tau}\left(\left\langle w,Z_{\phi^\ast(s_0)}(1;x)\right\rangle-y(x)\right)w
\end{aligned}
\end{equation}

Combine \eqref{differenceofZ1} and \eqref{differenceofprho1}, we obtain that
\[
\begin{aligned}
&\mathbb{E}_{x\sim \mu}\left(\left(\left\langle w,Z_{\phi^\ast(s_0)}(1;x)\right\rangle-y(x)\right)\left\langle w,Z_{\phi^\ast(s)}(1;x)-Z_{\phi^\ast(s_0)}(1;x)\right\rangle\right)\\
=&\mathbb{E}_{x\sim \mu}\left(\left\langle p_{\phi^\ast}(1;x),Z_{\phi^\ast(s)}(1;x)-Z_{\phi^\ast(s_0)}(1;x)\right\rangle\right)\\
=&-\mathbb{E}_{x\sim \mu}\left(\int^1_0p_{\phi^\ast}(t;x)\mathbb{E}\left(\partial_\theta f(Z_{\phi^\ast(s_0)}(t;x),\theta^\ast(s_0;t))\nabla_\theta\frac{\delta E_s(\phi^\ast(s_0))}{\delta \rho}(\theta^\ast(s_0;t),t)\right)\rd t\right)(s-s_0)\\
&+o(|s-s_0|)\,,
\end{aligned}
\]
which implies
\begin{equation}\label{eqn:Ifinal}
\begin{aligned}
{\textrm{(I)}}=&-\int^1_0\mathbb{E}\left(\mathbb{E}_{x\sim \mu}\left(p_{\phi^\ast}(t;x)\partial_\theta f(Z_{\phi^\ast(s_0)}(t;x),\theta^\ast(s_0;t))\right)\nabla_\theta\frac{\delta E_s(\phi^\ast(s_0))}{\delta \rho}(\theta^\ast(s_0;t),t)\right)\rd t(s-s_0)\\
&+o(|s-s_0|)
\end{aligned}
\end{equation}

\item[(II):] 

Similar to before, using \eqref{stabilityofODEsolutiontheta} and \eqref{closetotheta}, we obtain that
\begin{equation}\label{eqn:IIfinal}
\begin{aligned}
{\textrm{(II)}}=&\exp^{-s}\int^1_0\mathbb{E}(|\theta^\ast(s;t)|^2)\rd t-\exp^{-s}\int^1_0\mathbb{E}(|\theta^\ast(s_0;t)|^2)\rd t\\
+&\exp^{-s}\int^1_0\mathbb{E}(|\theta^\ast(s_0;t)|^2)\rd t-\exp^{-s_0}\int^1_0\mathbb{E}(|\theta^\ast(s_0;t)|^2)\rd t\\
=&\exp^{-s}\int^1_0\mathbb{E}(|\theta^\ast(s;t)|^2)\rd t-\exp^{-s}\int^1_0\mathbb{E}(|\theta^\ast(s_0;t)|^2)\rd t\\
&-\exp^{-s_0}\int^1_0\mathbb{E}(|\theta^\ast(s_0;t)|^2)\rd t(s-s_0)+o(|s-s_0|)\\
\stackrel{(a)}{=}&\exp^{-s_0}\int^1_0\mathbb{E}(|\theta^\ast(s;t)|^2)\rd t-\exp^{-s_0}\int^1_0\mathbb{E}(|\theta^\ast(s_0;t)|^2)\rd t\\
&-\exp^{-s_0}\int^1_0\mathbb{E}(|\theta^\ast(s_0;t)|^2)\rd t(s-s_0)+o(|s-s_0|)\\
\stackrel{(b)}{=}&-2\exp^{-s_0}\int^1_0\mathbb{E}\left(\theta^\ast(s_0;t)\nabla_\theta\frac{\delta E_s(\phi^\ast(s_0))}{\delta \rho}(\theta^\ast(s_0;t),t)(s-s_0)\right)\rd t\\
&-\exp^{-s_0}\int^1_0\mathbb{E}(|\theta^\ast(s_0;t)|^2)\rd t(s-s_0)+o(|s-s_0|)
\end{aligned}\,.
\end{equation}
In (a), we use similar calculations as the third to the fifth equality in \eqref{Zdifferencebig}. In (b), we use similar calculations as the last equality in \eqref{Zdifferencebig}. 

Noticing that
\[
\nabla_\theta\frac{\delta E_s(\phi^\ast(s_0))}{\delta \rho}(\theta^\ast(s_0;t),t)=\mathbb{E}_{x\sim \mu}\left(\partial_\theta f(Z_{\phi^\ast(s_0)}(t;x),\theta^\ast(s_0;t))p_{\phi^\ast}(t;x)\right)+2\exp^{-s_0}\theta^\ast(s_0;t)\,,
\]
we finally have
\[
{\textrm{(I)}}+{\textrm{(II)}}=-D(s_0)+o(|s-s_0|)
\]
by \eqref{eqn:Ds_0}, \eqref{eqn:Ifinal}, and \eqref{eqn:IIfinal}. This proves \eqref{eqn:derivativeproof} and the lemma.
\end{enumerate}
\end{proof}

\vskip 0.2in
\bibliography{ref}
\end{document}